\documentclass[review]{elsarticle}

\usepackage{lineno,hyperref}
\usepackage{algcompatible}
\usepackage{algorithm}
\usepackage[shortlabels]{enumitem}
\usepackage{epstopdf}
\usepackage{dsfont}
\usepackage{graphicx}
\usepackage{subfigure}
\usepackage[section]{placeins}
\usepackage{multirow}
\usepackage{hyperref}
\usepackage{amsthm}
\usepackage{amsfonts}
\usepackage{xcolor}
\usepackage{natbib} 
\usepackage{cool}
\usepackage{centernot}
\usepackage[T1]{fontenc}
\usepackage[left=2.5cm,right=2.5cm, bottom = 2.2cm, top = 2.2cm]{geometry}

\newtheorem{theorem}{Theorem}

\newtheorem{proposition}{Proposition}
\newtheorem*{Ass}{Assumptions}

\newcommand{\R}{{\mathbb{R}}}
\newcommand{\N}{{\mathbb{N}}}

\newcommand{\indep}{\rotatebox[origin=c]{90}{$\models$}}

\modulolinenumbers[5]

\journal{Journal of \LaTeX\ Templates}

\begin{document}

\begin{frontmatter}

\title{Asymptotic Unbiasedness of the Permutation Importance Measure in  \\Random Forest Models.}

\author{Burim Ramosaj$^*$, Markus Pauly}
\address{Faculty of Statistics \\
	 Institute of Mathematical Statistics and Applications in Industry\\
	 Technical University of Dortmund \\
	44227 Dortmund, Germany}


\cortext[mycorrespondingauthor]{\textit{Corresponding Author:} Burim Ramosaj \\
	\textit{Email address:} \texttt{burim.ramosaj@tu-dortmund.de} }


\begin{abstract}
	Variable selection in sparse regression models is an important task as applications ranging from biomedical research to econometrics have shown. Especially for higher dimensional regression problems, for which the link function between response and covariates cannot be directly detected, the selection of informative variables is challenging. Under these circumstances, the Random Forest method is a helpful tool to predict new outcomes while delivering measures for variable selection. One common approach is the usage of the permutation importance. Due to its intuitive idea and flexible usage, it is important to explore circumstances, for which the permutation importance based on Random Forest correctly indicates informative covariates. Regarding the latter, we deliver theoretical guarantees for the validity of the permutation importance measure under specific assumptions and prove its (asymptotic) unbiasedness. An extensive simulation study verifies our findings. 
\end{abstract}


\end{frontmatter}

\linenumbers

\section{Introduction}
Random Forest is a non-parametric classification and regression algorithm being known for its good predictive performance and simple applicability under various settings. The method is based on constructing each tree in the forest by bagging procedures, which enables the construction of several estimators based on \textit{Out-of-Bag} principles, such as prediction points or variance estimates. Main advantages of the Random Forest method compared to other Machine Learning tools is its relative ease in hyper-parameter tuning while delivering internal estimates of the mean squared error. Due to its complicated mathematical description, including data-dependent weighting, theoretical results such as consistency or central limit theorems have only been derived recently, see e.g. \cite{wager2014asymptotic,wager2014confidence, scornet2015consistency,mentch2016quantifying}.

Beyond its usage for prediction, Random Forest models can also be used as a tool for variable selection. Especially in high dimensional learning problems, where the number of variables exceeds the number of observations, the extraction of an informative feature subset is beneficial from three perspectives: Firstly, a reduced and simplified model is more accessible and interpretable than models of higher dimensions leading to faster and easier data collection processes. Secondly, model accuracy can sometimes even be enhanced under lower dimensional models bypassing the possibility of overfitting.  Thirdly, a reduced model makes well-known statistical inference procedures applicable. As mentioned in \cite{guyon2003introduction}, the Random Forest model can be considered as an embedded model, where variable selection is an integral part of the tree construction. In selecting variables, the Random Forest method delivers two measures: The permutation importance as well as the mean decrease impurity. For classification, the mean decrease impurity summarizes the decrease of the Gini impurity after conducting a cut over the whole tree structure and averages the result over all trees. For regression problems, this measure turns into the summation of the decrease in variance after conducting a cut at every node of the tree, averaged over the forest. The principle of the permutation importance is slightly different: In order to mimic the effect of a variable on the response, its values from the set of Out-of-Bag samples are randomly permuted and the decrease in model accuracy averaged over all trees is measured.  \\

Although simple to apply and intuitive, both measures have been criticized. In \cite{strobl2007bias}, for example, the authors could illustrate that the Gini importance for classification problems tends to prefer variables with larger numbers of categories and scale measurements. Furthermore, different results could be obtained when switching the sampling procedure in the bagging step to sampling with replacement instead of without replacement. In \cite{strobl2008conditional}, additional criticism was addressed towards the permutation importance, arguing that the permutation of the corresponding feature does not only break the relation with the response variable, but also with other potentially correlated covariates. This effect of correlated features has since been part of several research \cite{strobl2007bias, archer2008empirical,nicodemus2009predictor, nicodemus2010behaviour, nicodemus2011letter, altmann2010permutation,genuer2010variable}. Nevertheless, several authors such as \cite{strobl2007bias,nicodemus2011letter}  claimed that the permutation importance led to more accurate results than the importance measure based on decrease in node impurities. However, theoretical guarantees for the validity of the traditional Random Forest method regarding its importance measures are rather sparse. An exception is given in \cite{gregorutti2017correlation}, where a theoretical approach has been conducted within the framework of correlated features in additive regression models. Therein, the authors showed different identities of a formalized version of the Random Forest permutation importance measure (RFPIM). \\

The contributions of this paper are twofold: First, we aim to clarify the criticism on the RFPIM from a theoretical perspective. Therefore, we state assumptions, for which the permutation importance measure does correctly select informative features and prove its  (asymptotic) unbiasedness. This way, we also close the gap between the formalized version of the permutation measure as considered in \cite{gregorutti2017correlation} and the empirical permutation measure computed in a Random Forest model. Secondly, we identify main drivers for the quality of the RFPIM and support our findings by an extensive simulation study covering high-dimensional settings, too.

\section{Model Framework and Random Forest}\label{SecModel}

Our framework covers regression models, for which the covariable space is assumed to lie on the p-dimensional unit space, i.e. $\mathbf{X} \in [0,1]^p$. In fact, this assumption does not have sever generalization effects, since Random Forest models are invariant under (strictly) monotone transformations. For discrete distributions of $\mathbf{X}$, one could alternatively assume a finite support, such that a $[0,1]$-standardization exists for every feature $j = 1, \dots, p$. Furthermore, we will assume that the relationship between the response variable $Y$ and the covariates $\mathbf{X}$ can be modeled through
\begin{align} \label{RegModel}
	Y &= \widetilde{m}(\mathbf{X}) + \epsilon,
\end{align}
where $\tilde{m}: [0,1]^p \longrightarrow \R $ is a measurable function and $\mathbf{X}$ is independent of $\epsilon$ with $\mathbb{E}[\epsilon] = 0$, $Var(\epsilon) \equiv \sigma^2 \in (0, \infty)$. For sparse learning problems, not all of the given covariates are necessary, that is, there is a subset $\mathcal{S}$ with cardinality $s$ less than $p$ that covers all the information about $Y$. Assuming without loss of generality that $\mathcal{S} = \{ 1, \dots, s \}$, the regression model (\ref{RegModel}) can then be reduced to 
\begin{align}\label{SparsReg}
	Y &=  m(\mathbf{X}_{\mathcal{S}}) + \epsilon,
\end{align}
where $\mathbf{X}_{\mathcal{S}} =  [X_{1}, \dots, X_{s}]$ and $m: [0,1]^s \longrightarrow \R$ is another measurable function such that $\widetilde{m}(\mathbf{X}) = m(\mathbf{X}_\mathcal{S})$. The specification of $\mathcal{S}$, or also known as \textit{variable selection, feature selection or subset selection}, can be challenging, especially when the relationship is not linear or not deducible at all. Formally speaking, we refer to a variable $j \in \{1, \dots, p\}$ as \textit{informative} or \textit{important}, if the corresponding regression model given in $(\ref{RegModel})$ can be reduced to a regression model of the form $(\ref{SparsReg})$. This leads to the independence of $Y$ towards $X_j$  given all other covariates for features $j \notin \mathcal{S}$. That is $Y \indep  X_j | X_{1}, \dots, X_{j-1}, X_{j+1}, \dots, X_{p}$. For differentiable link-functions $\widetilde{m}$, one can alternatively define a variable as \textit{unimportant} or \textit{uninformative}, if for $\boldsymbol{h}_j = [0, \dots, 0, h, 0, \dots, 0]^\top \in \R^p$, with $h \in \R$ lying at the $j$-th position, it holds
\begin{align}\label{SparseDifferentiable}
	\frac{ \partial \widetilde{m}(\mathbf{x}) }{ \partial x_j } := \lim\limits_{ \| \boldsymbol{h}_j \| \rightarrow 0} \frac{\widetilde{m}( \mathbf{x} + \boldsymbol{h}_j) -  \widetilde{m}(\mathbf{x}) }{ \| \boldsymbol{h}_j \| } = 0.
\end{align}
Then a feature $j \in \{1, \dots, p\}$ is said to be \textit{informative} or \textit{important}, if it is not \textit{uninformative} or \textit{unimportant}. Under the scenario of a differentiable link function $\widetilde{m}$, both definitions given in $(\ref{SparsReg})$ and $(\ref{SparseDifferentiable})$ for an \textit{informative} or \textit{important} variable can be 
shown to be equivalent using a Taylor expansion of $\widetilde{m}$.  \\
Although there are several approaches in extracting informative features, difficulties exist if the underlying link function is of complex analytical structure. The Random Forest method enables the extraction of informative features during the training phase of the algorithm. To accept this, let us shortly recall the Random Forest. Given a training set  
 \begin{align} \label{Sample}
 	\mathcal{D}_n =  \{[\mathbf{X}_i^\top, Y_i]^\top \in [0,1]^p \times \R : i = 1, \dots, n \},
 \end{align}
of iid pairs $[\mathbf{X}_i^\top, Y_i]^\top$, $i = 1, \dots, n$, the Random 
Forest method estimates the functional relationship of $\widetilde{m}$ by piecewise constant functions over random partitions of the feature space. To be more precise, the Random Forest model for regression is a collection of $M \in \N$ decision trees, where for each tree, a bootstrap sample is taken from $\mathcal{D}_n$ using with or without replacement procedures. This is denoted as the resampling strategy $\mathcal{P}$.
Furthermore, at each node of the tree, feature sub-spacing is conducted selecting $v_{try} \in \{ 1, \dots ,p \}$ features for possible split direction. Denote with $\boldsymbol{\Theta}$ the generic random variable responsible for both, the bootstrap sample construction and the feature sub-spacing procedure. Then, $\boldsymbol{\Theta}_1, \dots, \boldsymbol{\Theta}_M$ are assumed to be independent copies of $\boldsymbol{\Theta}$ responsible for this random process in the corresponding tree, independent of $\mathcal{D}_n$. The combination of the trees is then conducted through averaging. i.e. 
\begin{align} \label{FiniteForest}
m_{n, M} (\mathbf{x}; \boldsymbol{\Theta}_1, \dots \boldsymbol{\Theta}_M, \mathcal{D}_n) = \frac{1}{M} \sum\limits_{j = 1}^M m_{n, 1}(\mathbf{x}; \boldsymbol{\Theta}_j, \mathcal{D}_n)
\end{align}

and is referred to as the finite forest estimate of $\widetilde{m}$, where $\mathbf{x} \in [0,1]^p$ is a fixed point. Here, $m_{n, 1}(\cdot; \boldsymbol{\Theta}_j, \mathcal{D}_n)$ refers to a single tree in the Random Forest build with $\boldsymbol{\Theta}_j$, $j = 1, \dots, M$. As explained in \cite{scornet2015consistency}, the strong law of large numbers (for $M \rightarrow \infty$) allows to study $\mathbb{E}_\Theta[ m_n(\mathbf{x}; \boldsymbol{\Theta} , \mathcal{D}_n) ]$ instead of $(\ref{FiniteForest})$. Hence, we set
\begin{align}\label{InfForest}
	m_n(\mathbf{x}) = m_n(\mathbf{x}; \mathcal{D}_n) = \mathbb{E}_{\boldsymbol{\Theta}}[ m_n(\mathbf{x}; \boldsymbol{\Theta} , \mathcal{D}_n) ], 
\end{align}
where $\mathbb{E}_\Theta$ denotes the expectation over $\boldsymbol{\Theta}$ given the training set $\mathcal{D}_n$, i.e. $\mathbb{E}_{\boldsymbol{\Theta}}[ m_n(\mathbf{x}; \boldsymbol{\Theta}_j , \mathcal{D}_n) ] = \mathbb{E} [ m_n(\mathbf{x}; \Theta, \mathcal{D}_n) | \mathcal{D}_n ]$. Similar to \cite{scornet2015consistency}, we refer to the Random Forest algorithm by identfiying three parameters responsible for the Random Forest tree construction: 
\begin{itemize}
	\item  $v_{try} \in \{1, \dots, p\}$, the number of pre-selected directions for splitting,
	\item $a_n \in \{1, \dots, n \}$, the number of sampled points in the bootstrap step and
	\item $t_n \in \{ 1, \dots, a_n \}$, the number of leaves in each tree. 
\end{itemize}
A detailed algorithm is given on page $1720$ in \cite{scornet2015consistency}, for example.\\
An advantage of the Random Forest method is the delivery of internal measures such as predictions or prediction accuracy without initially separating the training set $\mathcal{D}_n$ such as in cross-validation procedures. This is possible by making use of the bagging principle and Out-Of-Bag (OOB) samples. The latter extracts all random trees that have not used a fixed observation $\mathbf{X}_i$ in the set $\mathcal{D}_n$ during training and averages the prediction results over all those trees. In the sequel, we will denote with $m_{n, M}^{OOB}(\mathbf{X}_i)$ the OOB prediction of $\mathbf{X}_i \in \mathcal{D}_n$ using the finite forest estimate and $m_n^{OOB}(\mathbf{X}_i ) = \mathbb{E}_{\boldsymbol{\Theta}_{[i]} }[ m_{n ,1}(\mathbf{X}_i; \boldsymbol{\Theta}_{[i]}, \mathcal{D}_{n}) ]$ the corresponding infinite forest OOB prediction, where $\boldsymbol{\Theta}_{[i]}$ is the generic random vector, which has not selected observation $i \in \{1, \dots, n\}$. Note that the authors in \cite{ramosaj2019consistent} could show that even for the OOB finite forest prediction, it holds $\mathbb{P}_{\Theta}$ - almost surely that 
\begin{align*}
	m_{n, M}^{OOB}(\mathbf{X}_i) \longrightarrow m_n^{OOB}(\mathbf{X}_i), \quad \text{ as } M \rightarrow \infty. 
\end{align*}

In the sequel, it is required to have a look at a certain averaging step in the random tree ensemble of the Random Forest and its asymptotic behavior in case of $M \rightarrow \infty$. For later use,  we state this as a Proposition.  \\

\begin{proposition}\label{HelpingProposition}
	Assume regression model $(\ref{RegModel})$ and fix $i \in \{1,\dots, n\}$. Then it holds $\mathbb{P}_\Theta$ -  almost-surely that 
	\begin{align*}
		\frac{1}{M} \sum\limits_{ t = 1}^M m_{n, 1}(\mathbf{X}_i; \boldsymbol{\Theta}_{t}, \mathcal{D}_n) \cdot \mathds{1}\{ \mathbf{X}_i \text{ has not been selected} \} \longrightarrow c_n \cdot m_{n}^{OOB}(\mathbf{X}_i), \quad \text{ as } M \rightarrow \infty, 
	\end{align*}
	where 
	\begin{align*}
		c_n &= \begin{cases}
		1 - a_n/n, &\text{ if observations are subsampled (draws without replacement),} \\
		(1-1/n)^n, &\text{ if observations are bootstrapped with replacement.}
		\end{cases}
	\end{align*}
\end{proposition}

\section{Permutation Importance of the Random Forest} \label{PermImpRF}

Returning to the extraction of relevant features, the Random Forest permutation importance makes use of the Out-of-Bag principle. That is, for every tree constructed in the forest, the increase of mean squared error evaluated on the corresponding Out-of-Bag sample after permuting its observations along the $j$-th variable is measured, with $j \in \{1, \dots, p\}$. Hence, the measure clearly depends on the sampling strategy $\mathcal{P}$ chosen prior to tree construction. This could be seen in \cite{strobl2007bias} for example, where different results were obtained depending on the sampling strategy given in $\mathcal{P}$. Formally speaking, the permutation importance can be defined as
	\begin{align}\label{ImportanceDefinition}
		I_{n,M}^{OOB}(j) := \frac{1}{M \gamma_n} \sum\limits_{t = 1}^M  \sum\limits_{i \in \mathcal{D}_n^{-(t) } }\left\{  (Y_i  - m_{n, 1}(\mathbf{X}_i^{\pi_{j,t} }; \boldsymbol{\Theta}_t ) )^2   - (Y_i -  m_{n, 1}(\mathbf{X}_i; \boldsymbol{\Theta}_t) )^2  \right\}
	\end{align}
	for all $j \in \{1, \dots, p\}$, where $\mathcal{D}_{n}^{-(t)} = \mathcal{D}_n^{-(t)}(\boldsymbol{\Theta}_t)$ is the Out-of-Bag sample for the $t$-th tree, i.e. the set of observations not selected for training $m_{n,1 }(\cdot ; \boldsymbol{\Theta}_t, \mathcal{D}_n)$. The cardinality $\gamma_n$ of $\mathcal{D}_n^{-(t)}$ clearly depends on the sampling strategy $\mathcal{P}$. Moreover, $\pi_{j, t}$ is the non-trivial permutation of observations in $\mathcal{D}_{n}^{-(t)}$ along the $j$-th variable in decision tree $t \in \{1, \dots, M\}$. In \cite{zhu2015reinforcement} and \cite{gregorutti2017correlation}, a \textit{theoretical version} of $I_{n, M}^{OOB}(j)$, $j \in \{1, \dots, p\}$ was given by
	\begin{align}\label{THQuant}
		I(j) & : = \mathbb{E}[ (Y_1 - \widetilde{m}(\mathbf{X}_{j, 1}))^2 ] - \mathbb{E}[ (Y_1- \widetilde{m}(\mathbf{X}_1))^2 ] \notag \\
		&= \mathbb{E}[ (Y_1 - \widetilde{m}(\mathbf{X}_{j, 1}))^2 ] - \sigma^2, 
	\end{align}
	where $\mathbf{X}_{j, 1} = [X_{1,1}, \dots, X_{j-1, 1}, Z_j, X_{j+1, 1}, \dots, X_{p,1}]^\top$ and $Z_j$ is an independent copy of $X_{j, 1}$, independent of $Y_1$.  The intuition behind the definition in $(\ref{THQuant})$ is that $I(j)$, $j = 1, \dots, p$ measures the increase in variation after eliminating potential dependencies between the $j$-th variable and the response. \\ Assuming an additive regression model,  i.e. $\widetilde{m}(\mathbf{x}) = \sum\limits_{j = 1}^p \widetilde{m}_j(x_j)$, \cite{gregorutti2017correlation} proved that 
	\begin{align}\label{TheoreticalQuantity}
	I(j) &= \begin{cases}
	2 \cdot Cov(Y, \widetilde{m}_j(X_j)) - \sum\limits_{k \neq j } Cov(\widetilde{m}_j(X_j), \widetilde{m}_k(X_k) ) &\text{ if } \mathbb{E}[\widetilde{m}_j(X_j)] = 0, \\ 
	2 \cdot Var(\widetilde{m}_j(X_j)), &\text{ else}, 
	\end{cases}
	\end{align}
	for $j \in \{1, \dots, p\}$, where $I(j)$ can be further simplified in case of a multivariate normal distribution for $[\mathbf{X}^\top, Y]^\top \in \R^{p+1}$, see e.g. Proposition $2$ in \cite{gregorutti2017correlation}. So far, however, it is completely unclear in which sense the quantities $\boldsymbol{I}_{n, M}^{OOB} = [ I_{n, M}^ {OOB}(1), \dots, I_{n, M}^{OOB}(p) ]^\top$  and $\boldsymbol{I} = [I(1), \dots, I(p)]^\top$ relate to each other. This is of important interest, since $\boldsymbol{I}$ can, e.g.,  be considered as a key quantity for future significance tests during feature extraction. Below, we will study their relation in detail under a more general set-up not requiring the additivity of the link function $\widetilde{m}$. Instead, we set up some more general assumptions, under which we can guarantee asymptotically, that $\boldsymbol{I}_{n, M}$ is an unbiased estimator of $\boldsymbol{I}$. This will open new paths for feature selection tests using Random Forest.

\begin{Ass} \text{ } 
\begin{enumerate}[label=(A\arabic*)]
	\item There is at least one informative variable, i.e. $|\mathcal{S}| \ge 1$, \label{Ass0}
	\item Permutations are restricted to the class $\mathcal{V} = \{ \pi \in \mathcal{S}_{\gamma_n} : \pi(i) \neq i \}$, where $\mathcal{S}_{\gamma_n}$ is the symmetric group, \label{Ass01}
	\item The features are pairwise independent, i.e. $X_{i} $ is independent of $X_j$ for all $i \neq j \in \{1, \dots, p\}$, \label{Ass1}
	\item $\sup\limits_{\mathbf{x}} |\widetilde{m}(\mathbf{x})|  < \infty$, \label{Ass2} 
	\item Infinite Random Forests are $L_2$-consistent, i.e. $ \lim\limits_{ n  \rightarrow \infty}\mathbb{E}[ (\widetilde{m}(\mathbf{X}) - m_n(\mathbf{X}))^2 ] = 0$, where $\mathbf{X}$ is an independent copy of $\mathbf{X}_1$. \label{Ass3}
\end{enumerate}
\end{Ass}

Condition $\ref{Ass0}$ ensures that the random forest is not forced to select among non-informative variable. This can happen if $|\mathcal{S}| = 0$, since the tree construction process will continue until either a pre-defined number of leaves $t_n$ is reached or each leave in a tree consists of at most a pre-specified number of observations. Condition \ref{Ass01} is important from a technical perspective, in order to achieve (asymptotic) unbiasedness. Furthermore, this condition reveals some drawbacks of the traditional permutation approaches: considering arbitrary permutations $\pi \in \mathcal{S}_{\gamma_n}$, we cannot guarantee the (asymptotic) unbiasedness of the RFPIM. Hence, one should carefully consider implementations of RFPIM in statistical software packages such as \textsf{R} or \textsf{python} with regard to this assumption. Condition $\ref{Ass1}$ is essential in this context. The permutation used in $\boldsymbol{I}_{n, M}^{OOB}$ aims to break the relationship between the response variable and the corresponding covariate. In case of dependency structures among the other covariables, however, this dependency is then also broken clouding the primary effect of dependencies between the response and the covariable of interest. Note that assumption $\ref{Ass1}$ implies the assumption of no multicolinearity. Condition $\ref{Ass2}$ is rather technical. Instead, one could replace it with $\widetilde{m}$ being continuous, since the domain of $\mathbf{X}$ is the $p$-dimensional unit cube $[0,1]^p$. An important assumption is $\ref{Ass3}$, which was formally proven to be valid for Random Forest models in \cite{scornet2015consistency}. There, the authors proved the $L_2$ - consistency of the same Random Forest method as considered in our work. Note that their assumptions for the validity of $\ref{Ass3}$ do not exclude $\ref{Ass1} $ and $\ref{Ass2}$. Instead, one could completely overtake the assumptions given in Theorem $1$ or Theorem $2$ listed in \cite{scornet2015consistency} and replace them with $\ref{Ass1}$ - $\ref{Ass3}$. Assumptions $\ref{Ass0}$ and $\ref{Ass01}$ have then to be considered as additional assumptions in this context. A formal proof of this is given in the Appendix. However, for generality and as we also state non-asymptotic results, we decided to work with ours. \\
Our first result shows an alternative expression of the quantity $\boldsymbol{I}$ defined in $(\ref{THQuant})$, which makes variable selection possible for the Random Forest permutation importance. 
	
	\begin{proposition}\label{Proposition}
		Assume the regression model $(\ref{RegModel})$ and conditions $\ref{Ass0}$, $\ref{Ass1}$ and $\ref{Ass2}$. Then for every variable $j \in \{1, \dots, p\}$ it holds
		\begin{align*}
		I(j)  &= \begin{cases}
		\mathbb{E}[ (\widetilde{m}(\mathbf{X}_1) - \widetilde{m}(\mathbf{X}_{j,  1}))^2 ], \text{ if } j \in \mathcal{S}, \\
		0, \text{ else }. 
		\end{cases}
		\end{align*}
	\end{proposition}
	This property allows us to define the permutation importance as unbiased or asymptotically unbiased, if $\mathbb{E}[\boldsymbol{I}_{n,M}^{OOB}] = \boldsymbol{I}$ resp. $\lim\limits_{ M \rightarrow \infty}\mathbb{E}[\boldsymbol{I}_{n, M}^{OOB}] \longrightarrow \boldsymbol{I}$, as $n \rightarrow \infty$. Proposition \ref{Proposition} can be considered as an extension of the results given in equation $(\ref{TheoreticalQuantity})$, since the assumption for the link-function being additive is dropped. Anyhow, the above considerations finally lead to the main result of the current paper:  the (asymptotic) unbiasedness of RFPIM.

\begin{theorem} \label{UnbiasednessPermutationImportance}
	Under model (\ref{RegModel}) and conditions $\ref{Ass0}$ - $\ref{Ass3}$ while sampling is restricted to sampling without replacement, the RFPIM is unbiased for $j \in \mathcal{S}^C = \{ 1, \dots, p\} \setminus \mathcal{S}$ and asymptotically unbiased for $j \in \mathcal{S}$.  That is for $j \in \mathcal{S}^C$ it holds
	\begin{align*}
		\mathbb{E}[I_{n, M}^{OOB}(j) ]  & = 0 = I(j)
	\end{align*}
	and for $j \in \mathcal{S}$ we have 
\begin{align*}
	\lim\limits_{ M \rightarrow \infty} \mathbb{E}[ I_{n, M}^{OOB}(j) ] \longrightarrow I(j), \text{ as } n \rightarrow \infty. 
\end{align*}
\end{theorem}

Theorem \ref{UnbiasednessPermutationImportance} and equation (\ref{TheoreticalQuantity}) under the assumption of an additive link function reveal some important insights about the RFPIM. In case of non-informative variables, i.e. $Y$ is independent of $X_j$ or equivalently, $\partial \widetilde{m}(\mathbf{x})/\partial x_j \equiv 0$, the empirical variable importance does not select on average across non-informative variables. However, if the variable is informative, that is $\partial \widetilde{m}(\mathbf{x})/\partial x_j \neq 0$ and $X_j$ depends on $Y$, this will lead to $I(j) > 0$, such that on average, their is enough discriminating power between informative and non-informative variables. Furthermore, the theoretical results obtained from Theorem \ref{UnbiasednessPermutationImportance} and equation $(\ref{TheoreticalQuantity})$ allow the sorting of variables according to their signal strength, if the underlying link-function is assumed to be additive. Hence, under the assumptions $\ref{Ass0}$ - $\ref{Ass3}$ together with the assumption that $\widetilde{m}$ decomposes into an additive expansion of measurable functions, the RFPIM does not only detect informative variables, but also delivers an internal ranking across variables in $\mathcal{S}$.  In addition, the theoretical results in Theorem $\ref{UnbiasednessPermutationImportance}$ also reveal that unimportant variables tends to $0$ stronger than important ones, since the unbiasedness is exact in that case for any sample size $n \in \N$ and number of base learners $M \in \N$. The theoretical findings also indicate that the discriminating power of the permutation importance depends on the sample size of the training set $\mathcal{D}_n$ and the number of base learners $M$. Larger sample sizes with a relatively large number of decision trees in the Random Forest should deliver stronger discriminating power between variables in $\mathcal{S}$ and $\{1, \dots, p\} \setminus \mathcal{S}$.  Note that the theoretical findings do not reveal insights into the rate of convergence of the asymptotic. However, an important factor influencing the discriminating power of the permutation importance measure that cannot be directly extracted from the theoretical findings so far is the random noise arising from the residuals $\epsilon$. These contaminate the data especially depending on the scale of their variance $\sigma^2$. Nonetheless, if the systematic signal arising from the link function $m(\mathbf{x})$ is strong enough, the effect of noise can be appeased. Thus, keeping an eye on the ratio 
\begin{align}\label{}
	SN = \frac{Var(\widetilde{m}(\mathbf{X}))}{\sigma^2}
\end{align}
is an important task during the computation of the RFPIM. We refer to this meaasure as the \textit{signal-to-noise} ratio, which is formally defined in \cite{hastie2009overview}. Although this factor cannot be directly detected based on the results in Theorem \ref{UnbiasednessPermutationImportance}, a closer look at the specific cut criterion used in the Random Forest will deliver some insights into the interaction of $SN$ and the permutation measure $\boldsymbol{I}_{n, M}^{OOB}= [I_{n, M}^{OOB}(1), \dots, I_{n, M}^{OOB}(p)]^\top \in \R^p$. Recall that the empirical cut criterion of the Random Forest model within the construction of each tree is given by 
\begin{align}\label{EmpiricalCut}
L_{n,t}^{(k)}(j, z) &= \frac{1}{N_n(A_\ell^{(k)})} \sum\limits_{ i = 1}^n (Y_i - \bar{Y}_{A_\ell^{(k)}})^2 \mathds{1}\{ \mathbf{X}_i \in A_\ell^{(k)} \} \notag \\
& - \frac{1}{N_n(A_\ell^{(k)})} \sum\limits_{i = 1}^n (Y_i - \bar{Y}_{A_{\ell, L}^{(k)}}\mathds{1}\{ X_{ji} <z  \}  - \bar{Y}_{A_{\ell, R}^{(k)}} \mathds{1}\{ X_{ji} \ge z \}  )^2  \mathds{1}\{ \mathbf{X}_i \in A_{\ell}^{(k)} \} 
\end{align}
for $t = 1,\dots, M$. Here $A_{\ell}^{(k)}  = A_{\ell}^{(k)}(\Theta_t)\subset [0,1]^p$ denotes the hyper-rectangular cell obtained after cutting the tree at level $k \in \{1, \dots, \lceil \log_2(t_n) \rceil + 1 \}$, $A_{\ell, L}^{(k)} = A_{\ell, L}^{(k)}(\Theta_t)$ denotes the left hyper-rectangular cell after cutting $A_{\ell}^{(k)}$ on variable $j$ in $z$, i.e. $A_{\ell, L}^{(k)} = \{ \mathbf{x} \in A_{\ell}^{(k)} : x_j < z \}$ and $A_{\ell, R}^{(k)} =A_{\ell, R}^{(k)}(\Theta_t)$ is the  corresponding right hyper-rectangular cell $\{ \mathbf{x} \in A_{\ell}^{(k)} : x_j \ge z \}$.  Moreover, $\bar{Y}_{A}$ is the mean of all $Y$'s, that belong to the cell $A$ and $N_n(A)$ refers to the number of observations falling into cell $A$. As stated in \cite{scornet2015consistency}, the strong law of large numbers for $n \rightarrow \infty$ leads to the consideration of 
\begin{align}
L^{(k)}(j, z) &=  Var[Y_1 | \mathbf{X}_1 \in A_\ell^{(k)}] - \mathbb{P}[X_{j, 1} < z | \mathbf{X}_1 \in A_\ell^{(k)}] ¸\cdot Var[Y | \mathbf{X}_1 \in A_\ell^{(k)} , X_{j, 1} < z] \notag  \\
&- \mathbb{P}[X_{j, 1} \ge z | \mathbf{X}_1 \in A_\ell^{(k)}]\cdot Var[Y_1 | \mathbf{X}_1 \in A_\ell^{(k)}, X_{j, 1} \ge z]
\end{align}
such that $L_{n,1}^{(k)}(j,z) \longrightarrow L^{(k)}(j,z)$ holds $\mathbb{P}$ - almost surely for all $(j, z) \in \{1, \dots, p\} \times [0,1]$. If we oppose the cut criterion of the Random Forest to the variance decomposition of the response, we obtain 
\begin{align}
Var(Y_1) &= Var(\widetilde{m}(\mathbf{X}_1)) + \sigma^2. 
\end{align}
Assuming that the Random Forest is cut-consistent, that is 
\begin{align}\label{MEstimatorCuts}
(j_n, z_n ) := \arg\max\limits_{j, z} L_{n, t}^{(k)}(j,z) \longrightarrow (j,z) = \arg\max\limits_{j , z}L^{(k)}(j,z), \quad \mathbb{P} - \text{ almost surely, }
\end{align}

the influence of the signal-to-noise ratio on the cuts $(j_n, z_n)$ reduces immediately, since the residual variance drops out of the theoretical cut criterion which is then given by $L^{(k)}(j,z) = Var[\widetilde{m}(\mathbf{X}_1) | \mathbf{X}_1 \in A_\ell^{(k)}] - \mathbb{P}[X_{j, 1} < z | \mathbf{X}_1 \in A_\ell^{(k)}]\cdot Var[\widetilde{m}(\mathbf{X}_1) | \mathbf{X}_1 \in A_\ell^{(k)} , X_{j, 1} < z] - \mathbb{P}[X_{j, 1} \ge z | \mathbf{X}_1 \in A_\ell^{(k)}]\cdot Var[\widetilde{m}(\mathbf{X}_1) | \mathbf{X}_1 \in A_\ell^{(k)}, X_{j, 1} \ge z]$. For a formal proof, we refer to the Appendix. However, this clearly depends on the sample size and the assumption that Random Forest cuts are consistent M-estimators in the sense of $(\ref{MEstimatorCuts})$. The proof of the latter should therefore be considered in future research. In case of $\sigma^2$ being larger than $Var(\widetilde{m}(\mathbf{X}))$, the cut $(j_n,z_n)$ conducted by the Random Forest might be inflated in terms of potentially selecting non-informative variables. The estimation of $SN$ can therefore be considered as an additional control mechanism in computing $\boldsymbol{I}_{n, M}^{OOB}$. The authors in \cite{ramosaj2019consistent} proved the consistency of several estimators for $\sigma^2$, which are based on the sampling variance of residuals obtained from the Random Forest model using Out-of-Bag samples. These results enables practitioners to consistently estimate the signal-to-noise ratio given by 
\begin{align}\label{SNEstimate}
\widehat{SN}_n  &= \frac{ |\hat{\sigma}_Y^2 - \hat{\sigma}_{RF}^2 |}{\hat{\sigma}_{RF}^2},
\end{align}
where $\hat{\sigma}_Y^2$ is the sampling variance of the response $Y$ and $\hat{\sigma}_{RF}^2$ an residual variance estimator as given in \cite{ramosaj2019consistent}. In the sequel, we simply restrict our attention to the residual sampling variance estimator $\hat{\sigma}_{RF}^2 = 1/n \sum\limits_{i = 1}^n (\hat{\epsilon}_i - \bar{\hat{\epsilon}}_{n})^2$ for $\sigma^2$ as described in \cite{ramosaj2019consistent}, where $\hat{\epsilon}_ i = Y_i - m_n^{OOB}(\mathbf{X}_i)$ and $\bar{\hat{\epsilon}}_n$ is its corresponding mean.

\section{Simulation Study}

In order to provide practical evidence for the theoretical results of the previous section, we simulated artificial data and computed the empirical variable importance measure based on Out-of-Bag estimates for every variable. In doing so, several regression functions have been considered that are in line with the assumptions of the previous section. We first consider $p = 10$ covariates whose influence on $Y$ is described by means of a regression coefficient vector  $\boldsymbol{\beta}_0 = [2, 4, 2, -3, 1, 0, 0, 0 ,0, 0]^\top$. The data is then generated under the following frameworks:
\begin{enumerate}
	\item  \label{LinModel} For the simplest case, we assume a linear model, i.e. $m(\mathbf{x}_i) = \mathbf{x}_i^\top \boldsymbol{\beta}_0$, for $i = 1, \dots, n$. 
	\item Here, we assume a polynomial relationship, that is, $m(\mathbf{x}_i) = \sum\limits_{ j = 1}^p \beta_{0,j} x_{i,j}^j$ for $i = 1, \dots, n$. \label{PolyModel}
	\item In order to capture recurrent effects, a trigonometric function is assumed, i.e. $m(\mathbf{x}_i ) = 2\cdot\sin( \mathbf{x}_i^\top \boldsymbol{\beta}_0 + 2  )$ for $i = 1, \dots, n$. \label{SinModel}
	\item Finally, the effect of non-continuous functions is considered, that is 
	\begin{align*}
		m(\mathbf{x}_i) &= 	\begin{cases}
		\beta_{0,1}x_{i,1} + \beta_{0,2}x_{i, 2} + \beta_{0,3} x_{i, 3}, &\text{ if } x_{i,3} > 0.5 \\
		\beta_{0,4}x_{i, 4} + \beta_{0,5}x_{i, 5} + 3 &\text{ if } x_{i,3} \le 0.5
		\end{cases}
	\end{align*}
	for $i = 1, \dots, n$. \label{NonContModel}
\end{enumerate}
We used $MC = 1,000$ Monte-Carlo iterations to approximate the expectation of $\boldsymbol{I}_{n, M}^{OOB}$. That is, for every $m_c \in \{1, \dots, MC\}$, we generated $ \mathcal{D}_n^{m_c} =   \{ [\mathbf{X}_i^{m_c\top}, Y_i^{m_c}]^\top : i  = 1, \dots, n  \}$, where $\mathbf{X}_i^{m_c} \sim Unif([0,1]^p) $ and $Y_i = m(\mathbf{X}_i^{mc}) + \epsilon_i$ for every $i = 1, \dots, n$ and $m_c = 1, \dots, MC$. On every generated data set $\mathcal{D}_n^{mc}$, the empirical permutation importance based on Out-of-Bag samples $I_{n, M; m_c}^{OOB}(j)$, $j \in \{1, \dots, p\}$ is then computed. By the strong law of large number, we can guarantee almost surely that 
 \begin{align}
 \bar{I}_{n , M;\cdot}^{OOB}(j) := \frac{1}{MC} \sum\limits_{m_c = 1}^{MC} I_{n,M; m_c}^{OOB}(j) \longrightarrow \mathbb{E}[ I_{n, M}^{OBB}(j) ], \label{MCcomputation}
 \end{align}  
 as $MC \rightarrow \infty$,  which should give some practical insights into Theorem \ref{UnbiasednessPermutationImportance}. Different sample sizes of the form $n \in \{50,100,500,1000\}$ should also reflect the behavior of the permutation importance as prescribed in Theorem \ref{UnbiasednessPermutationImportance}. Throughout our simulations, we used $M = 1,000$ decision trees in the Random Forest model and trained it using sampling without replacement of $a_n = \lceil 2/3n \rceil < n$ data points. \\
 Regarding the noise $\epsilon$, a centered Gaussian distribution with homoscedastic variance $\sigma^2$ is assumed. As explained at the end of Section \ref{PermImpRF}, the discriminative power of the permutation importance measure clearly depends on the signal-to-noise ratio. In order to explore this effect, a signal-to-noise ratio of $SN \in \{0.5, 1,3, 5\}$ is considered. That is, the residual variance $\sigma^2$ is determined by setting $\sigma^2 = Var(m(\mathbf{X}_1)) \cdot SN^{-1}$.\\
We additionally generated data under high-dimensional settings, i.e. for $\boldsymbol{\beta}_{1} = [2,4,2,-3,1, \boldsymbol{0}^\top]^\top \in \R^{n+ 5}$ and $n \in \{50,100,500,1000\}$, we generated $\mathcal{D}_n^{m_c}$ and computed the permutation importance for every Monte-Carlo set $\mathcal{D}_n^{m_c}$. This leads to regression problems of the type $p > n$, for which Theorem \ref{UnbiasednessPermutationImportance} - unless not any of the given assumptions are violated - should also be valid.

 \subsection{Simulation Results} 

In this section, we present the simulation result for all four models $\ref{LinModel}. - \ref{NonContModel}.$ with $p = 10$ and a  sample size of $n  \in \{ 50, 1000 \} $. The results for the other sample sizes are moved to the supplement. Note that the solid black lines in the boxplots represented in Figure \ref{LinearPlot} to \ref{PlotNonCont}, refer not to the median, but to the empirical mean $ \bar{I}_{n, M;  \cdot}^{OOB}(j)$ as computed in  $(\ref{MCcomputation})$. The blue star point \textcolor{blue}{$\star$} in the plots refer to the expected value of the permutation importance based on Out-of-Bag samples. For additive models such as the linear and polynomial model, a direct computation of $\boldsymbol{I}$ could be obtained using equation $(\ref{TheoreticalQuantity})$. For non-additive link-functions, such as in the trigonometric or non-continuous case, the results given in Proposition \ref{Proposition} are used and approximated with additional $1,000$ Monte-Carlo iterations.  \\
Figure \ref{LinearPlot} gives boxplots of the permutation importance of all ten variables over all Monte-Carlo iterations for the \textbf{linear model}. It is apparent that in case of small sample sizes (left panel), the permutation importance had difficulties in clearly distinguishing informative and non-informative variables. This is in line with the asymptotic results obtained in Theorem  \ref{UnbiasednessPermutationImportance}.  The simulation results reveal that this depends on the signal-to-noise ratio and the scale of the regression coefficient, as discussed in Section \ref{PermImpRF}. For a signal-to-noise ratio less than $1$, a clear distinction was rather hard. Under the same sample size, with a signal-to-noise ratio larger than $1$, the permutation importance could distinguish informative and non-informative variables clearer. Smaller regression coefficients being close to $0$ such as $\beta_{0,5}$ resulted into lower permutation importance values. This is in line with equation $(\ref{TheoreticalQuantity})$, which results into $I(5) = \beta_{0,5}^2/6 = 1/6 \le \min\limits_{j \in \{1, \dots, 5\}} I(j)$. For larger sample sizes (right panel), the distinction power of the permutation importance is stronger making the dependence towards the signal-to-noise ratio weaker, as shown in Section \ref{PermImpRF}, considering the asymptotic of $I_{n, M}^{OOB}(j)$, $j \in \mathcal{S}$. \\
Regarding the \textbf{polynomial model}, the distinction power of the permutation importance increased, which can be extracted from Figure \ref{PlotPoly}. Under this setting, a sufficiently large signal-to-noise ratio could lead to a stronger distinction even for small sample sizes like $n = 50$ (left panel). Larger sample sizes emphasized the distinction making the selection clearer  and more independent towards the signal-to-noise ratio as shown in Section \ref{PermImpRF} by considering the cut criterion used in the Random Forest.  In addition, the empirical mean of the simulated result approached its theoretical, asymptotic counterpart $\boldsymbol{I}$ as proven in Theorem \ref{UnbiasednessPermutationImportance}. 

\begin{figure}[h!]
	\centering
	\mbox{\subfigure[$n = 50$]{\includegraphics[width=3in]{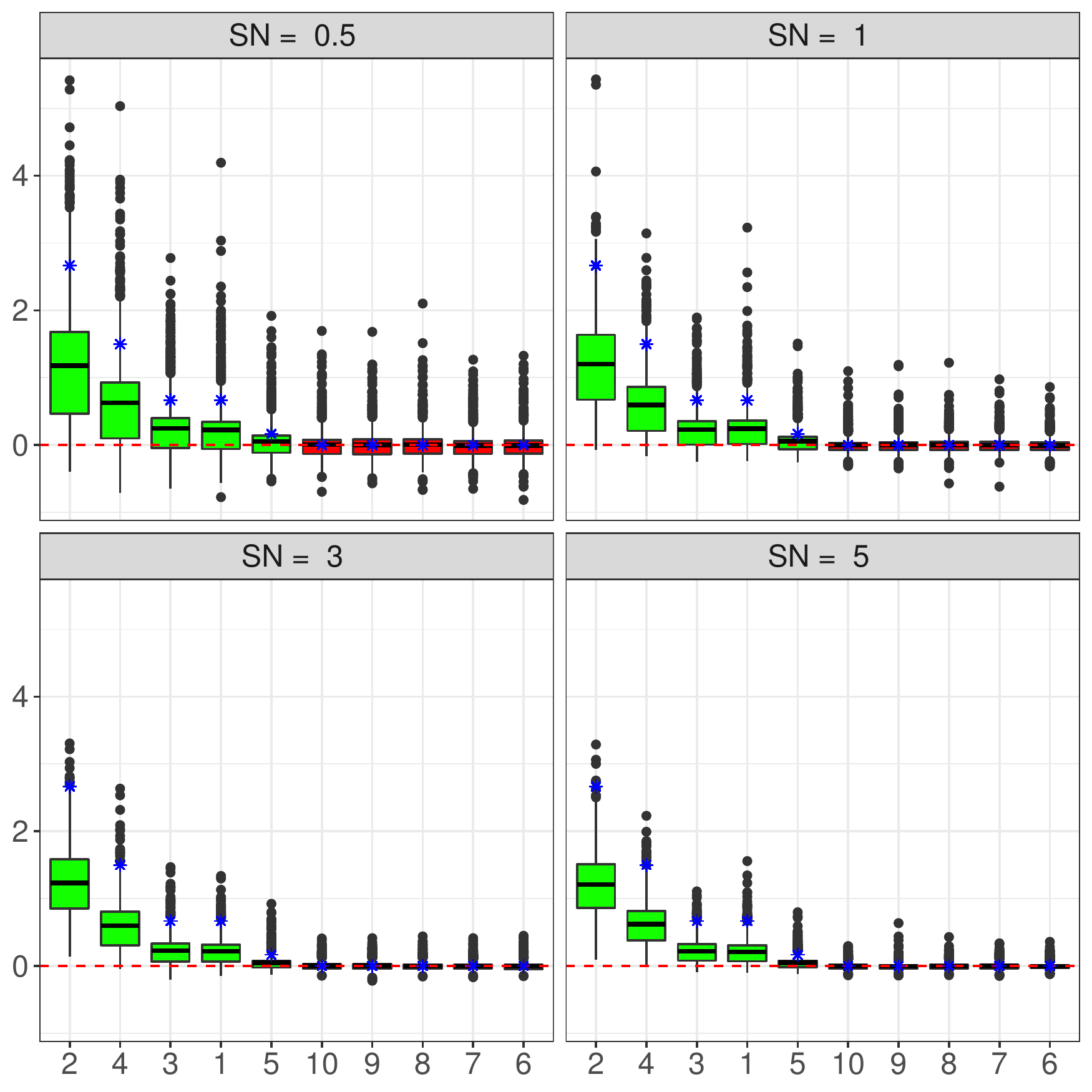}}
		\quad
		\subfigure[$n = 1,000$]{\includegraphics[width=3in]{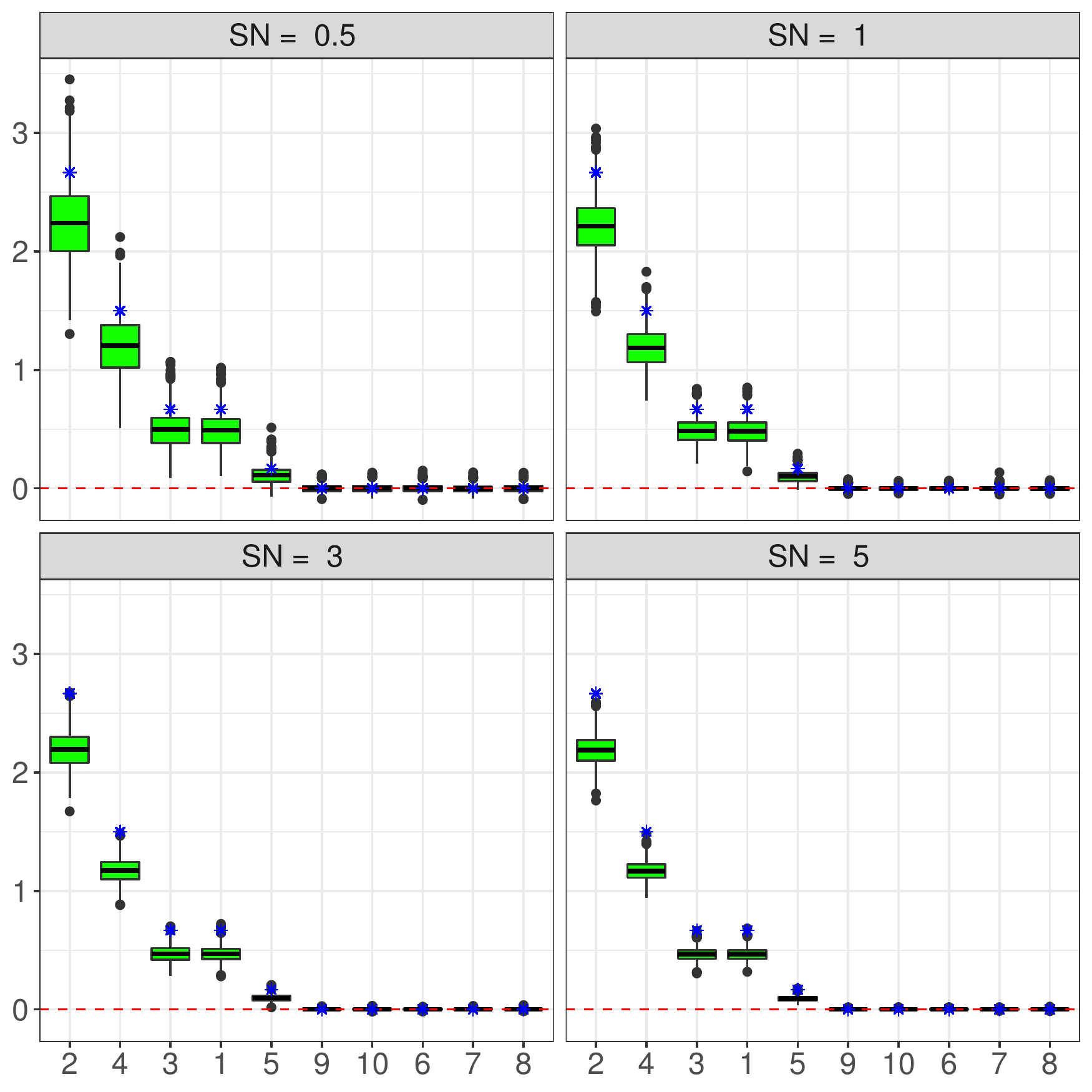} }}
	\caption{Permutation importance with various signal-to-noise ratios under a \textbf{linear model} as described in $\ref{LinModel}. $ using $MC = 1,000$ Monte-Carlo iterations with a sample size of (a) $ n  = 50$ and (b) $n = 1,000$. The solid line refers to the empirical mean $\bar{I}_{n, M; \cdot }^{OOB}$ and \textcolor{blue}{$\star$} to  its expectation. }
	\label{LinearPlot}
\end{figure}

\begin{figure*}[h]
	\centering
	\mbox{\subfigure[$n = 50$]{\includegraphics[width=3in]{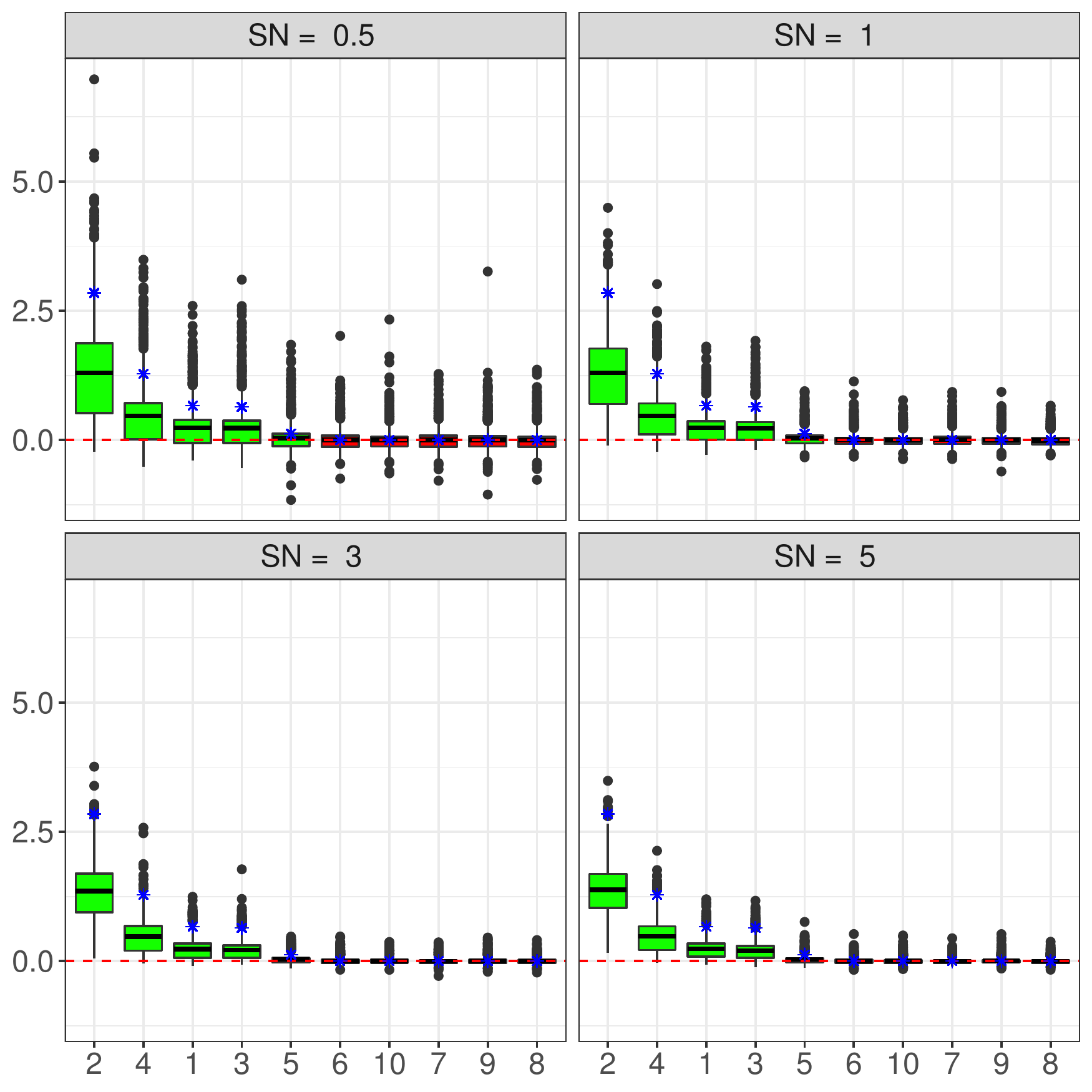}}
		\quad
		\subfigure[$n = 1,000$]{\includegraphics[width=3in]{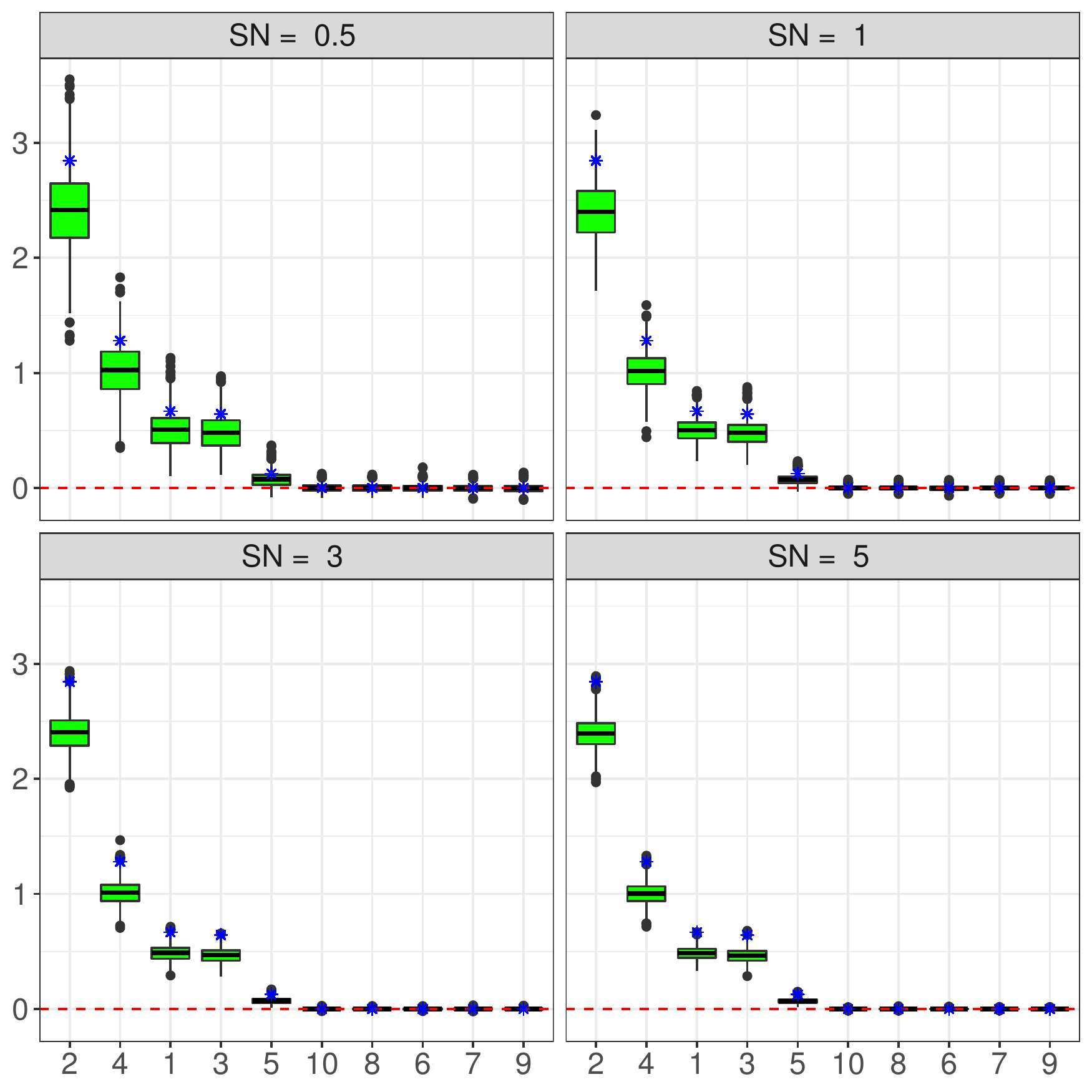} }}
	\caption{Permutation importance with various signal-to-noise ratios under a \textbf{polynomial model} as described in $\ref{PolyModel}. $ using $MC = 1,000$ Monte-Carlo iterations with a sample size of (a) $ n  = 50$ and (b) $n = 1,000$. The solid line refers to the empirical mean $\bar{I}_{n, M; \cdot }^{OOB}$ and \textcolor{blue}{$\star$} to its expectation.}
	\label{PlotPoly}
\end{figure*}
\FloatBarrier

\begin{figure*}[h!]
	\centering
	\mbox{\subfigure[$n = 50$]{\includegraphics[width=3in]{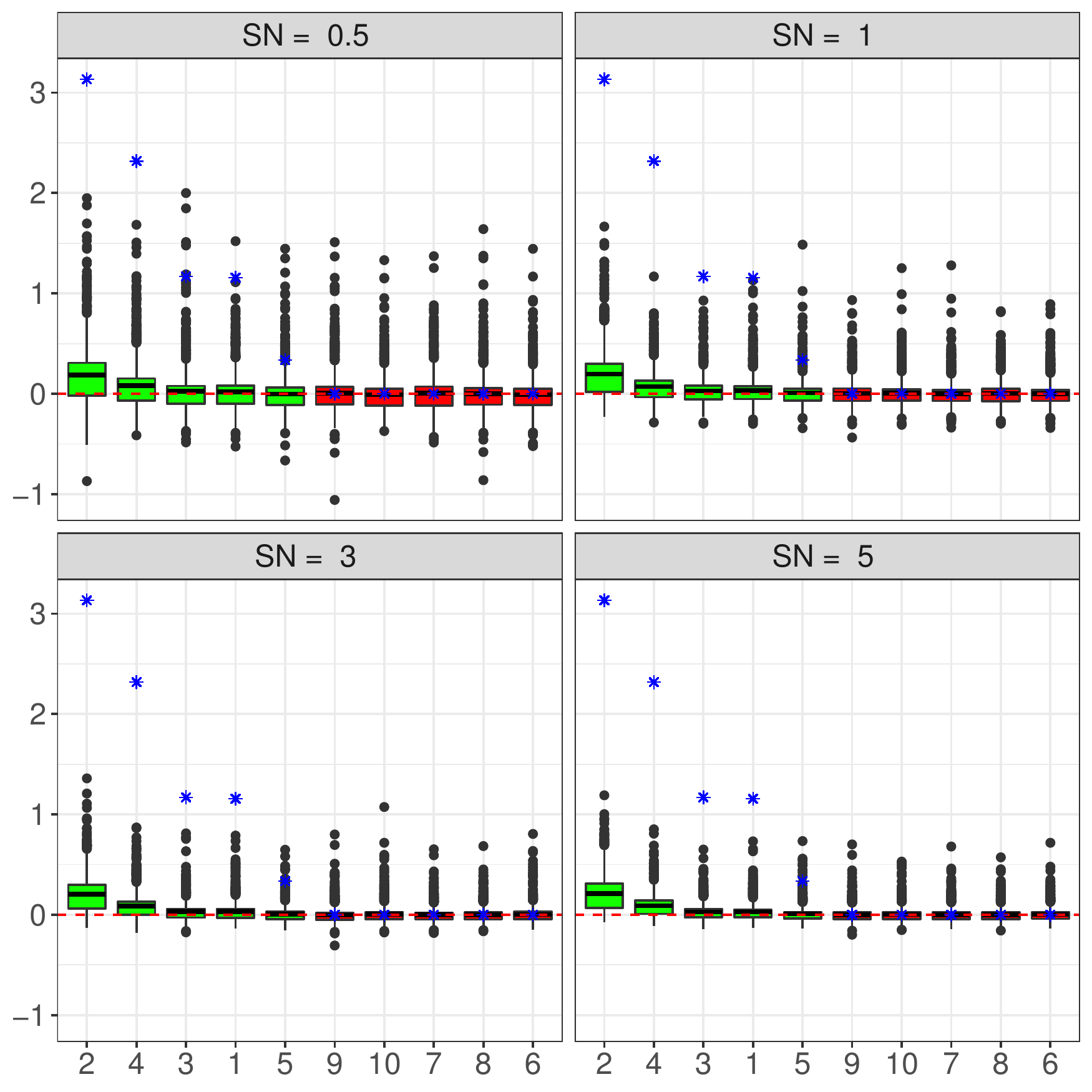}}
		\quad
		\subfigure[$n = 1,000$]{\includegraphics[width=3in]{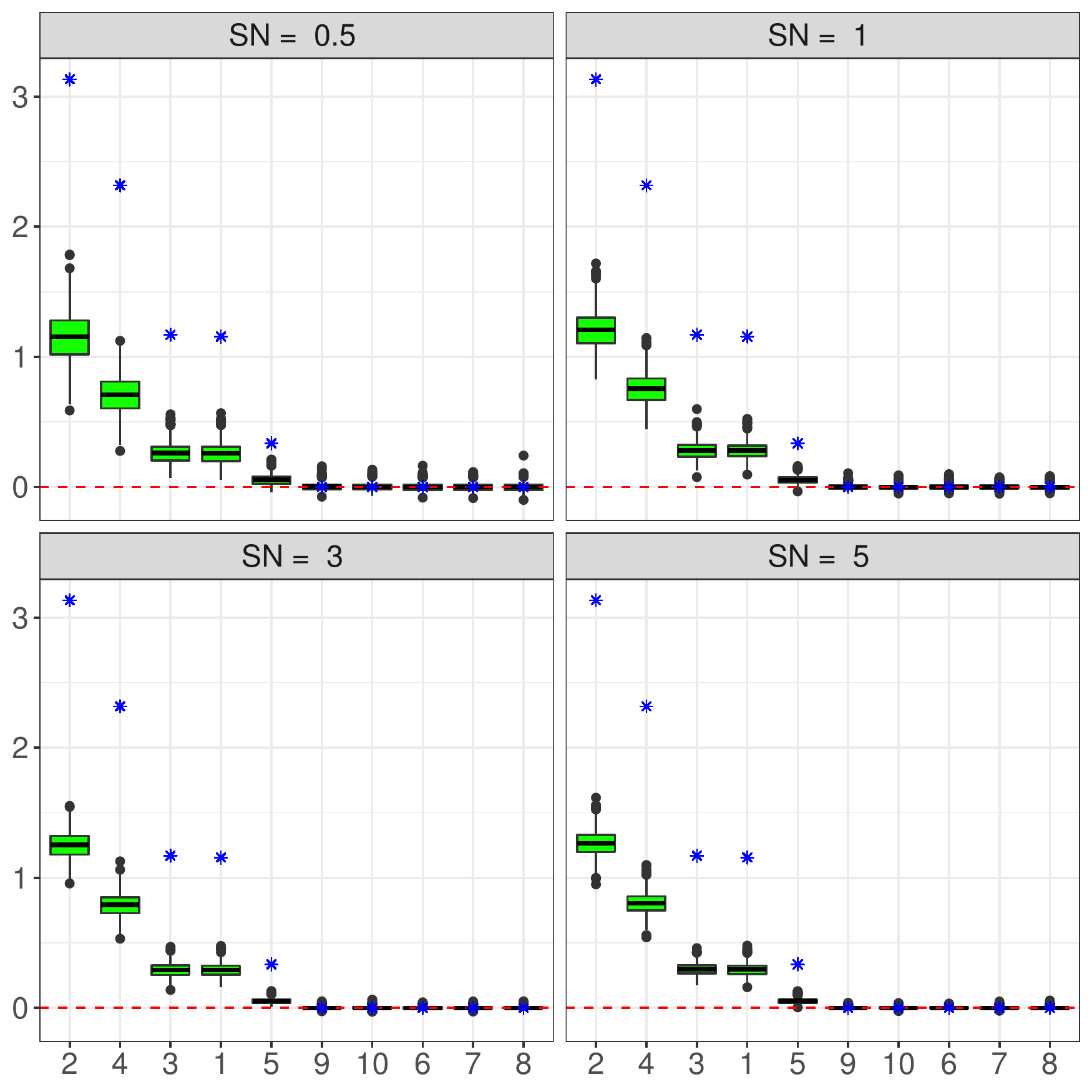} }}
	\caption{Permutation importance with various signal-to-noise ratios under a \textbf{trigonometric model} as described in $\ref{SinModel}. $ using $MC = 1,000$ Monte-Carlo iterations with a sample size of (a) $ n  = 50$ and (b) $n = 1,000$. The solid line refers to the empirical mean $\bar{I}_{n, M; \cdot }^{OOB}$ and \textcolor{blue}{$\star$} to  a Monte-Carlo approximation of its expectation. }
	\label{PlotSin}
\end{figure*}

\begin{figure*}[h!]
	\centering
	\mbox{\subfigure[$n = 50$]{\includegraphics[width=3in]{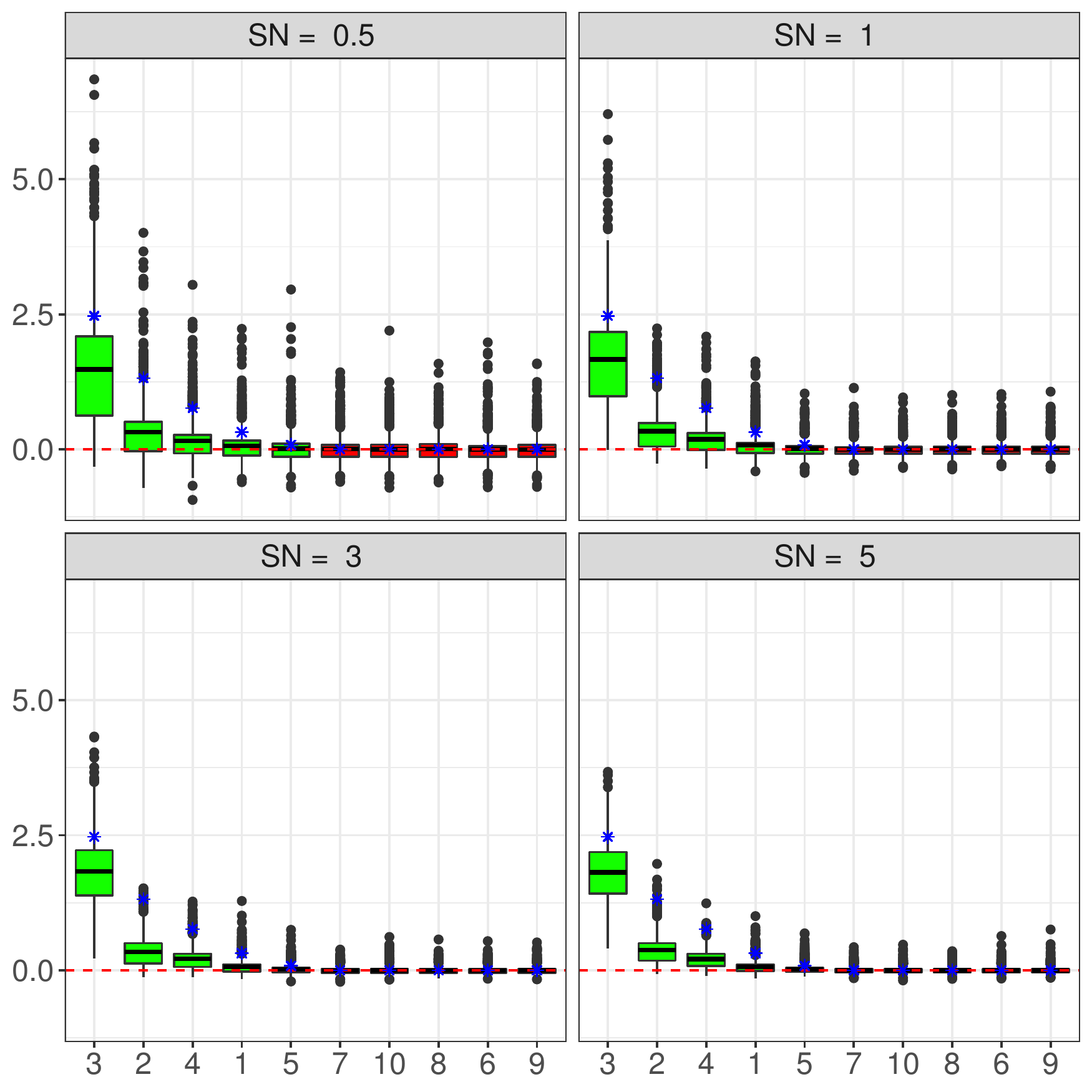}}
		\quad
		\subfigure[$n = 1,000$]{\includegraphics[width=3in]{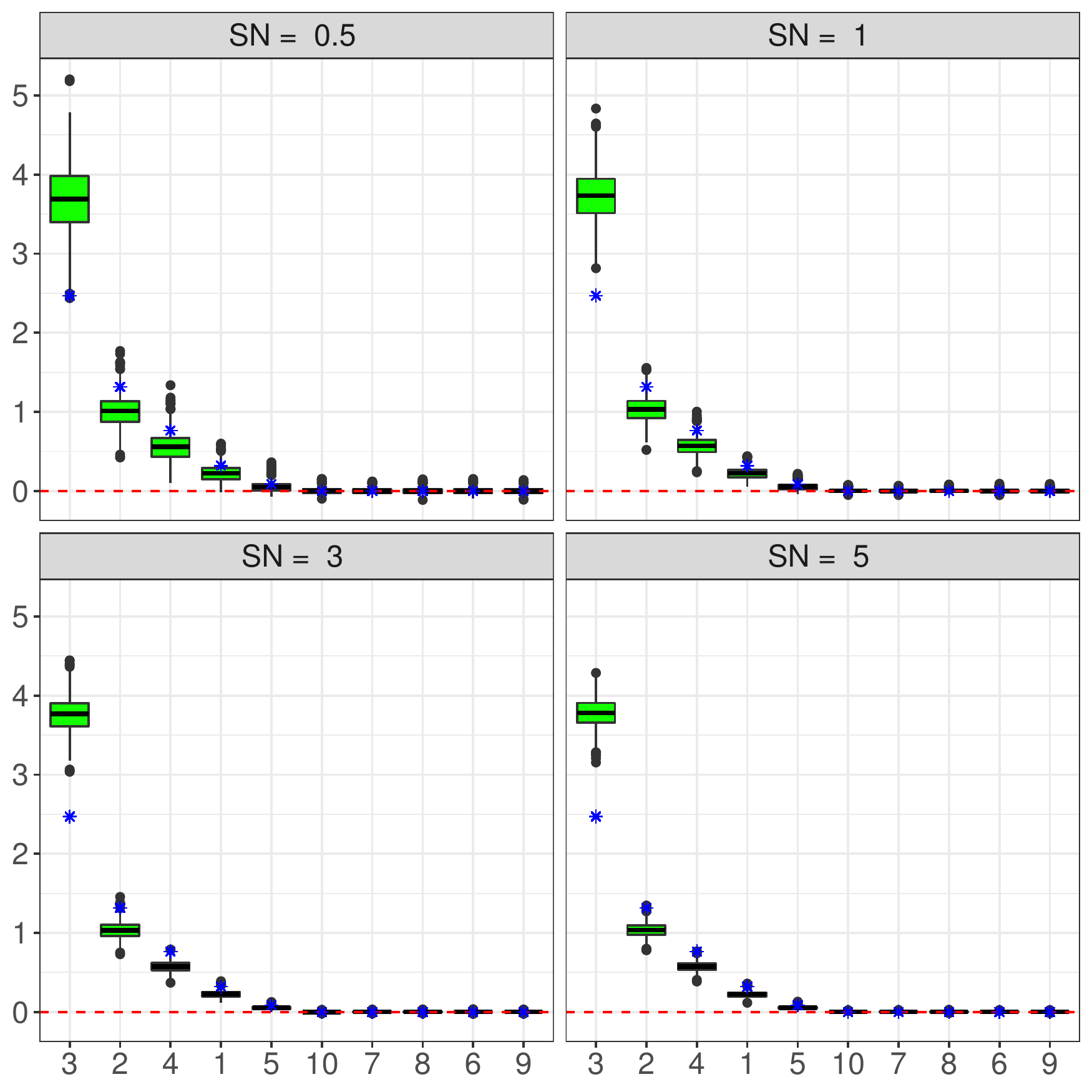} }}
	\caption{Simulation results for the permutation importance with various signal-to-noise ratios under a \textbf{non-continuous model} as described in $\ref{NonContModel}. $ using $MC = 1,000$ Monte-Carlo iterations with a sample size of (a) $ n  = 50$ and (b) $n = 1,000$. The solid line refers to the empirical mean $\bar{I}_{n, M; \cdot }^{OOB}$ and \textcolor{blue}{$\star$} to  a Monte-Carlo approximation of its expectation. }
	\label{PlotNonCont}
\end{figure*}
\FloatBarrier
For the polynomial model, we can also make use of equation $(\ref{TheoreticalQuantity})$, which will lead us to $I(j) = 2 \beta_j^2 \left( \frac{1}{2j+1}- \frac{1}{(j+1)^2} \right) $ for $j \in \mathcal{S}$. This justifies the relatively small values of $I_{n, M}^{OOB}(5)$, which should lie around $25/198 \approx0.13$.  \\
Regarding the \textbf{trigonometric} link function, the permutation importance measure lost in separating force when the sample size was relatively small. Here, a larger signal-to-noise ration was helpful, but for weak signals such as $\beta_5$, a clear distinction was rather hard. The results turned quickly into the right direction, when the sample size increased (right panel), as illustrated in Figure \ref{PlotSin}. In the latter scenario, the permutation importance was able to distinguish between elements in $\mathcal{S}$ and $\{1, \dots, p\}  \setminus \mathcal{S}$ while the empirical mean approached its theoretical counterpart $\boldsymbol{I}$. This was rather independent of the signal-to-noise ratio, as discussed in Section \ref{PermImpRF}. Note that under this model, equation $(\ref{TheoreticalQuantity})$ cannot be applied. However, it seems that a stronger or weaker signal resulted into lower or higher permutation importance.  \\
Moving to the \textbf{non-continuous} case with linear sub-functions, a stronger distinction power could be obtained compared to the linear link function. This, although equation $\ref{TheoreticalQuantity}$ is not applicable. A detailed result of the permutation importance measure under this setting can be extracted from Figure $\ref{PlotNonCont}$. There, the boxplot indicated a strong discriminative power towards non-informative variables for larger data sets, independent of the signal-to-noise ratio. The empirical mean of the simulated importance measures approached its theoretical counterpart $\boldsymbol{I}$ for an increased sample size. In addition, more importance is put on variable $3$ compared to the other frameworks. This arises from the usage of the third variable for both, the localization of the discontinuity point and its contribution to the response through the linear sub-function. However, this effect should vanish asymptotically according to Theorem \ref{UnbiasednessPermutationImportance}, as long as the assumptions are met. \\
Under all settings, it is worth to notice that the permutation importance resulted into larger variability, if the variables were informative. For non-informative variables, the Random Forest was \textit{sure} which variables were non-informative, especially when sample size increased. In fact, under all simulation settings, the RFPIM attained values very close to zero. This supports the findings in Theorem \ref{UnbiasednessPermutationImportance} for unimportant features, as the permutation importance is exactly unbiased in this case. \\

The boxplots of the permutation importance for the \textbf{high-dimensional} settings are summarized in Figures $5$ - $8$ given in the supplement. Under this framework, the linear model (see Figure $5$ in the supplement) lost in distinction power compared to $p < n$ problems, especially when the sample size was relatively small. Although $p >n$, an increase in $n$  led to an increase in separation force between variables in $\mathcal{S}$ and $\{1, \dots, p \}$ making the results clearer for $n \ge 500$. The empirical mean of the permutation importance moved closer to its theoretical counterpart $I(j)$, $j = 1, \dots, 5$. For $j \in \{6,\dots, 10\}$, they were almost exactly to zero as proven in Theorem \ref{UnbiasednessPermutationImportance}. There is also an increase in variation under the high-dimensional setting. Regarding the polynomial model (see Figure $6$ in the supplement), similar results could be obtained compared to the $p < n $ regression problem. However, the permutation importance was slightly downsized for all variables, but the distinction force was similar. Under the trigonometric function with $p >n$ (see Figure $7$ in the supplement), the permutation importance lost in separation force when the sample size was small. Evaluable results could be obtained for $n = 1,000$, but the permutation measure  was again downsized for all variables compared to its analogon under $p < n$. The simulation reveals that the convergence of the expectation is slower compared to its $p < n$ analogon. The non-continuous case (see Figure $8$ in the supplement) led to similar results than under the scenario of $p$  being less than $n$, with the exception that the permutation importance was again slightly downsized for all variables again. \\
\textbf{Final Thoughts.} Under both settings, i.e. $p < n $ and $p >n$, the permutation importance measure ranked the variables correctly according to the results given in equation $(\ref{TheoreticalQuantity})$ for the linear and polynomial model. The ranking remained the same for the trigonometric case, but was slightly changed when the sample size was rather small in high-dimensional settings. The ranking of the variables changed under the non-continuous model, where additional importance was set to variable $3$ for playing the role of a discontinuity point and its systematic influence on $Y$ through the sub-function. However, according to our findings, this effect should vanish asymptotically.

\section{Conclusion}
We proved the (asymptotic) unbiasedness of the permutation importance measure originating from the Random Forest for regression models. Our results are mainly based on assuming that features are independent, and hence uncorrelated while requiring that the Random Forest is $L_2$-consistent. Furthermore, we identified main drivers for the quality of the variable selection process such as the signal-to-noise ratio by explicitly considering the cut criterion of the Random Forest model. An extensive simulation study has been conducted for low- $(p < n)$ and high-dimensional $(p > n)$ regression frameworks. The results support our theoretical findings: even under high-dimensional settings, the permutation importance was able to correctly select among informative features, when the sample size was sufficiently large. Our findings also indicate that potential future research is worth to be conducted on $(i)$ the consistency of the involved cut-criterion and $(ii)$ the (asymptotic) distribution of the Random Forest permutation importance as a preliminary step towards the construction of valid statistical testing procedures for feature selection.

\begin{center}
	\section*{Acknowledgement}
\end{center}

We are very thankful to G\'{e}rard Biau and Erward Scornet for fruitful disucssions on Random Forest related issues during a scholary visit at the Sorbonne Universit\'{e} and the \'{E}cole Polytechnique. \\

\section{Appendix.}

In this section we state the proofs of Propositions \ref{HelpingProposition} and \ref{Proposition} and Theorem \ref{UnbiasednessPermutationImportance}. Additional proofs mentioned in the article are shifted at the end of this section.

\begin{proof}[Proof of Proposition \ref{HelpingProposition}]
	Let $i \in \{1, \dots, n \}$ be fixed and $\mathbf{X}_i \in \mathcal{D}_n$. Let $\{\boldsymbol{\Theta}_t\}_{t = 1}^M$ be the sequence of iid generic random vectors on the probability space $(\Omega_\Theta, \mathcal{F}_\Theta, \mathbb{P}_\Theta)$ being responsible for the sampling procedure and the feature sub-spacing  in the Random Forest algorithm. Note that the generic random vector can then be decomposed into $\boldsymbol{\Theta}_t = [ \boldsymbol{\Theta}_t^{(1)}, \boldsymbol{\Theta}_t^{(2)}]^\top$, where $\boldsymbol{\Theta}_t^{(1)} \in \{0,1\}^n$ indicates whether a certain observation has been selected in tree $t$ and $\boldsymbol{\Theta}_t^{(2)}$ models feature sub-spacing. Furthermore, denote with $Z_i = Z_i(M)$ the number of the $M$ regression trees not containing the $i$-th observation. Then we can conclude that 
	
	\begin{align*}
		Z_i(M) \sim Bin(M, c_n), \quad \text{ where } \quad c_n  = \begin{cases}
		 1- a_n/n  &\text{ for subsampling,} \\
		 (1 - 1/n)^n & \text{ for bootstrapping with replacement,}
		\end{cases}
	\end{align*}
	with $c_n > 0$. Since $Z_i(M) = \sum\limits_{ \ell = 1}^M B_\ell$, with $B_\ell \sim Bernoulli(c_n)$ independent and identically distributed under $\mathbb{P}_{\Theta}$, it follows by the strong law of large numbers that $ V_ {n, M} := Z_i(M)/M \stackrel{a.s.}{\longrightarrow} \mathbb{E}[ B_1 ] = c_n$, as $M \rightarrow \infty$. This implies that $Z_i(M) \stackrel{a.s.}{\longrightarrow} \infty$, as $M \rightarrow \infty$. Assuming without loss of generality that the first $Z_i(M)$ decision trees do not contain the $i$-th observation, this will yield to 
	\begin{align}
		R_{n, M} := \frac{1}{Z_i(M)} \sum\limits_{ t  = 1}^{Z_i(M)} m_{n, 1}(\mathbf{X}_i; \boldsymbol{\Theta}_t,  \mathcal{D}_n)  \longrightarrow m_n^{OOB}(\mathbf{X}_i) \quad \mathbb{P}_\Theta - a.s. \text{ as } M \rightarrow \infty,
	\end{align}
	where $m_n^{OOB}(\mathbf{X}_i) = \mathbb{E}_{ \boldsymbol{\Theta}_{[i]} }[ m_{n, 1}(\mathbf{X}_i; \boldsymbol{\Theta}_{[i]}, \mathcal{D}_n) ]$ with $\boldsymbol{\Theta}_{[i]} = [\boldsymbol{\Theta}^{(1)}, \Theta^{(2)}]$, such that $\Theta_{i}^{(1)} = 0$. 
	Now, let $\mathbf{K}_{n, M} = [V_{n, M}, R_{n, M}]^\top \in \R^2$ and set $N = N_1 \cup N_2$, where $N_1 = \{ \omega \in \Omega_\Theta: V_{n, M}(\omega) \nrightarrow c_n \}$ and $N_2 = \{ \omega \in\Omega_\Theta : R_{n, M}(\omega) \nrightarrow m_n^{OOB}(\mathbf{X}_i) \}$. Since $\mathbb{P}_\Theta(N_1) = \mathbb{P}_\Theta(N_2) = 0$, it follows immediately that $0 \le \mathbb{P}_\Theta(N)  = \mathbb{P}_\Theta(N_1) + \mathbb{P}_\Theta(N_2) = 0$, i.e. $N$ is a null-set. Hence, 
	\begin{align}\label{ConvInMsense}
		\mathbf{K}_{n, M} \longrightarrow  [c_n, m_n^{OOB}(\mathbf{X}_i)]^\top,  \quad  \mathbb{P}_\Theta - \text{ almost-surely  as } M \rightarrow \infty. 
	\end{align}

	  Since $\{ \boldsymbol{\Theta}_t \}_{t = 1}^M$ is a sequence of iid random variables, we can again assume without loss of generality, that the first $ Z_i(M)$ do not contain the $i$-th observation. Therefore, we can conclude that 
	\begin{align}\label{InfForestModification}
		 \frac{1}{M} \sum\limits_{ t= 1}^M m_{n , 1}(\mathbf{X}_i; \boldsymbol{\Theta}_t ) \mathds{1}\{ \mathbf{X}_i \text{ has not been selected} \} &= \frac{Z_i(M)}{M} \frac{1}{Z_i(M)} \sum\limits_{ t  = 1}^{Z_i(M)} m_{n, 1}(\mathbf{X}_i; \boldsymbol{\Theta}_t) \notag  \\
		&\longrightarrow c_n \cdot m_{n}^{OOB}(\mathbf{X}_i),
	\end{align}
	$\mathbb{P}_\Theta -  \text{ almost-surely}$ as $M \rightarrow \infty$. The convergence follows by applying the continuous mapping theorem on the function $g(x,y) = x \cdot y$ using $\mathbf{K}_{n, M}$ and $(\ref{ConvInMsense})$. 
\end{proof}

\begin{proof}[Proof of Proposition \ref{Proposition}]
	Let $\mathbf{X} = [X_1, \dots, X_p]^\top \in \R^p$ be an independent copy of $\mathbf{X}_1$ such that $Y = \widetilde{m}(\mathbf{X}) + \epsilon$ as in regression model $(\ref{RegModel})$. Furthermore, Let $j \in \{1, \dots, p\} \setminus \mathcal{S}$, i.e. $j$ is non-informative. According to our definition of being non-informative and the assumption that there are no dependencies among the features $\{X_j\}_{j = 1}^p$, this will lead us to $Y$ being indepdendent of $X_j$, while $X_j$ is also independent towards all other features $X_\ell$, $\ell \neq j  \in \{1, \dots, p\}$. Denoting with $\mathbf{X}_j = [X_1, \dots, X_{j-1},Z_j, X_{j+1} , \dots, X_p]^\top \in \R^p$, while $Z_j$ is an independent copy of $X_j$, independent of $X_\ell$ and $Y$ for all $\ell \neq j$, this will yield to $[\mathbf{X}_j^\top, Y]^\top \stackrel{d}{=} [\mathbf{X}^\top, Y]^\top$. Hence, we will obtain 
	\begin{align}
		I(j) &= \mathbb{E}[(Y -  \widetilde{m}(\mathbf{X}_j))^2] - \mathbb{E}[( Y - \widetilde{m}(\mathbf{X}))^2] \notag \\
		&= \mathbb{E}[ (Y - \widetilde{m}(\mathbf{X}))^2 ] - \mathbb{E}[ (Y - \widetilde{m}(\mathbf{X}))^2] \notag \\	
		&= 0. 
	\end{align} 
	
	On the other hand, if $j \in \mathcal{S}$, i.e. $j $ is informative, than we can deduce the following computations, where the third equation follows from the independence of $\mathbf{X}$ and $\epsilon$ together with $\mathbb{E}[\epsilon] = 0$. The second last equality follows from assumption $\ref{Ass1}$ leading to $\mathbf{X}_j \stackrel{d}{= } \mathbf{X}$. 
	\begin{align}
		I(j) &= \mathbb{E}[(Y - \widetilde{m}(\mathbf{X}_j))^2] - \mathbb{E}[(Y - \widetilde{m}(\mathbf{X}))^2] \notag \\
		&= \mathbb{E}[ ( Y - \widetilde{m}(\mathbf{X}) + \widetilde{m}(\mathbf{X}) - \widetilde{m}(\mathbf{X}_j) )^2 ] - \mathbb{E}[(Y - \widetilde{m}(\mathbf{X}))^2] \notag \\
		&= \mathbb{E}[ ( \widetilde{m}(\mathbf{X}) - \widetilde{m}(\mathbf{X}_j) )^2 ] + 2\mathbb{E}[ \epsilon ( \widetilde{m}(\mathbf{X}) - \widetilde{m}(\mathbf{X}_j)) ) ] \notag \\
		&= \mathbb{E}[ ( \widetilde{m}(\mathbf{X}) - \widetilde{m}(\mathbf{X}_j) )^2 ].
	\end{align}
\end{proof}

\begin{proof}[Proof of Theorem \ref{UnbiasednessPermutationImportance}]
		Let $j \in \{1, \dots, p\}$, $i \in \{1, \dots, n\}$ and $t \in \{1, \dots, M\}$ be fixed but arbitrary and assume that the Random Forest sampling mechanism is restricted to sampling $a_n \in \{1, \dots, n\}$ points without replacement such that $a_n < n$. Denote with $\mathcal{D}_n^{(t)}$ the collection of points selected for tree $t \in \ {1,\dots, M}$. Then we denote with $\mathcal{D}_n^{-(t)} = \mathcal{D}_n \setminus \mathcal{D}_n^{(t)}$ the subset of $\mathcal{D}_n$ in tree $t \in \{ 1, \dots, M \}$ with cardinality $\gamma_n$ for which its elements have not been selected during the sampling procedure. Note that the cardinality of $\mathcal{D}_n^{-(t)}$ remains fixed for all $t = 1, \dots, M$ and is given by $\gamma_n = n - a_n$, which is different to sampling with replacement.  In addition, we set $\mathcal{D}_{n, \mathbf{X}}^{-(t)} = \{ \mathbf{X}_i : [ \mathbf{X}_i^\top, Y_i ]^\top \in \mathcal{D}_n^{-(t)} \}$ to be the set of all features $\mathbf{X}$ that belong to $\mathcal{D}_n^{-(t)}$, i.e. that have been selected during resampling. Then we recall from $(\ref{ImportanceDefinition})$ that the permutation variable importance based on OOB estimates is given by 
		\begin{align}\label{PermutationImportance}
			I_{n, M}^{OOB}(j) &= \frac{1}{M \gamma_n} \sum\limits_{t = 1}^M  \sum\limits_{i \in \mathcal{D}_n^{-(t) } }\left\{  (Y_i  - m_{n, 1}(\mathbf{X}_i^{\pi_{j,t} };\boldsymbol{ \Theta}_t ) )^2   - (Y_i -  m_{n, 1}(\mathbf{X}_i; \boldsymbol{ \Theta}_t ) )^2  \right\} \notag \\
			&= \frac{1}{M \gamma_n} \sum\limits_{t = 1}^M  \sum\limits_{i  = 1}^n\left\{  (Y_i  - m_{n, 1}(\mathbf{X}_i^{\pi_{j,t} }; \boldsymbol{ \Theta}_t ) )^2   - (Y_i -  m_{n, 1}(\mathbf{X}_i; \boldsymbol{\Theta}_t ) )^2  \right\} \mathds{1}\{ \mathbf{X}_i \in \mathcal{D}_n^{-(t)}  \},
		\end{align}
		where $\pi_{j, t} $ is a \textit{real} permutation of the $j$-th covariable in $\mathcal{D}_{n, \mathbf{X}}^{-(t)}$, where we call a permutation as \textit{real}, if $\pi_{j, t} \in \{ \pi \in \mathcal{S}_{\gamma_n} : \pi(i) \neq i  \} =: \mathcal{V}$ and $\mathcal{S}_{\gamma_n}$ is the symmetric group.   Although we did not yet specify the dependence of $\mathcal{D}_n^{(t)}$ and $\mathcal{D}_n^{-(t)}$ towards the generic random vector $\boldsymbol{\Theta}_t$ in the Random Forest mechanism, it is worth to notice that in fact, $\mathcal{D}_n^{(t)} = \mathcal{D}_n^{(t)}(\boldsymbol{\Theta}_t)$ and $\mathcal{D}_n^{-(t)} =  \mathcal{D}_{n}^{-(t)} (\boldsymbol{\Theta}_t) $. \\
		Then, the following results can be obtained:
		\begin{align} \label{First}
			\mathbb{E}[ (Y_i - \widetilde{m}(\mathbf{X}_i))^2 \mathds{1}\{ \mathbf{X}_i \in \mathcal{D}_{n, \mathbf{X}}^{-(t)} \} ] &= \mathbb{E}[ \mathbb{E} [ (Y_i - \widetilde{m}(\mathbf{X}_i))^2 \mathds{1}\{ \mathbf{X}_i \in \mathcal{D}_{n, \mathbf{X}}^{-(t)} \} | \mathcal{D}_n ] ] \notag \\
			&= \mathbb{E}[ (Y_i - \widetilde{m}(\mathbf{X}_i))^2 \mathbb{P}[ \mathbf{X}_i \in \mathcal{D}_{n, \mathbf{X}}^{-(t)}(\Theta_t)  | \mathcal{D}_n]  ] \notag\\
			&=  \mathbb{E}[ (Y_i - \widetilde{m}(\mathbf{X}_i))^2 (1 -  \mathbb{P}[ \mathbf{X}_i \notin \mathcal{D}_{n, \mathbf{X}}^{-(t)}(\Theta_t) | \mathcal{D}_n ] )   ]  \notag \\
			&= \left( 1 - \frac{{n- 1 \choose a_n - 1}}{ { n \choose a_n } } \right) \mathbb{E}[ (Y_i  - \widetilde{m}(\mathbf{X}_i))^2 ]  \notag \\
			&=  \frac{n - a_n}{n} \mathbb{E}[ (Y_i  - \widetilde{m}(\mathbf{X}_i))^2 ]
		\end{align}
		The second equality follows from the measurability of $(Y_i - m(\mathbf{X}_i))$ and $ \mathbb{P}[ \mathbf{X}_i \notin \mathcal{D}_{n, \mathbf{X}}^{-(t)}(\Theta_t) | \mathcal{D}_n]$ is the probability of not selecting a fixed observation $i$ among $n$ elements, when resampling is conducted without replacement. \\
		
		Returning to the sequence of iid generic random vectors $\{ \boldsymbol{\Theta}_t \}_{t = 1}^M$, we recall that we can separate each generic random vector into $\boldsymbol{\Theta}_t  = [ \boldsymbol{\Theta}_t^{(1)}, \boldsymbol{\Theta}_t^{(2)}]$, where $\boldsymbol{\Theta}_t^{(1)}$ models the sampling mechanism prior to tree construction and $\boldsymbol{\Theta}_t^{(2)}$ is the random variable modeling feature sub-spacing during the tree construction. Note that in case of $m_{try} = p$, it follows that $\boldsymbol{\Theta}_t = \boldsymbol{\Theta}_t^{(1)}$.  Furthermore, $\boldsymbol{\Theta}_t^{(1)}$ can be decomposed into 
		
		\begin{align}\label{Decomposition}
			\boldsymbol{\Theta}_t^{(1)} &= [\Theta_{1,t}^{(1)}, \dots, \Theta_{n, t}^{(1)} ]^\top \in \{ 0,1\}^n,
		\end{align}
		where each entry $\Theta_{\ell, t}^{(1)}$, $1 \le \ell \le n$ is Bernoulli distributed indicating whether observation $\ell$ has been selected during the sampling procedure. For sampling without replacement the sequence $\{ \Theta_{\ell, t}^{(1)} \}_{\ell = 1}^n$ does not consist of independent random variables. However, it holds that $\sum\limits_{ \ell = 1}^n \Theta_{\ell, t} = a_n$ and that $\boldsymbol{\Theta}_t^{(1)}$ is independent of $(\mathbf{X}_i, Y_i, \boldsymbol{\Theta}_t^{(2)})$ for all $t = 1, \dots, M$ and all $i = 1, \dots, n$. Let $\Delta_n(\mathbf{X}_i, Y_i, \boldsymbol{\Theta}_t)  = \Delta_n(\mathbf{X}_i, Y_i, \boldsymbol{\Theta}_t^{(1)}, \boldsymbol{\Theta}_t^{(2)}) := (\widetilde{m}(\mathbf{X}_i) - m_{n, 1}(\mathbf{X}_i, \boldsymbol{\Theta}_t^{(1)}, \boldsymbol{\Theta}_t^{(2)}))^2$, declare $\mathbf{X}_i '$ as an independent copy of $\mathbf{X}_i$ independent of $m_{n, 1}$ and set $\mathcal{G} = \{ [v_1, \dots, v_n]^\top \in \{0,1\}^n : v_1 + \dots + v_n = a_n  \}$ and $\mathcal{G}_i := \{ \mathbf{v} \in \mathcal{G} : v_i = 0 \}$. Then we observe the following equality
		\begin{align}\label{ChangeIndependence}
			\mathbb{E}[ \Delta_n(\mathbf{X}_i, Y_i, \boldsymbol{\Theta}_t^{(1)},\boldsymbol{\Theta}_t^{(2)}) \mathds{1}\{ \Theta_{i, t}^{(1)} = 0 \} ] &= \sum\limits_{ \boldsymbol{\ell} \in \mathcal{G}} \mathbb{E}[ \Delta_n(\mathbf{X}_i, Y_i, \boldsymbol{\Theta}_t^{(1)}, \boldsymbol{\Theta}_t^{(2)}) \mathds{1}\{ \Theta_{i, t}^{(1)} = 0 \} | \boldsymbol{\Theta}_t^{(1)} = \boldsymbol{\ell} ] \cdot \mathbb{P}[\boldsymbol{\Theta}_t^{(1)} = \boldsymbol{\ell}] \notag \\
			&= \sum\limits_{ \boldsymbol{\ell} \in \mathcal{G}_i} \mathbb{E}[ \Delta_n(\mathbf{X}_i, Y_i, \boldsymbol{\Theta}_t^{(1)}, \boldsymbol{\Theta}_t^{(2)}) |  \boldsymbol{\Theta}_t^{(1)} = \boldsymbol{\ell} ]  \cdot \mathbb{P}[ \boldsymbol{\Theta}_t^{(1)} = \boldsymbol{\ell}] \notag \\
			&= \sum\limits_{ \boldsymbol{\ell} \in \mathcal{G}_i  } \mathbb{E}[ \Delta_n(\mathbf{X}_i', Y_i', \boldsymbol{\Theta}_t^{(1)}, \boldsymbol{\Theta}_t^{(2)}) | \boldsymbol{\Theta}_t^{(1)} = \boldsymbol{\ell} ] \cdot \mathbb{P}[ \boldsymbol{\Theta}_t^{(1)} = \boldsymbol{\ell} ]  \notag\\ 
			&= \mathbb{E}[ \Delta_n(\mathbf{X}_i', Y_i', \boldsymbol{\Theta}_t^{(1)}, \boldsymbol{\Theta}_t^{(2)}) \mathds{1}\{ \Theta_{i, t}^{(1)} = 0 \} ],			
		\end{align}
		where the second last equality follows from the independence of $\boldsymbol{\Theta}_{t}^{(1)}$ and $ (\mathbf{X}_i, Y_i, \boldsymbol{\Theta}_t^{(2)}) $ and $(\mathbf{X}_i, Y_i, \boldsymbol{\Theta}_t^{(1)}, \boldsymbol{\Theta}_t^{(2)}) \stackrel{d}{=} (\mathbf{X}_i', Y_i', \boldsymbol{\Theta}_t^{(1)}, \boldsymbol{\Theta}_t^{(2)} )$. Now, using $(\ref{ChangeIndependence})$, we obtain
		
		\begin{align}\label{Second}
			0 \le \mathbb{E}[ (\widetilde{m}(\mathbf{X}_i) - m_{n, 1}(\mathbf{X}_i; \boldsymbol{\Theta}_t) )^2 \mathds{1}\{ \mathbf{X}_i  \in \mathcal{D}_{n, \mathbf{X}}^{-(t)}  \} ] \notag &= \mathbb{E}[ \Delta_n(\mathbf{X}_i, Y_i,  \boldsymbol{\Theta}_t^{(1)}, \boldsymbol{\Theta}_t^{(2)} ) \mathds{1}\{ \mathbf{X}_i \in \mathcal{D}_{n, \mathbf{X}}^{-(t)} \} ] \notag \\ 
			&= \mathbb{E}[ \Delta_n( \mathbf{X}_i, Y_i, \boldsymbol{\Theta}_t^{(1)}, \boldsymbol{\Theta}_t^{(2)} ) \mathds{1}\{ \Theta_{i, t}^{(1)} = 0 \} ] \notag \\
			&= \mathbb{E}[ \Delta_n(\mathbf{X}_i', Y_i', \boldsymbol{\Theta}_t^{(1)}, \boldsymbol{\Theta}_t^{(2)}) \mathds{1}\{ \Theta_{i, t}^{(1)} = 0 \} ] \notag \\
			&= \mathbb{E}[   \Delta_n(\mathbf{X}_i', Y_i', \boldsymbol{\Theta}_t^{(1)}, \boldsymbol{\Theta}_t^{(2)}) \mathds{1}\{ \mathbf{X}_i \in \mathcal{D}_{n, \mathbf{X}}^{-(t)} \} ] \notag \\
			&= \mathbb{E}[ (\widetilde{m}(\mathbf{X}_i') - m_{n, 1}(\mathbf{X}_i'; \boldsymbol{\Theta}_t))^2 \mathds{1} \{ \mathbf{X}_i \in \mathcal{D}_{n, \mathbf{X}}^{-(t)} \} ] \notag \\
			&=: C_{n, i, t}.
		\end{align}
		
		Note that the random tree estimate $m_{n, 1}(\mathbf{X}_i'; \boldsymbol{\Theta}_t)$ can be rewritten into 
		\begin{align}\label{AlternativeRandomTree}
			m_{n, 1}(\mathbf{X}_i'; \boldsymbol{\Theta}_t) &= \sum\limits_{j = 1}^n W_{n, j}(\mathbf{X}_i'; \boldsymbol{\Theta}_t) Y_j, 
		\end{align}
		where $W_{n, j}(\mathbf{X}_i'; \boldsymbol{\Theta}_t) = \frac{\mathds{1}\{ \mathbf{X}_j \in A_{n}(\mathbf{X}_i'; \boldsymbol{\Theta}_t)  \}}{N_n(A_n(\mathbf{X}_i'; \boldsymbol{\Theta}_t ) )} $ with $A_n(\mathbf{X}_i'; \boldsymbol{\Theta}_t)$ being the hyper-rectangular cell containing $\mathbf{X}_i'$ under the random tree constructed by $\boldsymbol{\Theta}_t$ and $N_n(A_n(\mathbf{X}_i'; \boldsymbol{\Theta}_t))$ the number of observations falling in that hyper-rectangular cell. This way, one can deduce that $0 \le W_{n, j}(\mathbf{X}_{i}'; \boldsymbol{\Theta}_t) \le 1$ for all $j = 1, \dots, n$ and $\sum\limits_{j = 1}^n W_{n, j}(\mathbf{X}_i'; \boldsymbol{\Theta}_t) = 1$. Since $ K :=\sup\limits_{\mathbf{x}} |  \widetilde{m}(\mathbf{x})| <\ \infty$ by $\ref{Ass2}$ one obtains $\mathbb{E}[Y_1^2] = \mathbb{E}[\widetilde{m}(\mathbf{X}_1)^2] + \sigma^2 < K^2  + \sigma^2 < \infty$ and together with the Cauchy-Schwarz inequality, it holds for all $n \in \N$ that 
		\begin{align}\label{Finiteness}
			C_{n, i, t} &\le \mathbb{E}[ ( | \widetilde{m}(\mathbf{X}_i') - m_{n, 1}(\mathbf{X}_i'; \boldsymbol{\Theta}_t )| )^2  ]  \le \mathbb{E}[ ( |\widetilde{m}(\mathbf{X}_i')| + |m_{n, 1}(\mathbf{X}_i'; \boldsymbol{\Theta}_t )| )^2 ] \notag  \\
			&\le K^2 +2 K \cdot \mathbb{E}\left[  \left|\sum\limits_{j = 1}^n  W_{n, j}(\mathbf{X}_i'; \boldsymbol{\Theta}_t) Y_j \right|^2 \right]^{1/2}  + \mathbb{E}\left[ \left|\sum\limits_{j = 1}^n  W_{n,j}(\mathbf{X}_i'; \boldsymbol{\Theta}_t) Y_j \right|^2 \right] \notag \\
			&\le K^2 +2K \cdot \mathbb{E}\left[ \left( \sum\limits_{j = 1}^n   W_{n, j}(\mathbf{X}_i' ; \boldsymbol{\Theta}_t) |Y_j| \right)^2 \right]^{1/2} + \mathbb{E}\left[ \left(\sum\limits_{j = 1}^n W_{n, j}(\mathbf{X}_i'; \boldsymbol{\Theta}_t) |Y_j|\right)^2 \right] \notag \\
			&\le K^2 + 2K\cdot \mathbb{E}\left[ \left(\sum\limits_{j = 1}^n |Y_j| \right)^2 \right]^{1/2} + \mathbb{E}\left[ \left( \sum\limits_{j = 1}^n |Y_j|  \right)^2 \right] \notag \\
			&\le K^2 + 2K n (\mathbb{E}[ Y_1^2 ])^{1/2} + n^2 \mathbb{E}[ Y_1^2 ]			 < \infty.
 		\end{align}

		Set $\Delta_{n, i }(\boldsymbol{\Theta}_t) = \widetilde{m}(\mathbf{X}_i) - m_{n, 1}( \mathbf{X}_i; \boldsymbol{\Theta}_t )$ and recall that $\epsilon_i = Y_i - \widetilde{m}(\mathbf{X}_i)$ according to model $(\ref{RegModel})$. Then it follows from the law of total probability that 
		\begin{align}\label{Third}
			\mathbb{E}[ (Y_i - \widetilde{m}(\mathbf{X}_i))(\widetilde{m}(\mathbf{X}_i) - m_{n, 1}(\mathbf{X}_i; \boldsymbol{\Theta}_t))  \mathds{1}\{ \mathbf{X}_i \in \mathcal{D}_{n, \mathbf{X}}^{-(t)} \}] \notag &= \mathbb{E}[ \epsilon_i \cdot \Delta_{n,i}(\boldsymbol{\Theta}_t)  \mathds{1}\{ \mathbf{X}_i \in \mathcal{D}_{n, \mathbf{X}}^{-(t)} \}]  \notag\\
			&= \mathbb{P}[ \Theta_{i, t}^{(1)} = 0 ] \cdot \mathbb{E}[ \epsilon_i \cdot \Delta_{n,i}(\boldsymbol{\Theta}_t) |\Theta_{i, t}^{(1)} = 0 ]    \notag\\
			&= \frac{\gamma_n}{n} \cdot \mathbb{E}[ \epsilon_i | \Theta_{i, t}^{(1)} = 0 ]  \cdot \mathbb{E}[ \Delta_{n,i}(\Theta_t)  | \Theta_{i, t}^{(1)} = 0 ] \notag \\
			&= 0,
		\end{align}
		since given the condition $\mathbf{X}_i \in \mathcal{D}_{n, \mathbf{X}}^{-(t)}$, or equivalently, $\Theta_{i, t}^{(1)} = 0$, $\epsilon_i$ is independent of $\Delta_{n, i}(\boldsymbol{\Theta}_t) $. Furthermore note that we used the independence of $\epsilon_i $ towards $\mathbf{X}_i$ and $\Theta_{i, t}^{(1)}$ leading to $\mathbb{E}[ \epsilon_i | \Theta_{i, t}^{(1)} = 0 ] = \mathbb{E}[\epsilon_i] = 0$. \\
		
Hence, combining the results from $(\ref{First})$, $(\ref{Second})$ and $(\ref{Third})$, we obtain
\begin{align}\label{FinalFirst}
	\mathbb{E}[ (Y_i - m_{n, 1}(\mathbf{X}_i; \boldsymbol{\Theta}_t))^2 \mathds{1}\{ \mathbf{X}_i \in \mathcal{D}_{n, \mathbf{X}}^{-(t)} \} ] &= \mathbb{E}[ (Y_i - \widetilde{m}(\mathbf{X}_i))^2 \mathds{1}\{ \mathbf{X}_i \in \mathcal{D}_{n, \mathbf{X}}^{-(t)} \}]  \notag \\
	&\quad + 2\mathbb{E}[(Y_i - \widetilde{m}(\mathbf{X}_i))(\widetilde{m}(\mathbf{X}_i ) - m_{n, 1}(\mathbf{X}_i; \Theta_t))  \mathds{1}\{ \mathbf{X}_i \in \mathcal{D}_{n, \mathbf{X}}^{-(t)} \} ]   \notag \\
	&\quad + \mathbb{E}[ (\widetilde{m}(\mathbf{X}_i) - m_{n, 1}(\mathbf{X}_i; \boldsymbol{\Theta}_t))^2  \mathds{1}\{ \mathbf{X}_i \in \mathcal{D}_{n, \mathbf{X}}^{-(t)} \} ] \notag \\
	&= \frac{n - a_n}{n}\mathbb{E}[(Y_i - \widetilde{m}(\mathbf{X}_i))^2] + C_{n, i, t}
\end{align}

Defining $\tilde{\mathbf{X}}_{j, i} = [X_{1,i}, \dots X_{j-1,i}, Z_{ j}, X_{j+1, i}, \dots, X_{p, i} ]^\top$ for $i = 1, \dots, n$, where $Z_{ j}$ is independent of $[X_{1, i}, \dots, X_{j-1,i}, X_{j+1,i}, \dots, X_{p, i}]$ and $\epsilon_i$ and $Y_i$, but has the same marginal distribution as $X_{j} \stackrel{d}{=} X_{j, i}$, we can deduce that for any arbitrary measurable function $f: \R^p \longrightarrow \R$ and $\pi \in \mathcal{V}$, it holds: 
		\begin{align}\label{LemmaPermutation}
			\mathbb{E}[ f( X_{1,i}, \dots, X_{j , \pi(i)}, \dots, X_{ p, i}  )  ] &= \mathbb{E}[  \mathbb{E}[ f(X_{1, i}, \dots, X_{j, \pi(i)}, \dots, X_{p, i})  | \pi ] ]  \notag\\
			&= \mathbb{E}[f(\tilde{\mathbf{X}}_{j, i})],
		\end{align}
		since $\mathbb{E}[ f( X_{1, i}, \dots, X_{j, \pi(i)}, \dots, X_{p, i} )  | \pi] \stackrel{d}{=} \mathbb{E}[f(\tilde{\mathbf{X}}_{j, i})]$ due to the independence of the samples. \\
		Now, following exactly the same calculation rules as in the derivation of equation $(\ref{First})$, while also using $(\ref{LemmaPermutation})$, we receive
		\begin{align} \label{FirstPermutation}
			\mathbb{E}[ (Y_i - \widetilde{m}(\mathbf{X}_i^{\pi_{j, t}} ))^2 \mathds{1}\{ \mathbf{X}_i \in \mathcal{D}_{n, \mathbf{X}}^{-(t)} \} ]  &= \frac{n - a_n}{n} \mathbb{E}[ (Y_i - \widetilde{m}(\mathbf{X}_i^{\pi_{j, t}}))^2  ]  \notag \\
			&= \frac{n - a_n}{n}  \mathbb{E}[ (Y_i - \widetilde{m}(\tilde{\mathbf{X}}_{j,i}))^2 ].
		\end{align}
		
Now denote with $\tilde{\mathbf{X}}_{j, i}'$ an independent copy of $\tilde{\mathbf{X}}_{j,i}$ independent of $m_{n, 1}$. Since sampling is restricted to without replacement, the permutation $\pi_{j, t}$ is independent of $\boldsymbol{\Theta}_t$, $\mathcal{D}_n$ and hence independent of $\mathcal{D}_{n, \mathbf{X}}^{-(t)}$. This would be different if sampling is conducted with replacement, since the cardinality of $\mathcal{D}_{n, \mathbf{X}}^{-(t)}$ would be random leading to the dependence of $\pi_{j, t}$ towards $\boldsymbol{\Theta}_t$. This independence allows us to conduct the following computations
\begin{align}
	0 & \le \mathbb{E}[ (\widetilde{m}(\mathbf{X}_i^{\pi_{j, t}}) - m_{n, 1}(\mathbf{X}_i^{\pi_{j, t}}; \boldsymbol{\Theta}_t) )^2 \mathds{1}\{ \mathbf{X}_i  \in \mathcal{D}_{n, \mathbf{X}}^{-(t)}  \} ] \notag \\
	&=  \mathbb{E}[ (\widetilde{m}(\tilde{\mathbf{X}}_{j, i}) - m_{n, 1}(\tilde{\mathbf{X}}_{j, i}; \boldsymbol{\Theta}_t) )^2   \mathds{1}\{ \Theta_{i, t}^{(1)} = 0 \} ] \tag{33a}\label{SecondPermutationa} \\	
	&= \mathbb{E}[ (\widetilde{m}(\tilde{\mathbf{X}}_{j, i}') - m_{n, 1}(\tilde{\mathbf{X}}_{j, i}'; \boldsymbol{\Theta}_t) )^2   \mathds{1}\{ \Theta_{i, t}^{(1)} = 0 \} ]  \tag{33b} \label{SecondPermutationb}\\ 
	&= \mathbb{E}[ (\widetilde{m}(\mathbf{X}_{ i}') - m_{n, 1}(\mathbf{X}_{i}'; \boldsymbol{\Theta}_t) )^2   \mathds{1}\{ \Theta_{i,t}^{(1)} = 0 \}   ] \notag \\
	&= \mathbb{E}[ (\widetilde{m}(\mathbf{X}_i') -m_{n, 1}( \mathbf{X}_i'  ; \boldsymbol{\Theta}_t))^2 \mathds{1}\{ \mathbf{X}_i \in \mathcal{D}_{n, \mathbf{X}}^{-(t)} \} ] = C_{n, i, t},\label{SecondPermutation} 
\end{align} 
where equality $(\ref{SecondPermutationa})$ follows from applying $(\ref{LemmaPermutation})$, equality $(\ref{SecondPermutationb})$ from the calculation results obtained from equation $(\ref{ChangeIndependence})$ and $(\ref{Second})$ and the second last equality from $\mathbf{X}_i' \stackrel{d}{=} \tilde{\mathbf{X}}_{j,i}'$ together with the independence property towards all other random elements, under the event that $\Theta_{i, t}^{(1)} = 0$.

Similarly, set $\tilde{\Delta}_{n, i}^{(j)}(\boldsymbol{\Theta}_t) = \widetilde{m}(\tilde{\mathbf{X}}_{j,i}) - m_{n, 1}(\tilde{\mathbf{X}}_{j, i}; \boldsymbol{\Theta}_t)$ and $\tilde{\epsilon}_{j, i} = Y_i - \widetilde{m}(\tilde{\mathbf{X}}_{j, i})$. Then, recall from model $(\ref{RegModel})$ that
\begin{align}\label{MomentResiduals}
	\mathbb{E}[ \tilde{\epsilon}_{j, i}  | \mathbf{X}_i \in \mathcal{D}_{n, \mathbf{X}}^{-(t)}  ] &= \mathbb{E}[ Y_i - \widetilde{m}(\tilde{\mathbf{X}}_{j, i}) | \Theta_{i, t}^{(1)} = 0 ] \notag \\
	&= \mathbb{E}[ \widetilde{m}(\mathbf{X}_i) + \epsilon_i - \widetilde{m}(\tilde{\mathbf{X}}_{j, i}) | \Theta_{i, t}^{(1)} = 0 ] \notag \\ 
	&= \mathbb{E}[ \widetilde{m}(\mathbf{X}_i)  | \Theta_{i, t}^{(1)} = 0]   + \mathbb{E}[ \epsilon_i | \Theta_{i, t}^{(1)} = 0]  - \mathbb{E}[ \widetilde{m}(\tilde{\mathbf{X}}_{j, i}) | \Theta_{i, t}^{(1)} = 0 ] \notag \\
	&= \mathbb{E}[ \epsilon_i ]  + \mathbb{E}[ \widetilde{m}(\mathbf{X}_i) ] - \mathbb{E}[\widetilde{m}( \tilde{\mathbf{X}}_{j, i} )] =  \mathbb{E}[ \widetilde{m}(\mathbf{X}_i) ]  - \mathbb{E}[ \widetilde{m}(\mathbf{X}_i) ] \notag  \\
	&= 0, 
\end{align}
where we explicitly used assumption $\ref{Ass1}$ in the second-last equality equality and the independence of $\boldsymbol{\Theta}_t^{(1)}$ towards $\epsilon_i$ and $\mathbf{X}_i$ in the fourth equality. Now, consider

\begin{align}\label{ResidualAndDelta}
	\mathbb{E}[ \tilde{\epsilon}_{j,i} \cdot  \tilde{\Delta}_{n, i}^{(j)}(\boldsymbol{\Theta}_t) \cdot \mathds{1} \{\Theta_{i, t}^{(1)} = 0 \}  ] &=  \mathbb{E}[ \epsilon_i \cdot \tilde{\Delta}_{n, i}^{(j)}(\boldsymbol{\Theta}_t) \cdot \mathds{1} \{\Theta_{i, t}^{(1)} = 0 \}] + \notag  \\
	&+ \mathbb{E}[ ( \widetilde{m}(\mathbf{X}_i ) - \widetilde{m}(\tilde{ \mathbf{X} }_{j, i}) ) ( \widetilde{m}( \tilde{\mathbf{X}}_{j, i} ) - m_{n, 1}( \tilde{\mathbf{X}}_{j, i} ; \boldsymbol{\Theta}_t  ) )  \cdot \mathds{1} \{\Theta_{i, t}^{(1)} = 0 \} ] \notag \\
	&= \mathbb{P}[ \Theta_{i, t}^{(1)} = 0 ] \cdot  \mathbb{E}[\epsilon_i ] \cdot  \mathbb{E}[ \tilde{\Delta}_{n, i}^{(j)}(\boldsymbol{\Theta}_t)  | \Theta_{i, t}^{(1)} = 0 ] \notag \\ 
    &+ \mathbb{E}[ ( \widetilde{m}(\mathbf{X}_i ) - \widetilde{m}(\tilde{ \mathbf{X} }_{j, i}) ) ( \widetilde{m}( \tilde{\mathbf{X}}_{j, i} ) - m_{n, 1}( \tilde{\mathbf{X}}_{j, i} ; \boldsymbol{\Theta}_t  ) )  \cdot \mathds{1} \{\Theta_{i, t}^{(1)} = 0 \} ]  \notag \\
	&= \mathbb{E}[ ( \widetilde{m}(\mathbf{X}_i ) - \widetilde{m}(\tilde{ \mathbf{X} }_{j, i}) ) ( \widetilde{m}( \tilde{\mathbf{X}}_{j,i} ) - m_{n, 1}( \tilde{\mathbf{X}}_{ j, i}; \boldsymbol{\Theta}_t ) )  \cdot \mathds{1} \{\Theta_{i, t}^{(1)} = 0 \} ] \notag \\
	&= \mathbb{P}[ \Theta_{i, t}^{(1)} = 0 ] \cdot \mathbb{E}[  ( \widetilde{m}(\mathbf{X}_i ) - \widetilde{m}(\tilde{ \mathbf{X} }_{j, i}) ) ( \widetilde{m}( \tilde{\mathbf{X}}_{j, i} ) - m_{n, 1}( \tilde{\mathbf{X}}_{j, i}; \boldsymbol{\Theta}_t ) ) | \Theta_{i, t }^{(1)} = 0 ] \notag \\
	&=  \frac{\gamma_n}{n} \cdot Cov_{\Theta_{i, t}^{(1)} = 0}\left( \{ \widetilde{m}(\mathbf{X}_i) - \widetilde{m}( \tilde{\mathbf{X}}_{j, i} )\} ; \{ (\widetilde{m}(\tilde{\mathbf{X}}_{j, i}) - m_{n, 1}(\tilde{\mathbf{X}}_{j, i} ; \boldsymbol{\Theta}_t) ) \} \right) \notag \\
	&=: \frac{\gamma_n}{n}\cdot \xi_{n, i}^{(j)}(\boldsymbol{\Theta_t})
\end{align}
The second equality follows from the law of total expectation and the independence of $\epsilon_i$ and $ \tilde{\Delta}_{n, i}^{(j)} $ under the event that $\Theta_{i, t}^{(1)} = 0$, i.e. that the $i$-th observation has not been selected during training. The third equality follows from equation $(\ref{MomentResiduals})$. The second last equality follows from the fact that $ \mathbb{E}[ \widetilde{m}(\mathbf{X}_i)  - \widetilde{m}(\tilde{\mathbf{X}}_{j, i} ) ] = \mathbb{E}[\widetilde{m}(\mathbf{X}_i)  ] - \mathbb{E}[\widetilde{m}(\tilde{\mathbf{X}}_{j, i} )] = 0$, since $\widetilde{m}(\mathbf{X}_i) \stackrel{d}{ = } \widetilde{m}(\tilde{\mathbf{X}}_{j, i} )$. Finally, we can now obtain
\begin{align}\label{ThirdPermutation}
&\mathbb{E}[ (Y_i - \widetilde{m}(\mathbf{X}_i^{\pi_{j, t}}))(\widetilde{m}(\mathbf{X}_i^{\pi_{j, t}}) - m_{n, 1}(\mathbf{X}_i^{\pi_{j, t}}; \boldsymbol{\Theta}_t))  \mathds{1}\{ \mathbf{X}_i \in \mathcal{D}_{n, \mathbf{X}}^{-(t)} \}] \notag \\
&  = \mathbb{E}[ (Y_i - \widetilde{m}(\tilde{\mathbf{X}}_{j, i} ))(\widetilde{m}(\tilde{\mathbf{X}}_{j, i}) - m_{n, 1}(\tilde{\mathbf{X}}_{j, i}; \boldsymbol{\Theta}_t))  \mathds{1}\{  \Theta_{i, t}^{(1)} = 0 \}]  \notag \\
 &= \mathbb{E}[ \tilde{\epsilon}_{i, j} \cdot \tilde{\Delta}_{n, i}^{(j)}( \boldsymbol{\Theta}_t)  \cdot \mathds{1}\{ \Theta_{i, t}^{(1)} = 0  \}]  \notag\\
&= \frac{\gamma_n}{n} \cdot \xi_{n, i }^{(j)}(\boldsymbol{\Theta_t})
\end{align}
In the second equality, we used $(\ref{LemmaPermutation})$, while the last equality follows from applying equation $(\ref{ResidualAndDelta})$.  \\ 
Using the results from $(\ref{FirstPermutation})$, $(\ref{SecondPermutation})$ and $(\ref{ThirdPermutation})$, one can now obtain:
\begin{align}\label{FinalSecond}
\mathbb{E}[ (Y_i - m_{n, 1}(\mathbf{X}_i^{\pi_{j, t}}; \Theta_t))^2 \mathds{1}\{\mathbf{X}_i \in \mathcal{D}_{n, \mathbf{X}}^{-(t)}  \} ] &= \mathbb{E}[ (Y_i - \widetilde{m}(\mathbf{X}_i^{\pi_{j, t}}))^2  \mathds{1}\{ \mathbf{X}_i \in \mathcal{D}_{n, \mathbf{X}}^{-(t)} \}  ] \notag \\
&\quad +  \mathbb{E}[ (\widetilde{m}(\mathbf{X}_i^{\pi_{j, t}}) - m_{n, 1}(\mathbf{X}_i^{\pi_{j, t}}; \Theta_t) )^2 \mathds{1}\{ \mathbf{X}_i  \in \mathcal{D}_{n, \mathbf{X}}^{-(t)}  \} ] \notag  \\
& \quad +  2 \mathbb{E}[ \tilde{\epsilon}_{i, j} \cdot  \tilde{\Delta}_{n, i}^{(j)}(\Theta_t) \mathds{1}\{ \mathbf{X}_i \in \mathcal{D}_{n, \mathbf{X} }^{-(t)} \} ]   \notag \\
&=  \frac{n - a_n}{n} \mathbb{E}[ ( Y_i - \widetilde{m}(\tilde{\mathbf{X}}_{j, i}))^2 ] + C_{n, i, t} + \frac{2\gamma_n}{n}\cdot  \xi_{n, i }^{(j)}(\boldsymbol{\Theta}_t)
\end{align}

Finally, using $(\ref{FinalFirst})$ and $(\ref{FinalSecond})$ together with $(\ref{Finiteness})$, we obtain 
\begin{align}\label{AlmostFinished}
	\mathbb{E}[I_{n, M}^{(OOB)}(j)]  &= \frac{1}{M \gamma_{n}} \sum\limits_{t = 1}^M \sum\limits_{i = 1}^n \mathbb{E}[\{ (Y_i - m_{n, 1}(\mathbf{X}_i^{\pi_{j, t}}; \boldsymbol{\Theta}_t))^2 - (Y_i - m_{n, 1}(\mathbf{X}_i; \boldsymbol{\Theta}_t))^2 \} \mathds{1}\{ \mathbf{X}_i \in \mathcal{D}_{n, \mathbf{X}}^{-(t)} \} ] \notag \\
	&= \frac{1}{M\gamma_n} \sum\limits_{t = 1}^M \sum\limits_{ i = 1}^n \left\{ \frac{n - a_n}{n} \{ \mathbb{E}[( Y_i - \widetilde{m}(\tilde{\mathbf{X}}_{j, i}))^2 ] - \mathbb{E}[(Y_i - \widetilde{m}(\mathbf{X}_i))^2] \} + C_{n, i,t} - C_{n, i,t} + \frac{2 \gamma_n}{n} \cdot \xi_{n, i}^{(j)}(\boldsymbol{\Theta}_t) \right\} \notag \\
	&= \frac{ n-a_n}{\gamma_n} \left\{\mathbb{E}[(Y_1 - \widetilde{m}(\tilde{\mathbf{X}}_{j, 1}))^2] - \mathbb{E}[(Y_1 - \widetilde{m}(\mathbf{X}_1))^2] \right\} + \frac{2}{\gamma_n} \sum\limits_{i = 1}^n \left(  \frac{1}{M}\sum\limits_{t = 1}^M \frac{\gamma_n}{n} \xi_{n, i}^{(j)}(\boldsymbol{\Theta}_t)  \right) \notag \\
	&=  \mathbb{E}[(Y_1 - \widetilde{m}(\tilde{\mathbf{X}}_{j, 1}))^2] - \mathbb{E}[(Y_1 - \widetilde{m}(\mathbf{X}_1))^2]  + 2 \cdot  \left(  \frac{1}{M} \sum\limits_{t = 1}^M \xi_{n, 1}^{(j)}( \boldsymbol{\Theta_t} )  \right)
\end{align}

where the second last equality follows from the identical distribution (in $i$) of the sequence $\{ Y_i - m(\tilde{\mathbf{X}}_{j,i}) \}_{i = 1}^n$, respectively $\{ Y_i - m(\mathbf{X}_i) \}_{i = 1}^n$. The last equality follows from the identical distribution of the sequence $\{ \xi_{n, i }^{(j)}(\boldsymbol{\Theta}_t) \}_{i = 1}^n$.\\

Without loss of generality, assume that the first $1 \le s \le p$ features are informative, i.e. $\mathcal{S} = \{1, \dots, s\}$ and define $\mathbf{X}_{i; \mathcal{S}} =  [ X_{1,i}, X_{2,i}, \dots, X_{s, i} ]^\top \in \R^s$, the $i$-th random vector reduced to informative features characterized by $\mathcal{S}$. Similarly, let $\tilde{\mathbf{X}}_{j, i; \mathcal{S}}$ be the reduced random vector of $\tilde{\mathbf{X}}_{j, i}$, in which the $j$-th position is substituted by $Z_j$, with $1 \le j \le s$ . 

We distinguish between two cases: First, let $j \in \mathcal{S}^{C} = \{ 1, \dots, p \} \setminus \mathcal{S}.$ Under this scenario, we know that $\widetilde{m}( \tilde{\mathbf{X}}_{j, 1} ) = \widetilde{m}(\mathbf{X}_1) = m(\mathbf{X}_{1; \mathcal{S}})$. Hence, we have 
\begin{align}\label{NotInfCovariance}
	\xi_{n, 1}^{(j)}(\boldsymbol{\Theta}_t)&=  Cov_{\Theta_{1, t}^{(1)} = 0}\left( \{\widetilde{m}(\mathbf{X}_1) - \widetilde{m}( \tilde{\mathbf{X}}_{j, 1} )\} ; \{ (\widetilde{m}(\tilde{\mathbf{X}}_{j, 1}) - m_{n, 1}(\tilde{\mathbf{X}}_{ j, 1} ; \boldsymbol{\Theta}_t) ) \cdot \mathds{1}\{ \Theta_{1, t}^{(1)} = 0  \} \} \right) \notag \\
	&= Cov_{\Theta_{1, t}^{(1)} = 0}\left(0; \{ (\widetilde{m}(\tilde{\mathbf{X}}_{j, 1}) - m_{n, 1}(\tilde{\mathbf{X}}_{j, 1} ; \boldsymbol{\Theta}_t) ) \cdot \mathds{1}\{ \Theta_{1, t}^{(1)} = 0  \} \right) = 0.
\end{align}
 Therefore, it immediately follows by applying $(\ref{AlmostFinished})$ and $(\ref{NotInfCovariance})$ that 
 \begin{align}
 	\mathbb{E}[I_{n, M}^{(OOB)}(j) ] &= \mathbb{E}[ (Y_1 - \widetilde{m}( \tilde{\mathbf{X}}_{j, 1} ))^2]  - \mathbb{E}[ (Y_1 - \widetilde{m}(\mathbf{X}_1) )^2 ] \notag \\
 	&= \mathbb{E}[( Y_1 - m(\mathbf{X}_{1; \mathcal{S}}) )^2] -  \mathbb{E}[(  Y_1 - m(\mathbf{X}_{1; \mathcal{S}}) )^2] \notag \\
 	&= 0 = I(j).
 \end{align}
 
 Secondly, let $j \in \mathcal{S}$ be informative. Then notice that 
\begin{align}\label{AsymptoticStep}
	\frac{\gamma_n}{n}\frac{1}{M} \sum\limits_{t = 1}^M \xi_{n, 1}^{(j)}(\boldsymbol{\Theta}_t) 	&=  \frac{\gamma_n}{n}\frac{1}{M} \sum\limits_{ t = 1}^M Cov_{ \Theta_{i, t}^{(1)} = 0 } \left( \{\widetilde{m}(\mathbf{X}_1) - \widetilde{m}( \tilde{\mathbf{X}}_{j, 1} )\} ; \{ \widetilde{m}(\tilde{\mathbf{X}}_{j, 1}) - m_{n, 1}(\tilde{\mathbf{X}}_{ j, 1} ; \boldsymbol{\Theta}_t) )  \} \right)  \notag \\
	&= \frac{\gamma_n}{n} \frac{1}{M} \sum\limits_{ t = 1}^M \mathbb{E}[ ( m(\mathbf{X}_{1; \mathcal{S}})   - m(\tilde{\mathbf{X}}_{j, 1; \mathcal{S}}) )  \cdot  ( \widetilde{m}(\tilde{\mathbf{X}}_{ j, 1}) - m_{n, 1} (\tilde{\mathbf{X}}_{j, 1} ; \boldsymbol{\Theta}_t ) ) | \Theta_{1, t}^{(1)} = 0 ] \notag \\
	&= \frac{1}{M} \sum\limits_{ t= 1}^M \mathbb{E}[ ( m(\mathbf{X}_{1; \mathcal{S}}) - m(\tilde{\mathbf{X}}_{ j, 1; \mathcal{S}} ) ) \cdot ( \widetilde{m}(\tilde{\mathbf{X}}_{j, 1}) - m_{n , 1}(\tilde{\mathbf{X}}_{ j, 1 }; \boldsymbol{\Theta}_t )  ) \cdot \mathds{1}\{ \Theta_{1, t}^{(1)}  = 0\} ]  \notag \\
	&= \mathbb{E}\left[ ( m(\mathbf{X}_{1; \mathcal{S}}) - m(\tilde{\mathbf{X}}_{j, 1; \mathcal{S}} ) )  \cdot   \frac{Z_1(M)}{M} \cdot \left( m(\tilde{\mathbf{X}}_{ j, 1 }  ) - m_{n, M}^{OOB}( \tilde{\mathbf{X}}_{ j, 1} ) \right) \right] , 
	\end{align}
	where $Z_1(M) = \sum\limits_{ t = 1}^M \mathds{1} \{ \Theta_{1, t}^{(1)} = 0 \} = \sum\limits_{ t = 1}^M \mathds{1}\{ \mathbf{X}_{1} \text{ has not been selected under }  \boldsymbol{\Theta}_t \}$ is the number of times the first observation has not been selected during the sampling procedure and \\
	 $m_{n, M}^{OOB}(\tilde{\mathbf{X}}_{j, 1}) = \frac{1}{Z_{1}(M)} \sum\limits_{ t= 1}^M m_{n, 1}(\tilde{\mathbf{X}}_{j, 1}; \boldsymbol{\Theta}_t) \cdot \mathds{1}\{ \Theta_{i, t}^{(1)} = 0 \}$. Due to assumption $\ref{Ass2}$, we can deduce that 
	\begin{align}\label{FirstFiniteMajority}
		|m(\mathbf{X}_{1; \mathcal{S}} ) - m(\tilde{\mathbf{X}}_{ 1, j; \mathcal{S}} ) | &\le 2K < \infty
	\end{align}
	On the other hand, we observe the following bound: 
	\begin{align}\label{SecondFiniteMajority}
		\left |  \frac{Z_1(M)}{M} \cdot \left(\widetilde{m}(\tilde{\mathbf{X}}_{1,j }) -  m_{n, M}^{OOB}(\tilde{\mathbf{X}}_{1, j } ) \right)  \right | &\le  K +\sum\limits_{ \ell = 1}^n W_{n, \ell}( \tilde{\mathbf{X}}_{1, j} ; \boldsymbol{\Theta}_1, \dots, \boldsymbol{\Theta}_M ) \cdot |Y_\ell| \notag \\
		&\le K  + \max\limits_{ 1 \le \ell \le n} |Y_\ell| =: K + f_n, 
	\end{align}
	where $W_{n, \ell}(\cdot; \boldsymbol{\Theta}_1, \dots, \boldsymbol{\Theta}_M) = \frac{1}{M} \sum\limits_{ t= 1}^M W_{n, \ell}(\cdot; \boldsymbol{\Theta}_t)$.	Hence, we can deduce by applying $(\ref{FirstFiniteMajority})$ and $(\ref{SecondFiniteMajority})$ that 
	\begin{align}
	|f_{n, M} | := 	\left|( m(\mathbf{X}_{1; \mathcal{S}}) - m(\tilde{\mathbf{X}}_ {j, 1; \mathcal{S}} ) )  \cdot   \frac{Z_1(M)}{M} \cdot \left( \widetilde{m}(\tilde{\mathbf{X}}_{j, 1 }  ) - m_{n, M}^{OOB}( \tilde{\mathbf{X}}_{1, j} ) \right) \right| \le 2K ( K + f_n) =: g_n,  
	\end{align}
	i.e. $g_n$ is a finite upper bound for $|f_{n, M}|$, independent of $M$ such that $\mathbb{E}_{\boldsymbol{\Theta}}[|g_n|] = 2K\cdot(K + f_n) < \infty$, where $f_n := \max\limits_{ 1 \le \ell \le n } |Y_\ell|$. Applying Lebesgue's dominated convergence theorem while using Proposition $\ref{HelpingProposition}$ under the sampling without replacement scheme with $c_n = 1 - a_n /n = \gamma_n/n$ and using $Z_1(M)/M \longrightarrow c_n$ as $M \rightarrow \infty$ due to $(\ref{InfForestModification})$, we obtain 
	\begin{align}\label{AlternativeExpress}
		&\lim\limits_{M \rightarrow \infty} \mathbb{E}\left[   ( m(\mathbf{X}_{1; \mathcal{S}})  - m(\tilde{\mathbf{X}}_{ j,1; \mathcal{S}})) \cdot \frac{Z_1(M)}{M} \left( \widetilde{m}(\tilde{\mathbf{X}}_{j, 1}) - m_{n, M}^{OOB}( \tilde{\mathbf{X}}_{ j, 1 }  )   \right)    \right] \notag \\ 
		&= \frac{\gamma_n}{n} \mathbb{E}[ ( m(\mathbf{X}_{1; \mathcal{S}})  - m(\tilde{\mathbf{X}}_{j, 1; \mathcal{S}})) ( \widetilde{m}(\tilde{\mathbf{X}}_{j, 1} ) - m_{n}^{OOB}(\tilde{\mathbf{X}}_{j, 1 } ) )  ]  \notag \\
		&=: \frac{\gamma_n}{n} J_n
	\end{align} 
	Note that $J_n$ can be bounded the following way using  the Cauchy-Schwarz inequality:
	\begin{align}
		 J_n \le |J_n| &\le \sqrt{\mathbb{E}[| m(\mathbf{X}_{1; \mathcal{S}}) - m( \tilde{\mathbf{X}}_ {j, 1; \mathcal{S}} )|^2] } \sqrt{\mathbb{E}[|\widetilde{m}(\tilde{\mathbf{X}}_{j, 1} ) - m_{n}^{OOB}( \tilde{\mathbf{X}}_{ j, 1} )|^2]}
	\end{align}
	Since $J_n \ge - |J_n|$ and due to assumption $\ref{Ass3}$, we can deduce that $\lim\limits_{n \rightarrow \infty} J_n = 0$. Note that the $L_2$ consistency of the Random Forest estimate $m_n^{OOB}$ for Out-of-Bag samples follows by $\ref{Ass3}$ and a Corollary given in \cite{ramosaj2019consistent}. Finally, we can conclude with $(\ref{AsymptoticStep})$ and $(\ref{AlternativeExpress})$ that
	\begin{align}
		\lim\limits_{n \rightarrow \infty} \lim\limits_{M \rightarrow \infty } \frac{1}{M} \sum\limits_{t = 1}^M \xi_{n, 1}^{(j)}(\boldsymbol{\Theta}_t) 	= \lim\limits_{ n\rightarrow \infty}\frac{n}{\gamma_n} \frac{\gamma_n}{n} J_n = \lim\limits_{n \rightarrow\infty} J_n = 0,
	\end{align}
	which completes the proof. \\
\end{proof}

In the sequel, we will shortly deliver proofs for the following claims, that have been mentioned in the main article: $(i)$ We argued that a variable $j \in \{1,\dots, p \}$ is \textit{important}, if the partial derivate of $\widetilde{m}(\mathbf{x})$ w.r.t. $x_j$ vanishes, i.e. we claimed the equivalence of both definitions $(\ref{SparsReg})$ and $(\ref{SparseDifferentiable})$ mentioned in the article. $(ii)$ We claimed that the assumptions given in \cite{scornet2015consistency} can replace $\ref{Ass1} - \ref{Ass3}$. $(iii)$ We claimed that the theoretical cut criterion $L^{(k)}(j,z)$ is independent of the residual noise $\sigma^2$. 

\begin{proof}[Proof of $(i)$]
Suppose that being important is defined through $(\ref{SparsReg})$ and assume without loss of generality, that the first $s \le p$ features are important, i.e. $\widetilde{m}(\mathbf{x}) = m(\mathbf{x}_{\mathcal{S}} )$, where $\mathbf{x}_{\mathcal{S}} = [x_1, \dots, x_s]^\top \in \R^s$. Then it follows immediately that $\frac{\partial \widetilde{m}(\mathbf{x}) }{ \partial x_j} = 0$ for all $j \in \{1, \dots, p\} \setminus \mathcal{S}$, since $\widetilde{m}(\mathbf{x}) = m(\mathbf{x}_{\mathcal{S}})$ does not depend on $j$. Hence, variable $j \in \{1, \dots, p\} \setminus \mathcal{S}$ is unimportant according to the definition given in $(\ref{SparseDifferentiable})$. \\
For the other direction, define the set $\mathcal{C} := \{ k \in \{1,\dots, p\} : \frac{\partial \widetilde{m}(\mathbf{x}) }{ \partial x_k } \neq 0 \}$ and suppose that $j \in \{1, \dots, p \}$ is informative in the sense that $j \in \mathcal{C}$. Then, let $\mathbf{a} \in \R^p$ be fixed but arbitrary. Using the multivariate Taylor expansion of $\widetilde{m}$ at $\mathbf{a}$, one has
\begin{align}\label{TaylorExpans}
	\widetilde{m}(\mathbf{x}) &\approx \widetilde{m}(\mathbf{a}) + \nabla \widetilde{m}(\mathbf{a})^\top (\mathbf{x} - \mathbf{a}) \notag \\
	&= \widetilde{m}(\mathbf{a}) + \sum\limits_{ s \in \mathcal{C}} \widetilde{m}_s'(a_s)(x_s  - a_s) =: m(\mathbf{x}_{\mathcal{C}})
\end{align}
which yields to $\mathcal{S} = \mathcal{C}$, i.e. the function $\widetilde{m}$ can be reduced to a function of potentially lower dimension, since $\mathbf{a}$ is chosen arbitrary and $(\ref{TaylorExpans})$ holds for any fixed $\mathbf{a}$. \\
\end{proof}

\begin{proof}[Proof of $(ii)$]
	Recalling some of the assumptions given in \cite{scornet2015consistency} in order to establish $L_2$ consistency, we have 
	\begin{enumerate}
		\item \label{FirstScornet} $\widetilde{m}(\mathbf{x}) = \sum\limits_{ k = 1}^p \widetilde{m}_k(x_k)$, where $\{ m_k(x_k) \}_{k = 1}^p$ is a sequence of univariate and continuous functions.  
		\item \label{SecondScornet} The feature vector $\mathbf{X} = [X_1, \dots, X_p]^\top \in \R^p$ is assumed to be uniformly distributed over $[0,1]^p$. 
		\item \label{ThirdScornet} The residuals are assumed to be centered Gaussian with variance $\sigma^2 \in (0, \infty)$, independent of $\mathbf{X}$. 
		\item \label{FourthScornet} Sampling is restricted to sampling without replacement such that $a_n \rightarrow \infty$, $t_n \rightarrow \infty$ and $\frac{t_n \cdot ( \log(a_n) )^9}{a_n} \rightarrow 0$ as $n \rightarrow \infty$.
	\end{enumerate}
	
	Now, since $\widetilde{m}_k$ is continuous for every $k \in \{1, \dots, p \}$ according to $\ref{FirstScornet}$., it immediately follows that $\widetilde{m}$ resp. $|\widetilde{m}(\mathbf{x})|$ is continuous. Hence, since $[0,1]^p$ as the support of $\mathbf{X}$ is compact, so is the set $\{ \widetilde{m}(\mathbf{x}) : \mathbf{x} \in [0,1]^p \}$, which then yields to $\sup\limits_{\mathbf{x} \in [0,1]^p } \widetilde{m}(\mathbf{x}) = \max\limits_{ \mathbf{x} \in [0,1]^p }  \widetilde{m}(\mathbf{x}) = K< \infty$. This is nothing else than assumption $\ref{Ass2}$. Furthermore, we have from $\ref{SecondScornet}$. that $\mathbf{X} \sim Unif([0,1]^p)$, which yields to $f_{\mathbf{X}} (x_1, \dots x_p) = \mathds{1}\{ \mathbf{x} \in [0,1]^p  \} = \prod\limits_{j = 1}^p \mathds{1}\{ X_j \in [0,1] \}  = \prod\limits_{ j = 1}^p f_{Unif(0,1)} (x_j)$, i.e. the multivariate density decomposes into the product of univariate densities. Therefore, the sequence of random variables  $\{ X_j \}_{j = 1}^p$ is pairwise independent. Hence, assumption $\ref{Ass1}$ follows. Assuming that the residuals are centered Gaussian with finite variance $\sigma^2$ as given in \ref{ThirdScornet}. is nothing else than the specification of our assumption that $\mathbb{E}[\epsilon] = 0$ and $Var(\epsilon) \in (0, \infty)$ by imposing explicitly the Gaussian distribution. Assumption $\ref{Ass3}$ then immediately follows  by using Theorem $1$ in \cite{scornet2015consistency} and the assumptions $\ref{FirstScornet}$ - $\ref{FourthScornet}$. Assumptions $\ref{Ass0}$ and $\ref{Ass01}$ are not required in \cite{scornet2015consistency}, and hence, they do not prohibit us to use Theorem $1$ in \cite{scornet2015consistency}. Therefore, they can be taken over additionally.\\ 
\end{proof}

\begin{proof}[Proof of $(iii)$]
	Consider the theoretical cut criterion $L^{(k)}(j,z)$ at level $1 \le k \le \lceil \log_2(t_n) \rceil + 1$ with $1 \le \ell \le 2^{k -1}$. Then we can see that this is independent of $\sigma^2$:
	\begin{align*}
	L^{(k)}(j, z) &=  Var[Y_1 | \mathbf{X}_1 \in A_\ell^{(k)}] - \mathbb{P}[X_{j, 1} < z | \mathbf{X}_1 \in A_\ell^{(k)}] ¸\cdot Var[Y_1 | \mathbf{X}_1 \in A_\ell^{(k)} , X_{j, 1} < z]   \\
	&- \mathbb{P}[X_{j, 1} \ge z | \mathbf{X}_1 \in A_\ell^{(k)}]\cdot Var[Y_1 | \mathbf{X}_1 \in A_\ell^{(k)}, X_{j, 1} \ge z] \\
	&= Var[ \widetilde{m}(\mathbf{X}_1) | \mathbf{X}_1 \in A_{\ell}^{(k)} ] + Var[\epsilon_1 | \mathbf{X}_1 \in A_{\ell}^{(k)}] 	-  \mathbb{P}[X_{j, 1} < z | \mathbf{X}_1 \in A_\ell^{(k)}] \cdot \left\{  Var[ \widetilde{m}(\mathbf{X}_1) | \mathbf{X}_1 \in A_\ell^{(k)} , X_{j, 1} < z] \right. +  \\
	& \quad \left.  Var[\epsilon_1 | \mathbf{X}_1 \in A_{\ell}^{(k)} X_{j, 1} < z ]  \right\} - \mathbb{P}[X_{j, 1} \ge z | \mathbf{X}_1 \in A_\ell^{(k)}]\cdot \left\{ Var[  \widetilde{m}(\mathbf{X}_1) | \mathbf{X}_1 \in A_\ell^{(k)}, X_{j, 1} \ge z] + \right. \\
	& \quad \left.  Var[\epsilon_1 | \mathbf{X}_1 \in A_{\ell}^{(k)}, X_{j, 1} \ge z] \right\} \\
	&= Var[ \widetilde{m}(\mathbf{X}_1) | \mathbf{X}_1 \in A_{\ell}^{(k)} ] + \sigma^2	-  \mathbb{P}[X_{j, 1} < z | \mathbf{X}_1 \in A_\ell^{(k)}] \cdot \left\{  Var[ \widetilde{m}(\mathbf{X}_1) | \mathbf{X}_1 \in A_\ell^{(k)} , X_{j, 1} < z] + \sigma^2 \right\} - \\
	& \quad \mathbb{P}[X_{j, 1} \ge z | \mathbf{X}_1 \in A_\ell^{(k)}]\cdot \left\{ Var[  \widetilde{m}(\mathbf{X}_1) | \mathbf{X}_1 \in A_\ell^{(k)}, X_{j, 1} \ge z]   +   \sigma^2\right\} \\
	&= Var[ \widetilde{m}(\mathbf{X}_1) | \mathbf{X}_1 \in A_{\ell}^{(k)} ] -  \mathbb{P}[X_{j, 1} < z | \mathbf{X}_1 \in A_\ell^{(k)}] \cdot  Var[ \widetilde{m}(\mathbf{X}_1) | \mathbf{X}_1 \in A_\ell^{(k)} , X_{j, 1} < z]   - \\
	&\quad \mathbb{P}[X_{j, 1} \ge z | \mathbf{X}_1 \in A_\ell^{(k)}]\cdot  Var[  \widetilde{m}(\mathbf{X}_1) | \mathbf{X}_1 \in A_\ell^{(k)}, X_{j, 1} \ge z]  ,  
	\end{align*}
	where the third equality follows from the independence of $\epsilon_1$ and $\mathbf{X}_1$. 
\end{proof}

\newpage


\bibliography{BibFile}

\newpage 
\begin{center}
	\textbf{SUPPLEMENT}
\end{center}
 
 In the following, we present supplementary material, which has been part of the simulation study in the main article. 

\section{Results for $p < n$ Problems}
\begin{figure}[h!]
	\centering
	\mbox{\subfigure[$n = 100$]{\includegraphics[width=3in]{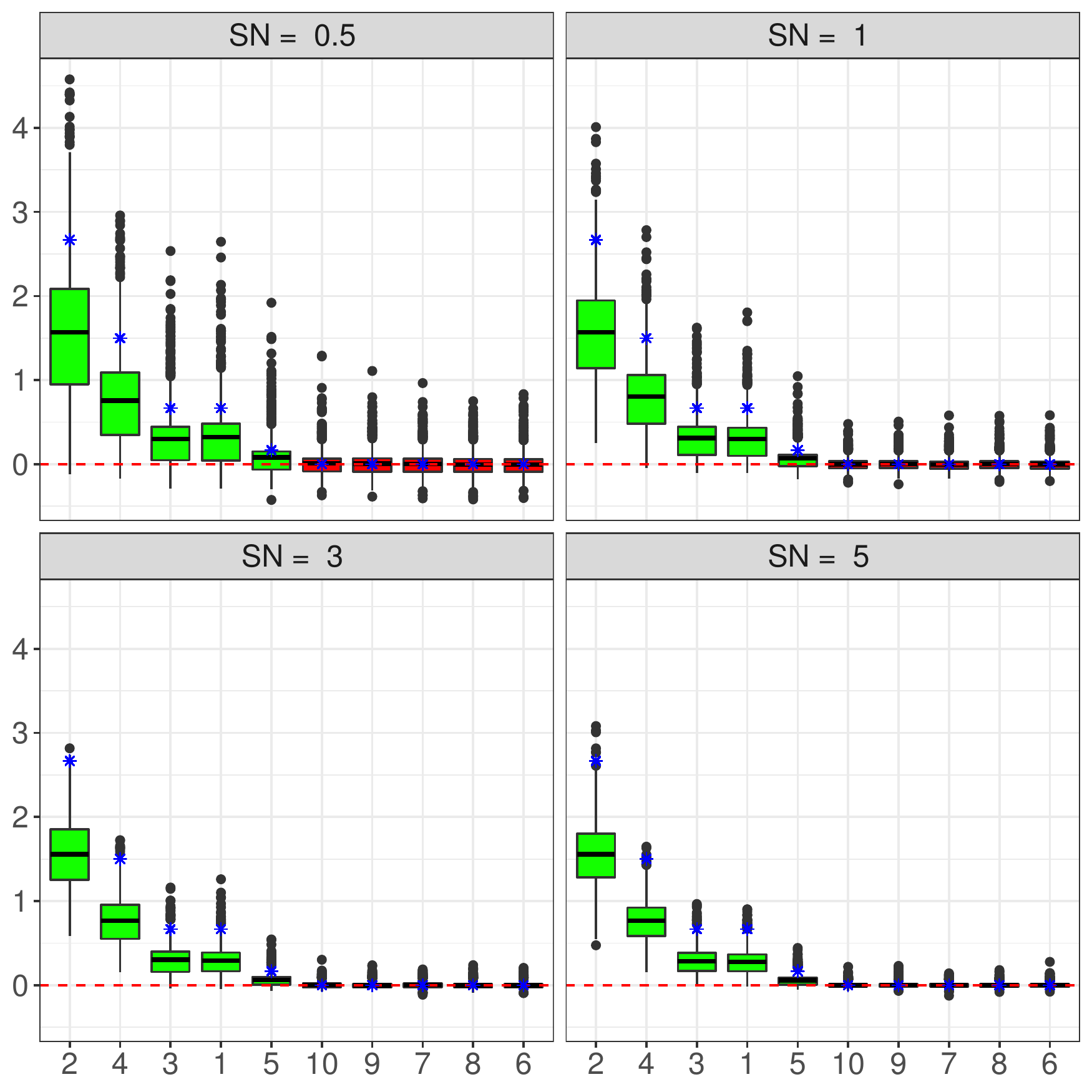}}
		\quad
		\subfigure[$n = 500$]{\includegraphics[width=3in]{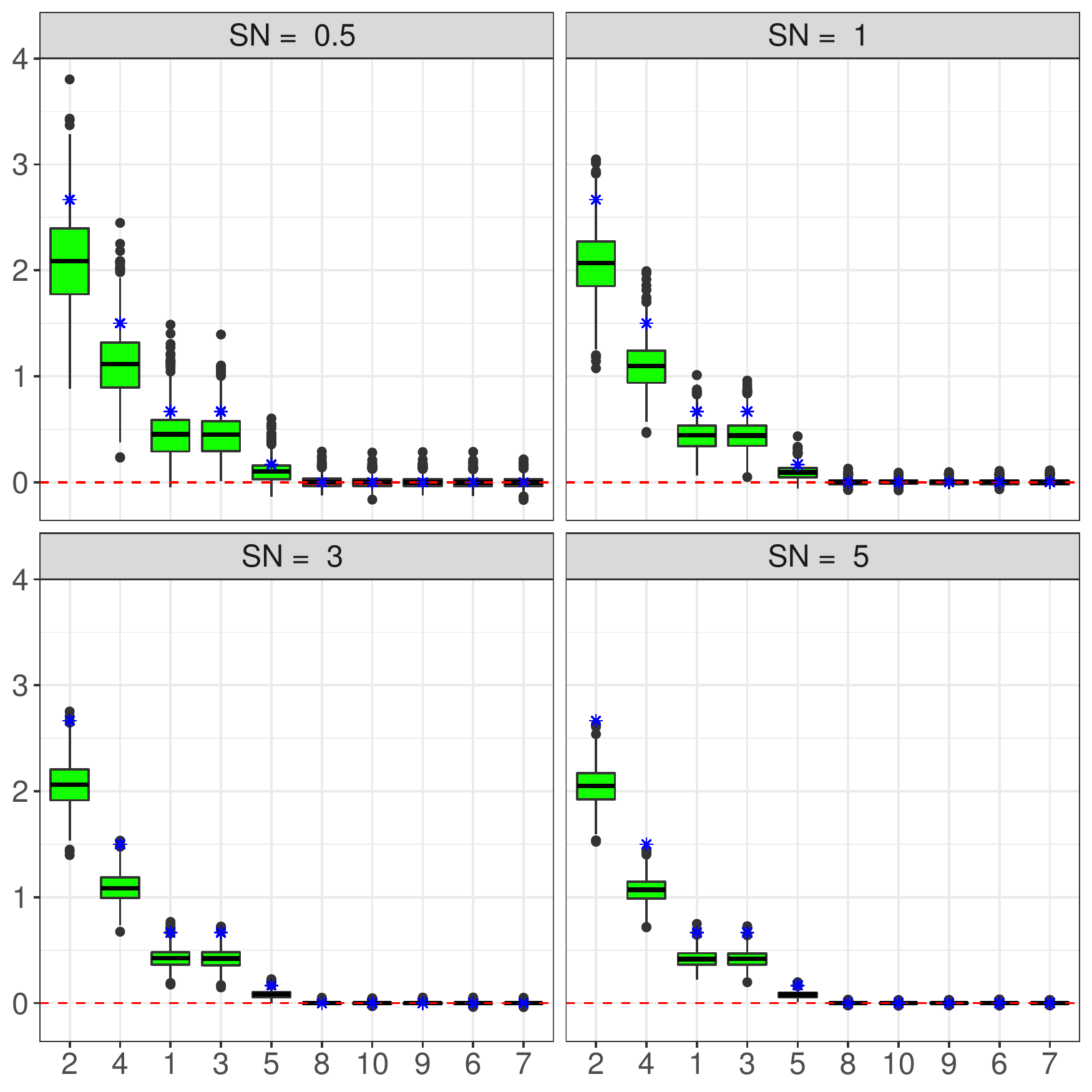} }}
	\caption{Simulation results for the permutation importance with various signal-to-noise ratios under a \textbf{linear model} as described in $(1)$ of the main article using $MC = 1,000$ Monte-Carlo iterations with a sample size of (a) $ n  = 100$ and (b) $n = 500$. The solid lines refer to the empirical mean and \textcolor{blue}{$\star$} to its expectation. }
	\label{LinearPlot}
\end{figure}
\FloatBarrier

\begin{figure}[h!]
	\centering
	\mbox{\subfigure[$n = 100$]{\includegraphics[width=3in]{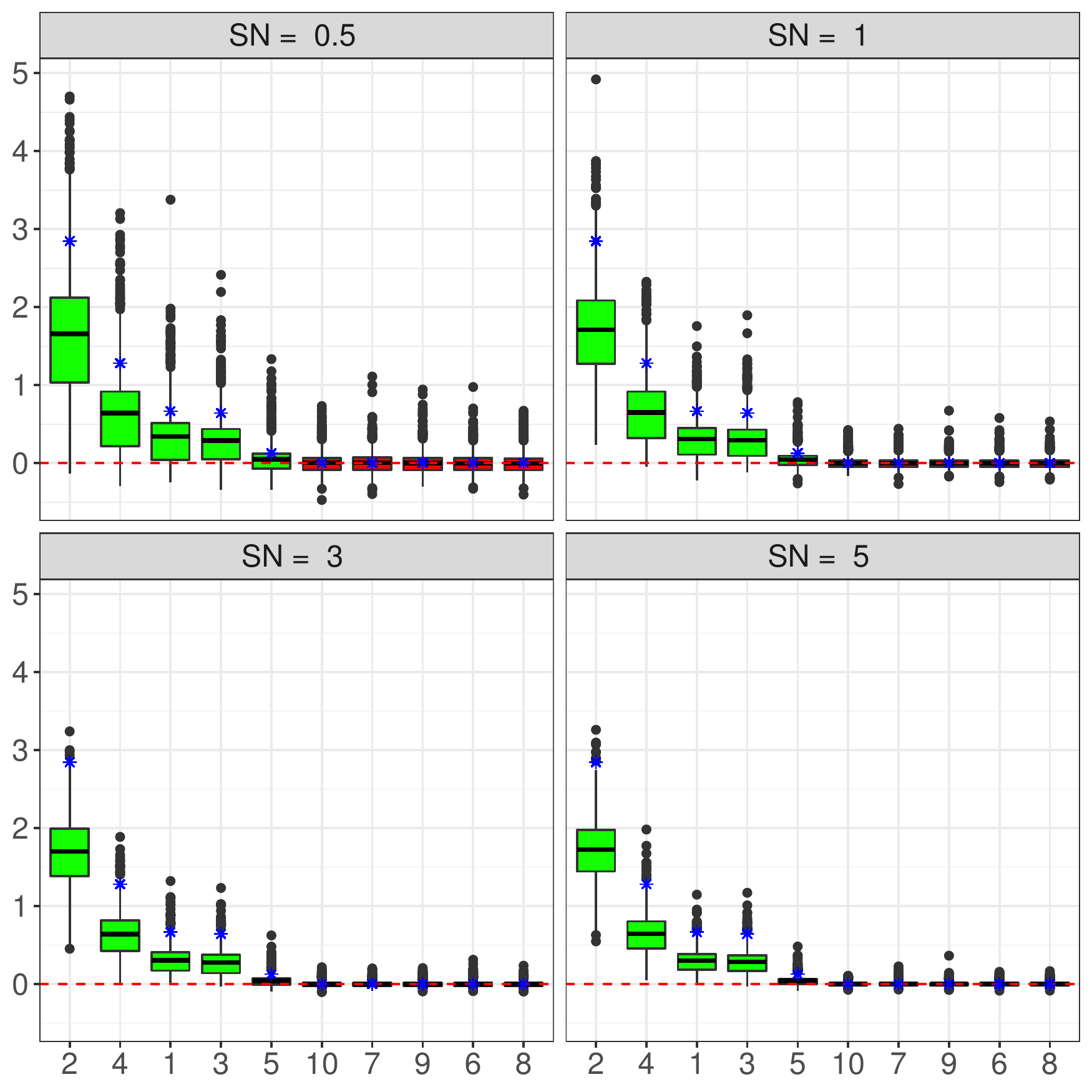}}
		\quad
		\subfigure[$n = 500$]{\includegraphics[width=3in]{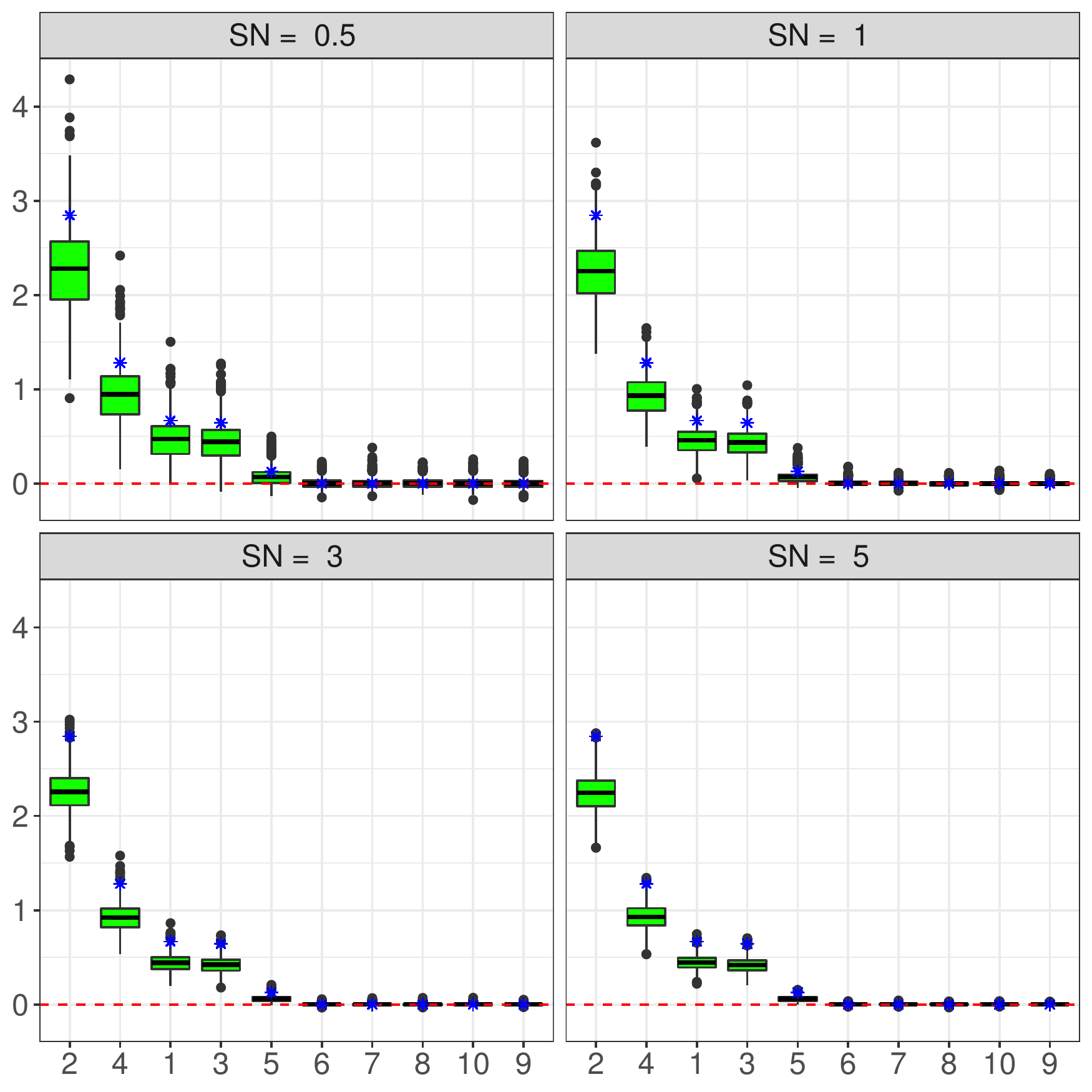} }}
	\caption{Simulation results for the permutation importance with various signal-to-noise ratios under a \textbf{polynomial model} as described in $(1)$ of the main article using $MC = 1,000$ Monte-Carlo iterations with a sample size of (a) $ n  = 100$ and (b) $n = 500$. The solid lines refer to the empirical mean and \textcolor{blue}{$\star$} to its expectation. }
	\label{PolyPlot}
\end{figure}
\FloatBarrier

\begin{figure}[h!]
	\centering
	\mbox{\subfigure[$n = 100$]{\includegraphics[width=3in]{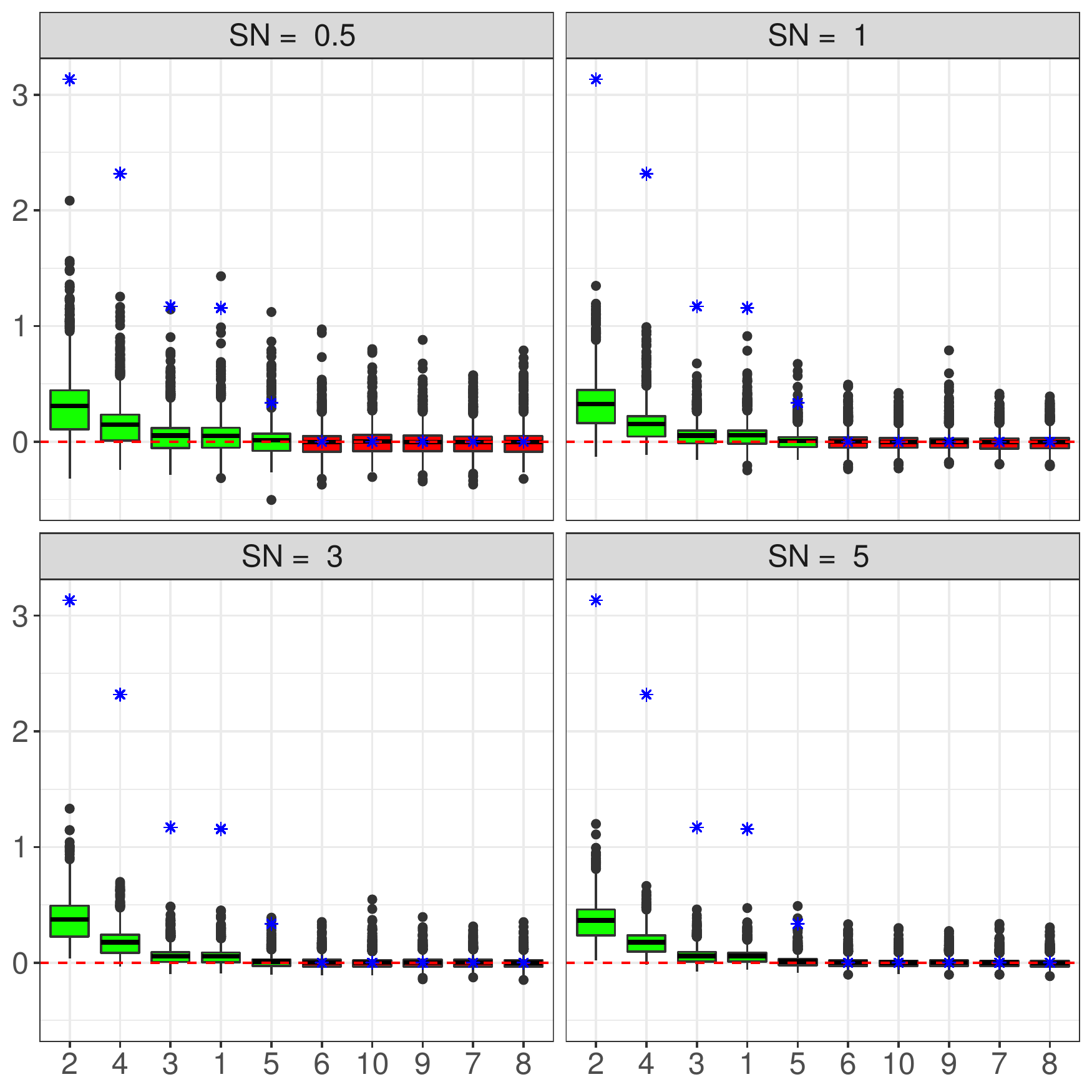}}
		\quad
		\subfigure[$n = 500$]{\includegraphics[width=3in]{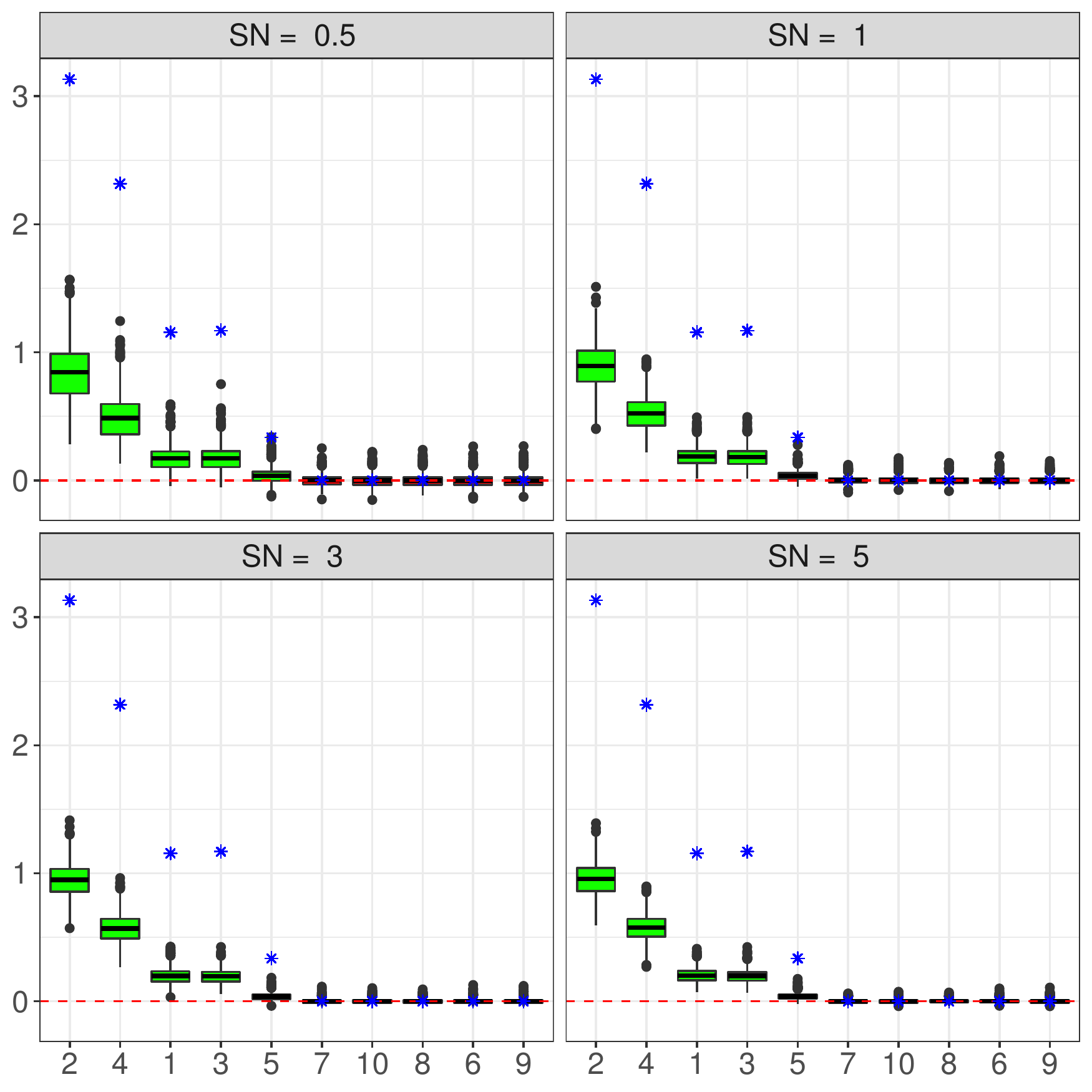} }}
	\caption{Simulation results for the permutation importance with various signal-to-noise ratios under a \textbf{trigonometric model} as described in $(1)$ of the main article using $MC = 1,000$ Monte-Carlo iterations with a sample size of (a) $ n  = 100$ and (b) $n = 500$. The solid lines refer to the empirical mean and \textcolor{blue}{$\star$} to a Monte-Carlo approximation of its expectation. }
	\label{SinPlot}
\end{figure}
\FloatBarrier

\begin{figure}[h!]
	\centering
	\mbox{\subfigure[$n = 100$]{\includegraphics[width=3in]{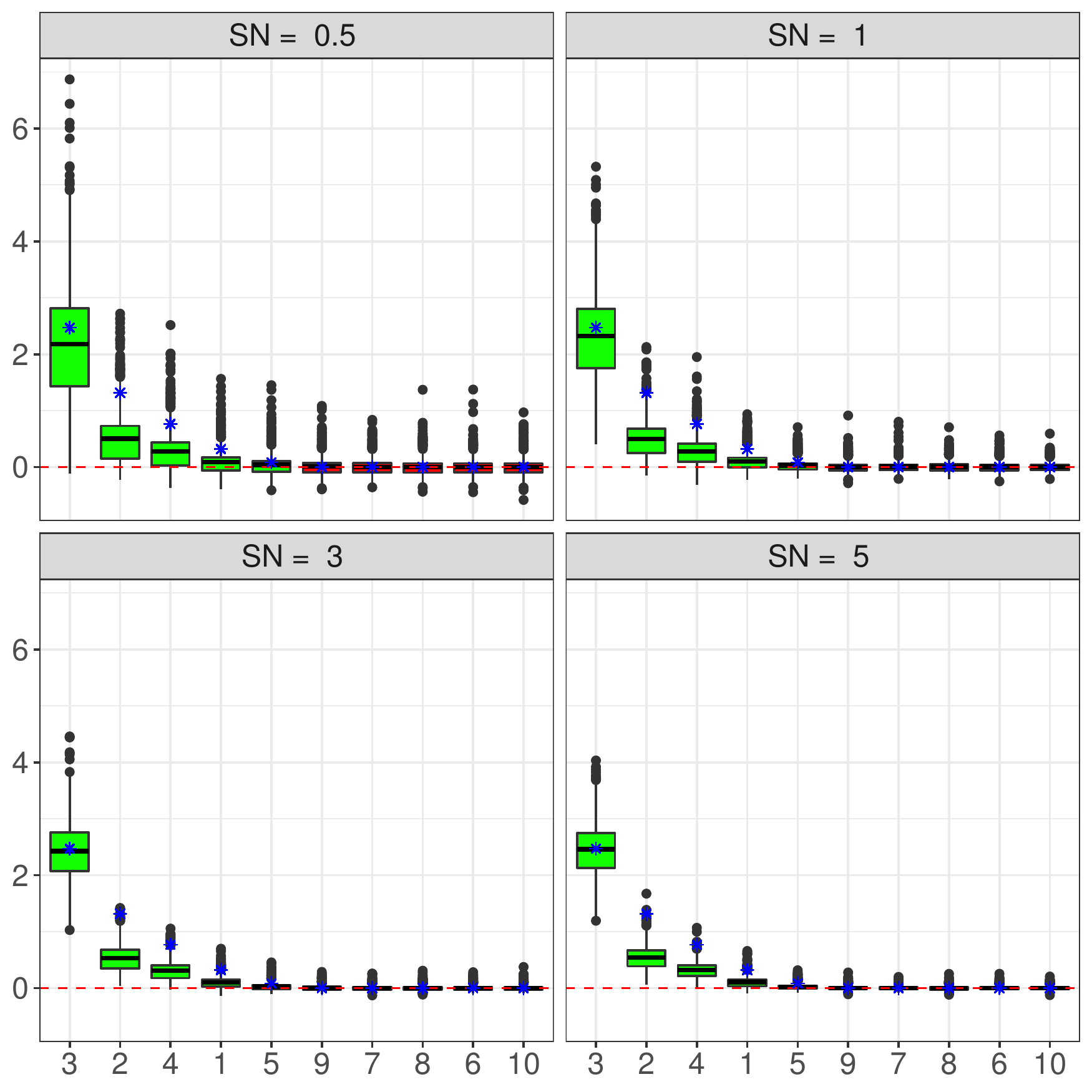}}
		\quad
		\subfigure[$n = 500$]{\includegraphics[width=3in]{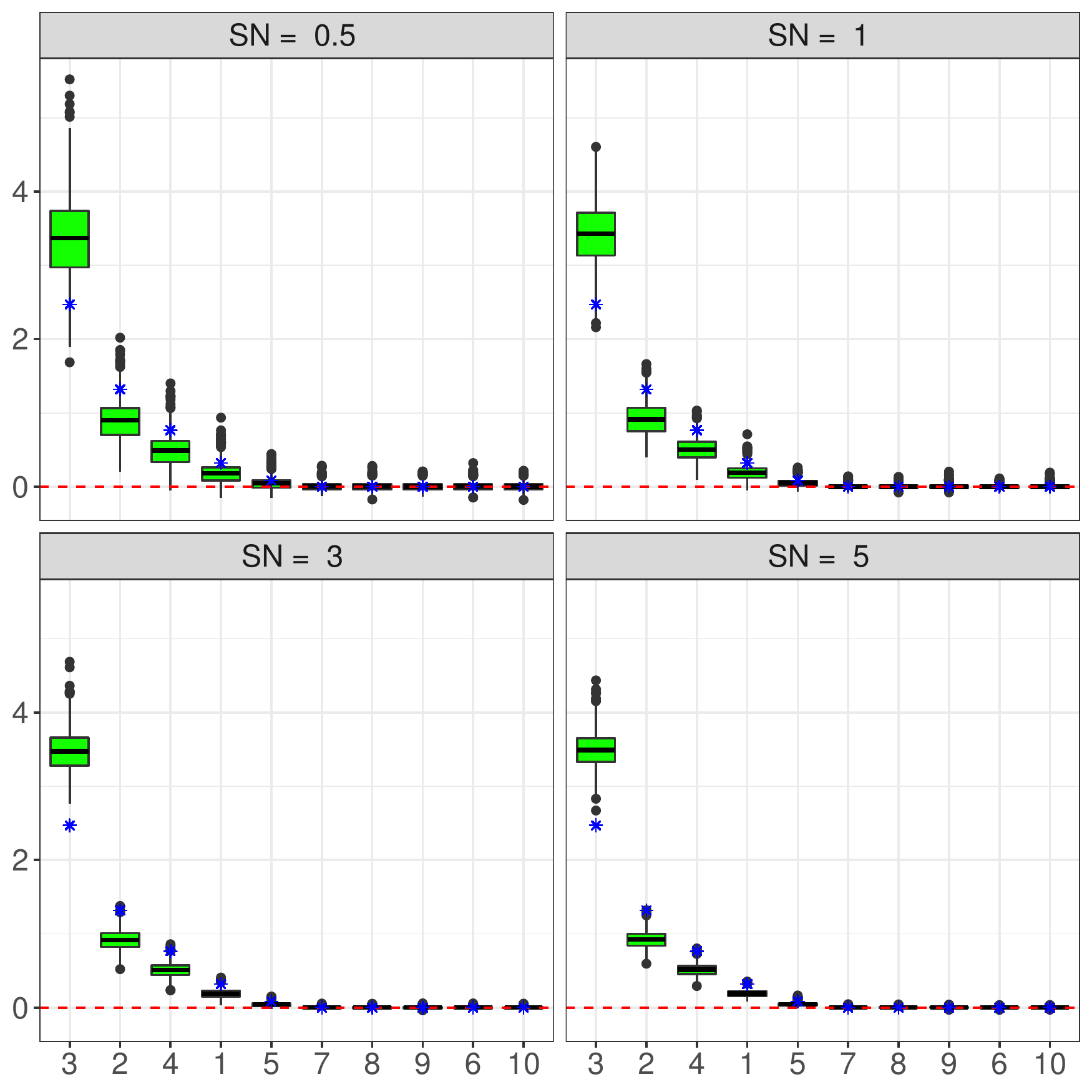} }}
	\caption{Simulation results for the permutation importance with various signal-to-noise ratios under a \textbf{non-continuous model} as described in $(1)$ of the main article using $MC = 1,000$ Monte-Carlo iterations with a sample size of (a) $ n  = 100$ and (b) $n = 500$. The solid lines refer to the empirical mean and \textcolor{blue}{$\star$} to a Monte-Carlo approximation of its expectation. }
	\label{NonContPlot}
\end{figure}
\FloatBarrier

\begin{table}[h!]
	\centering
	\begin{tabular}{ |c|c| c| c| c| c| c| c| c | c |}
		\hline 
		\multicolumn{2}{|c|}{ }& \multicolumn{4}{|c|}{ $n = 50$}  &  \multicolumn{4}{|c|}{ $n = 100$}  \\
		\cline{2-10}
		& $SN = $ & 0.5 & 1 & 3 & 5 &  0.5 & 1 & 3 & 5 \\
		\hline
		\multirow{4}{*}{\rotatebox[origin=c]{90}{\parbox[c]{1cm}{\centering Model }}} &linear & 0.189 & 0.364  &  0.807& 1.001 & 0.248 &   0.509& 1.181 & 1.528   \\
		\cline{2-10}
		& polynomial & 0.184 & 0.362 & 0.807  &1.033  & 0.243   &0.5  & 1.197  &1.594  \\
		\cline{2-10}
		& trigonometric & 0.1  &0.1  & 0.1 & 0.1 &   0.061&  0.064& 0.102  & 0.119  \\
		\cline{2-10}
		& non-continuous & 0.158 & 0.309  & 0.726 & 0.936 &  0.204 &  0.473& 1.152 &1.523  \\
		\hline \hline 
		\multicolumn{2}{|c|}{ }& \multicolumn{4}{|c|}{ $n = 500$}  &  \multicolumn{4}{|c|}{ $n = 1,000$}  \\
		\cline{2-10}
		& $SN = $ & 0.5 & 1 & 3 & 5 &  0.5 & 1 & 3 & 5 \\
		\hline
		\multirow{4}{*}{\rotatebox[origin=c]{90}{\parbox[c]{1cm}{\centering Model }}} &linear & 0.365 & 0.743  & 1.937 & 2.781 & 0.400 &   0.808&  2.178  & 3.240\\
		\cline{2-10}
		& polynomial &  0.365&  0.754& 1.995  & 2.919  &  0.395 &  0.812& 2.246 & 3.400 \\
		\cline{2-10}
		& trigonometric & 0.098 &  0.215& 0.451 &0.549  & 0.170  & 0.335 &  0.7&  0.862\\
		\cline{2-10}
		& non-continuous & 0.357  & 0.759 &2.094  &  3.153&  0.395 &  0.829& 2.376 &3.714  \\
		\hline
	\end{tabular}
	\caption{Estimator $\widehat{SN}_n$ as given in equation $(14)$ of the main article under various sample sizes and signal-to-noise ratios using $MC =1,000$ Monte-Carlo iterates for the $p < n$ regression problem.  } \label{SNEstimator}
\end{table}

Table \ref{SNEstimator} refers to the estimator $\widehat{SN}_n$ of $SN$ as proposed in the main article under various sample sizes. One can see that $\widehat{SN}_n$ tends to be smaller than $SN$, but slowly moves to $SN$ for an increased sample size. 

\newpage 
\section{Results for $p > n$ Problems}

\begin{figure}[!ht]
	\begin{minipage}{\textwidth}
		\centering
		\includegraphics[width=.4\textwidth]{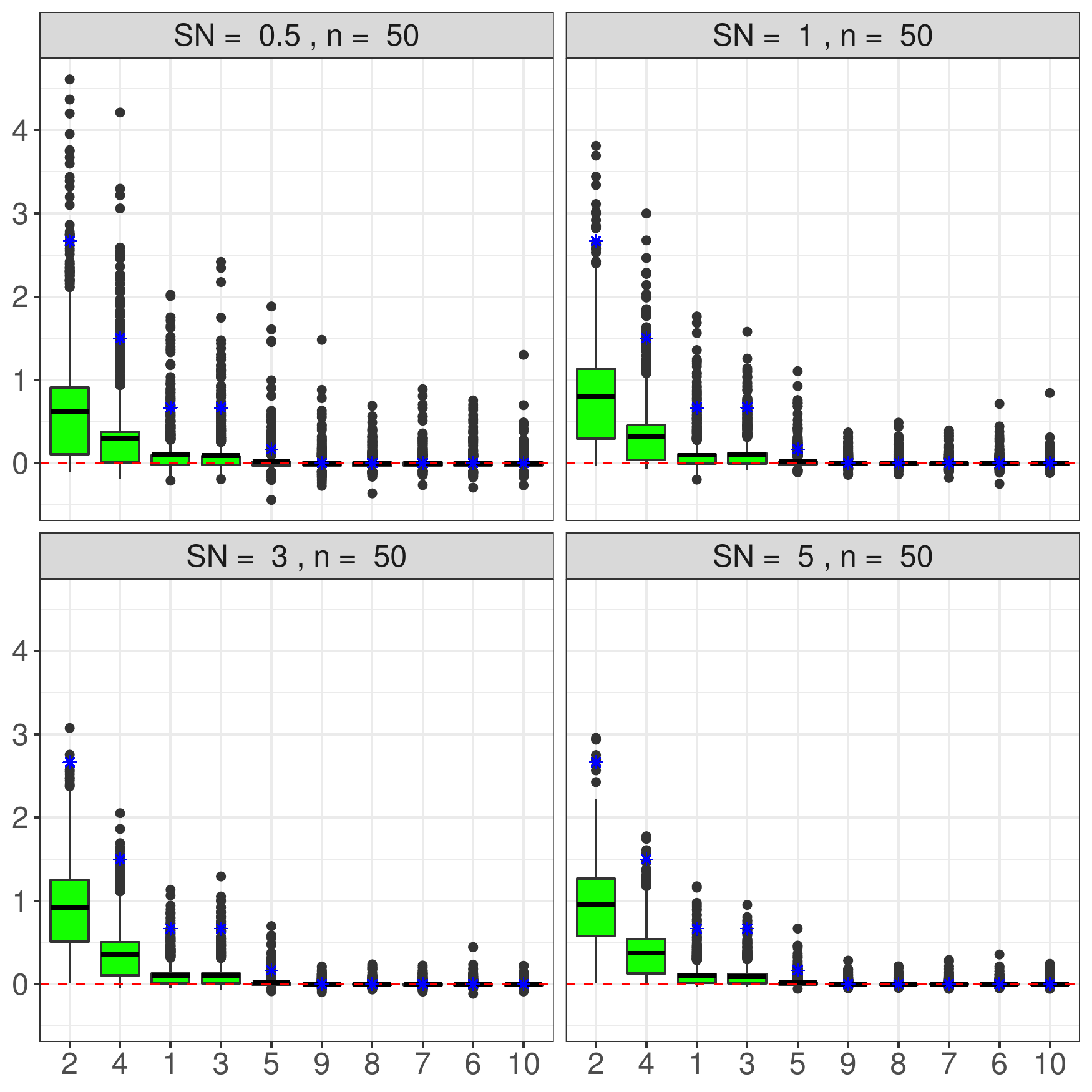}\quad
		\includegraphics[width=.4\textwidth]{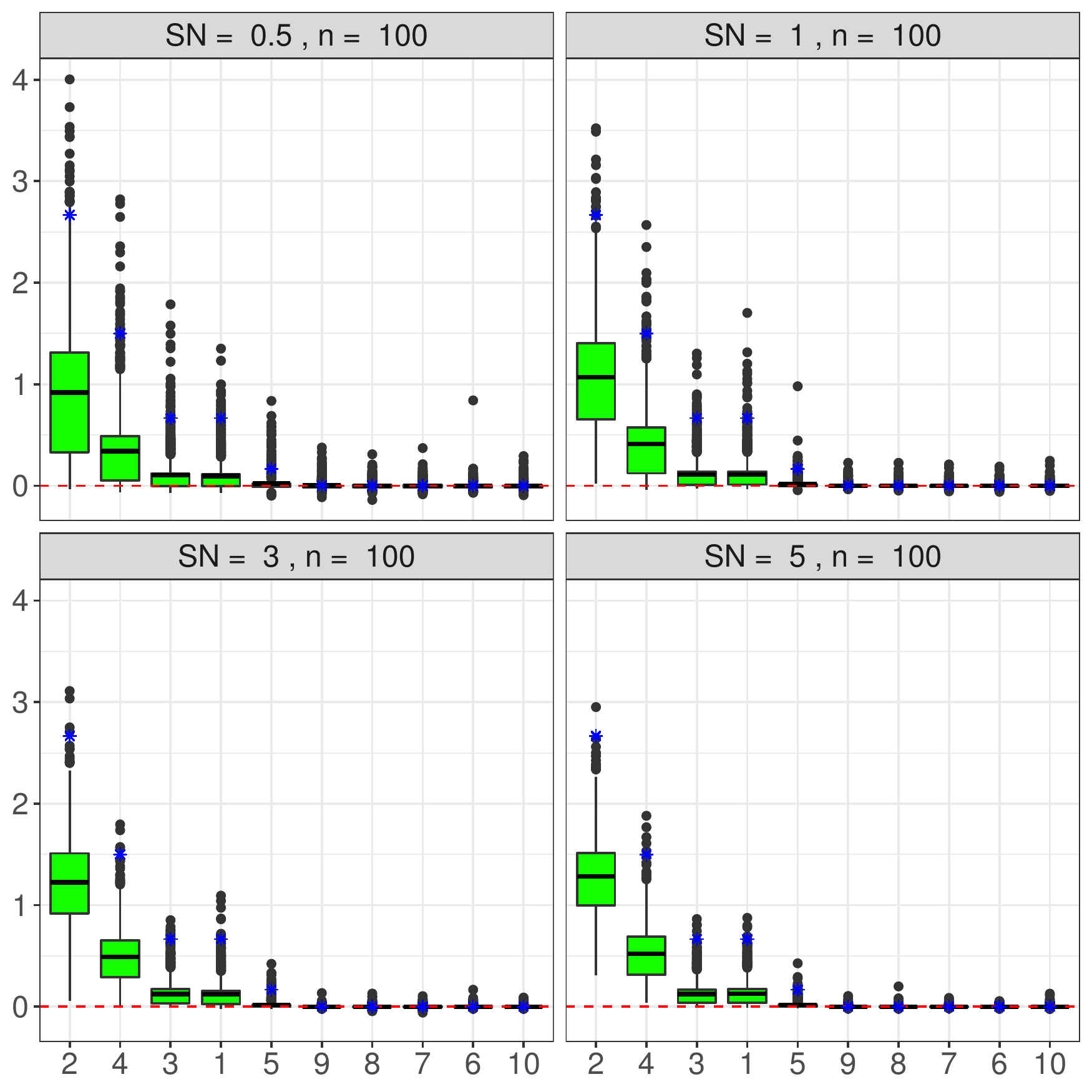}\\
		\includegraphics[width=.4\textwidth]{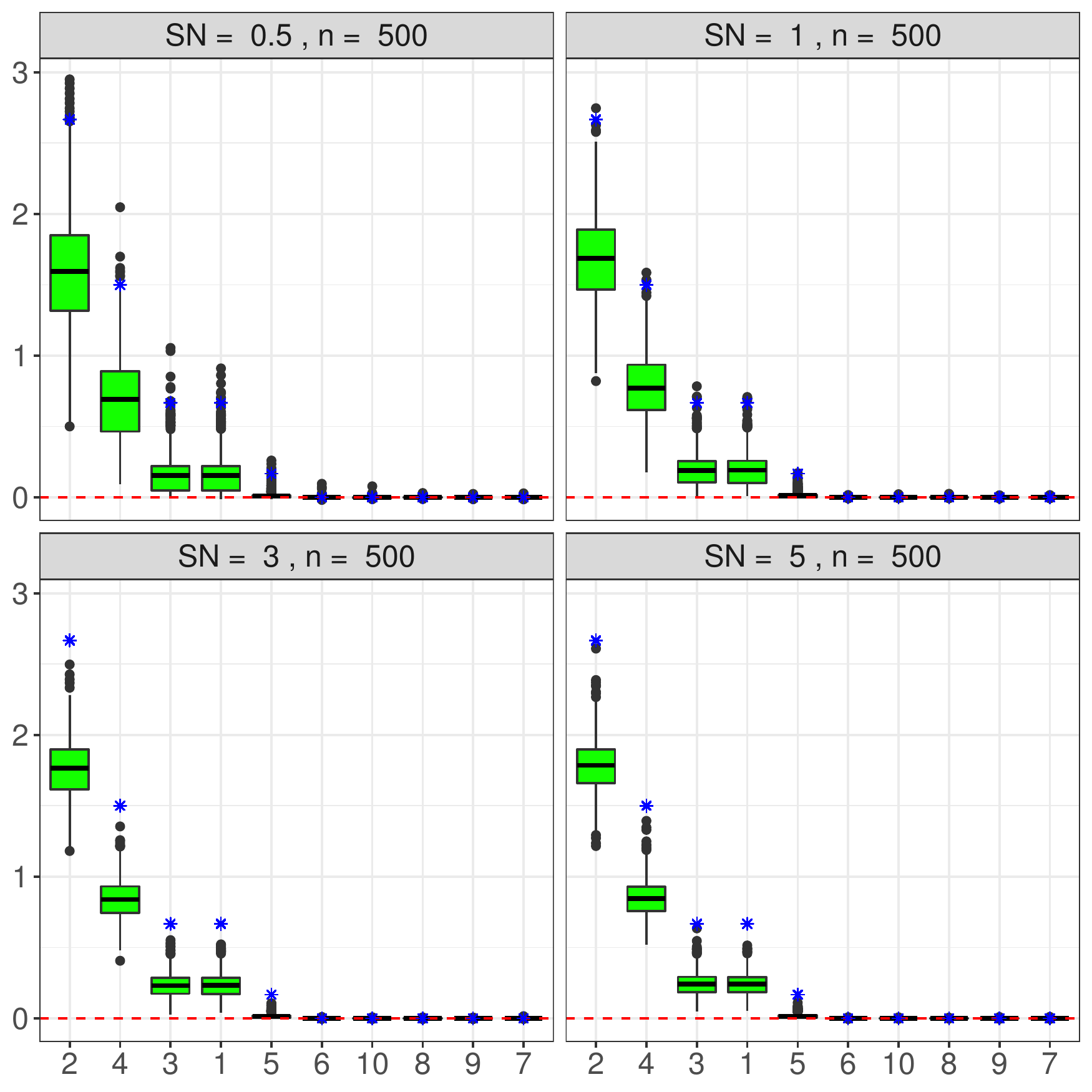}\quad
		\includegraphics[width=.4\textwidth]{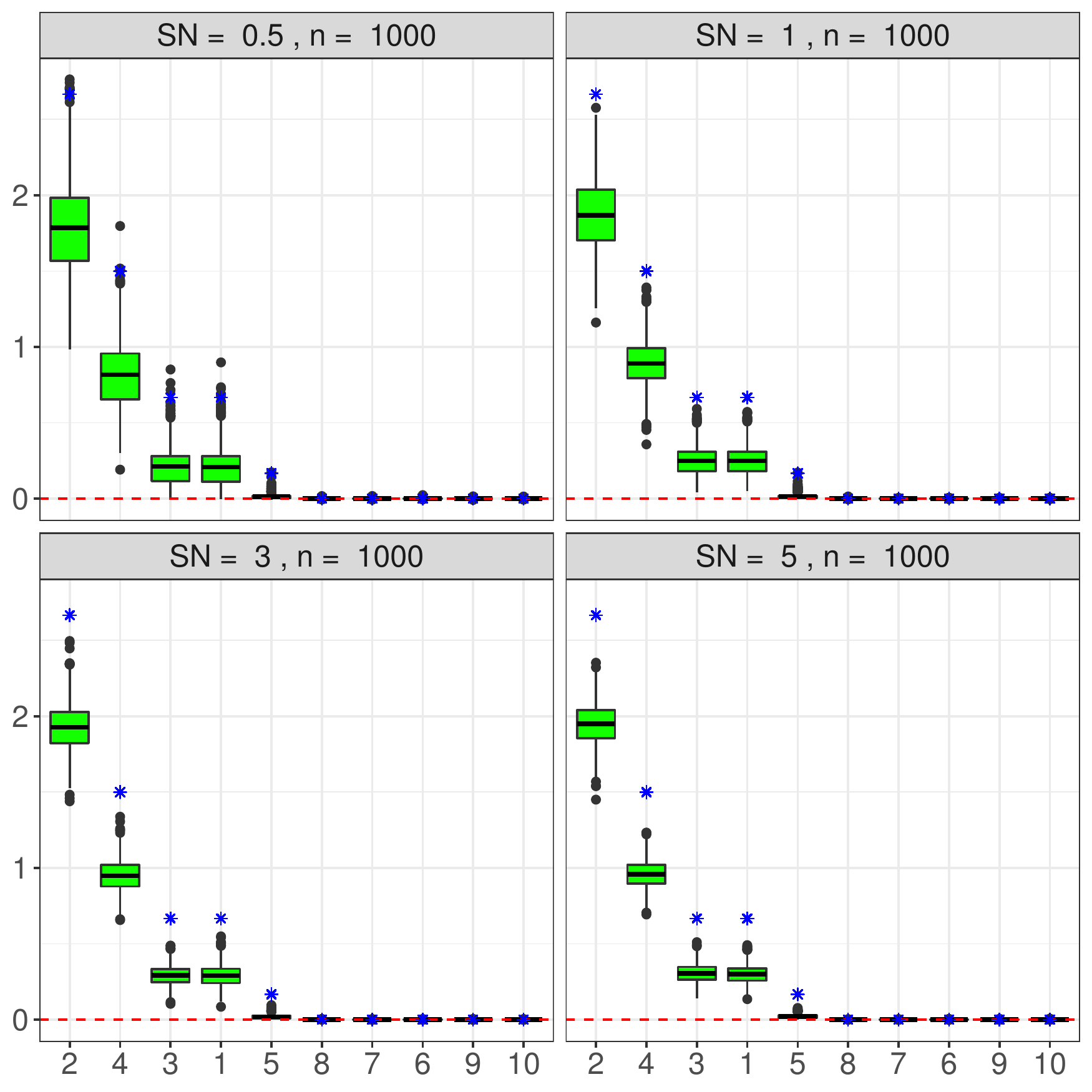}
		\caption{Simulation results for the permutation importance with various signal-to-noise ratios under a \textbf{linear model} as described in $(1)$ of the main article using $MC = 1,000$ Monte-Carlo iterations under the \textbf{high-dimensional} setting. The solid lines refer to the empirical mean and \textcolor{blue}{$\star$} to its expectation. }
		\label{HighDimLinear}
	\end{minipage}\\[1em]
\end{figure}
\FloatBarrier

\begin{figure}[!ht]
	\begin{minipage}{\textwidth}
		\centering
		\includegraphics[width=.4\textwidth]{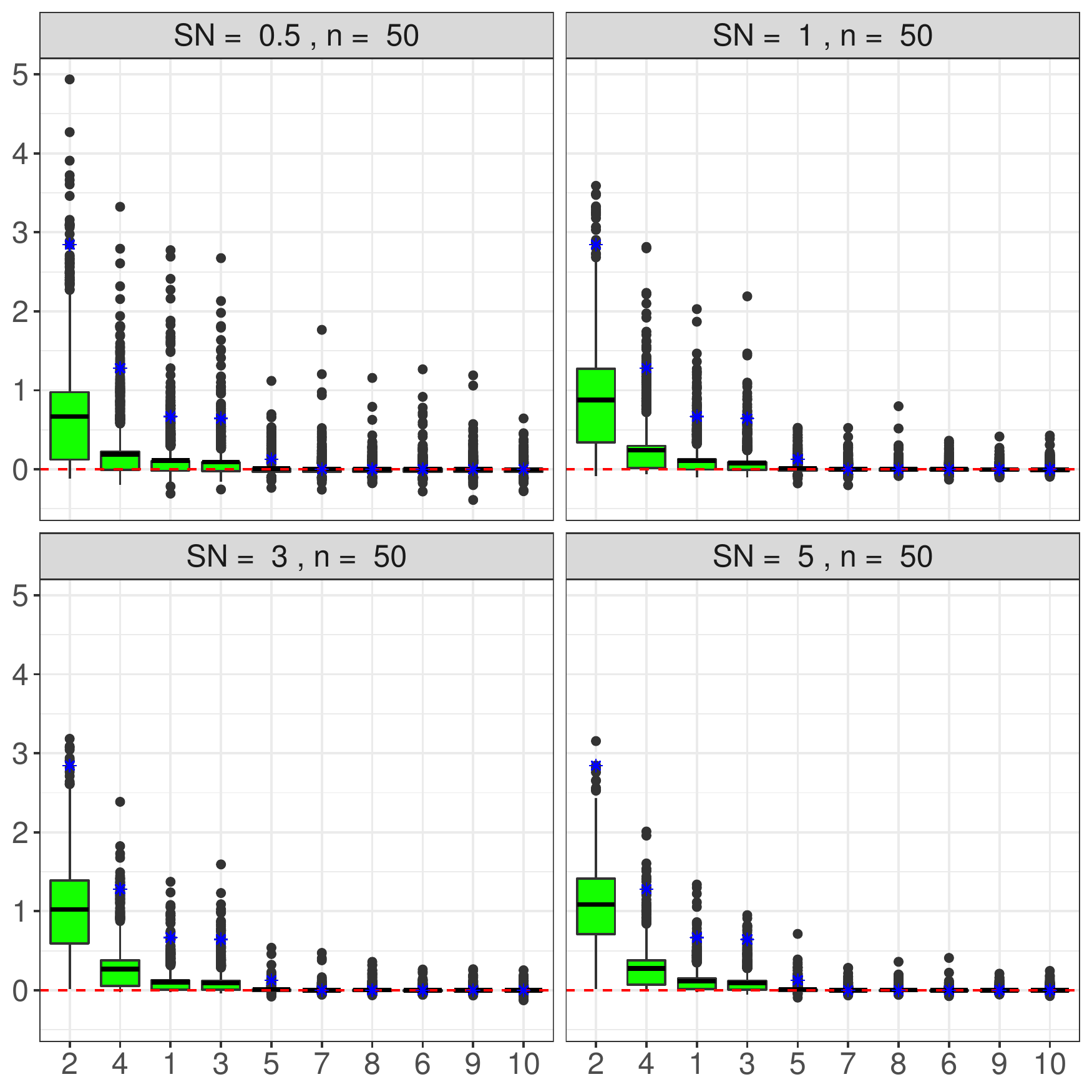}\quad
		\includegraphics[width=.4\textwidth]{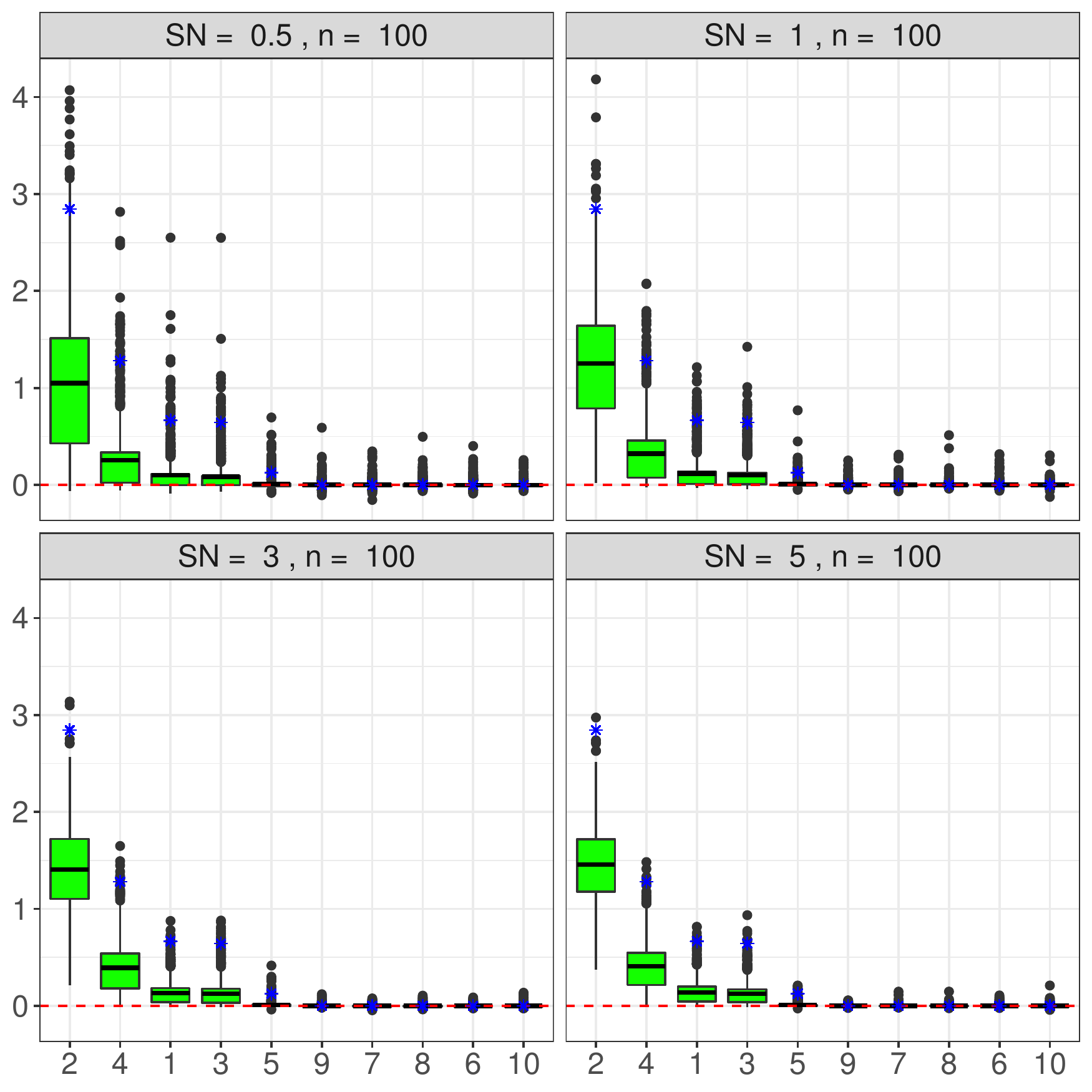}\\
		\includegraphics[width=.4\textwidth]{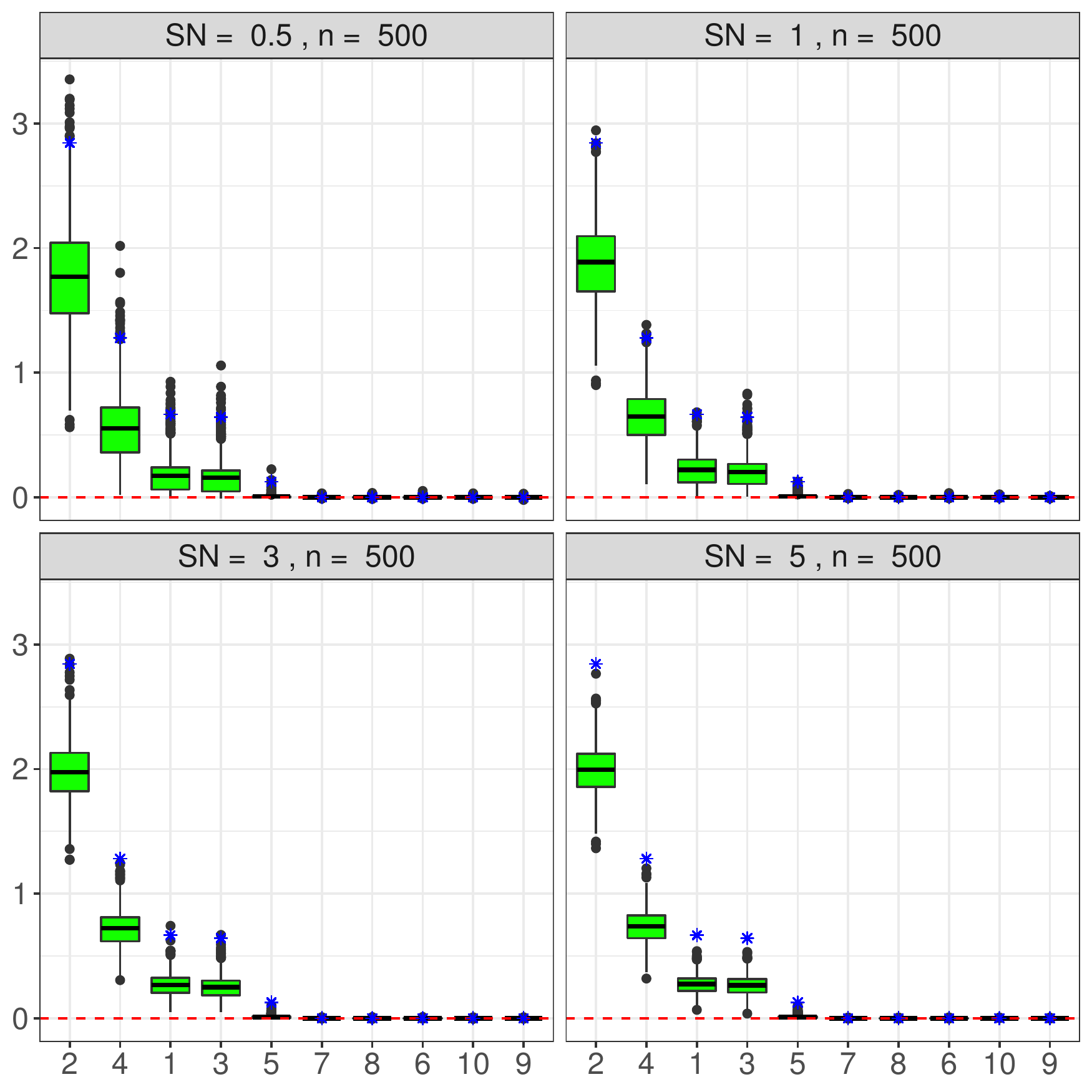}\quad
		\includegraphics[width=.4\textwidth]{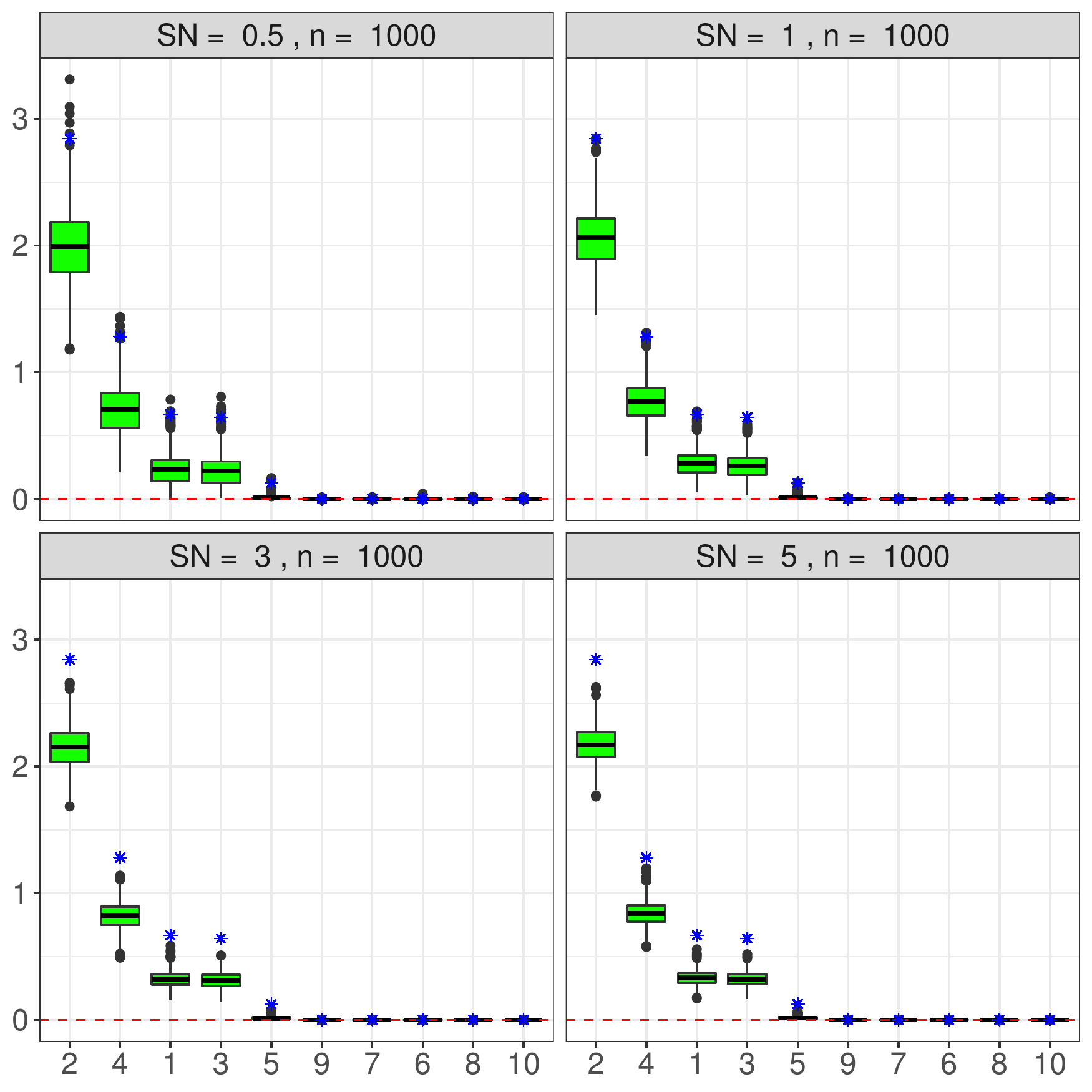}
		\caption{Simulation results for the permutation importance with various signal-to-noise ratios under a \textbf{polynomial model} as described in $(1)$ of the main article using $MC = 1,000$ Monte-Carlo iterations under the \textbf{high-dimensional} setting. The solid lines refer to the empirical mean and \textcolor{blue}{$\star$} to its expectation.  }
		\label{HighDimPoly}
	\end{minipage}\\[1em]
\end{figure}
\FloatBarrier

\begin{figure}[!ht]
	\begin{minipage}{\textwidth}
		\centering
		\includegraphics[width=.4\textwidth]{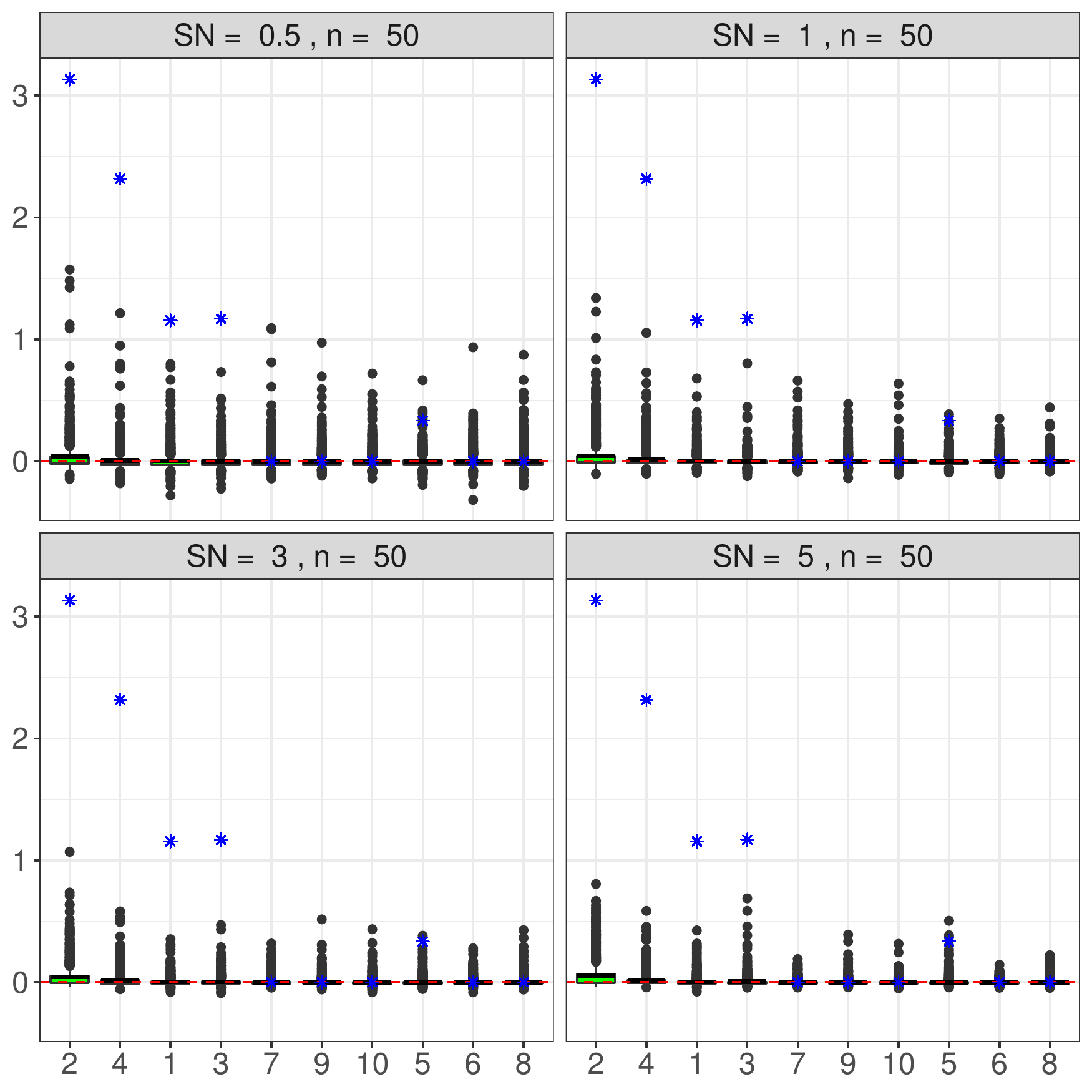}\quad
		\includegraphics[width=.4\textwidth]{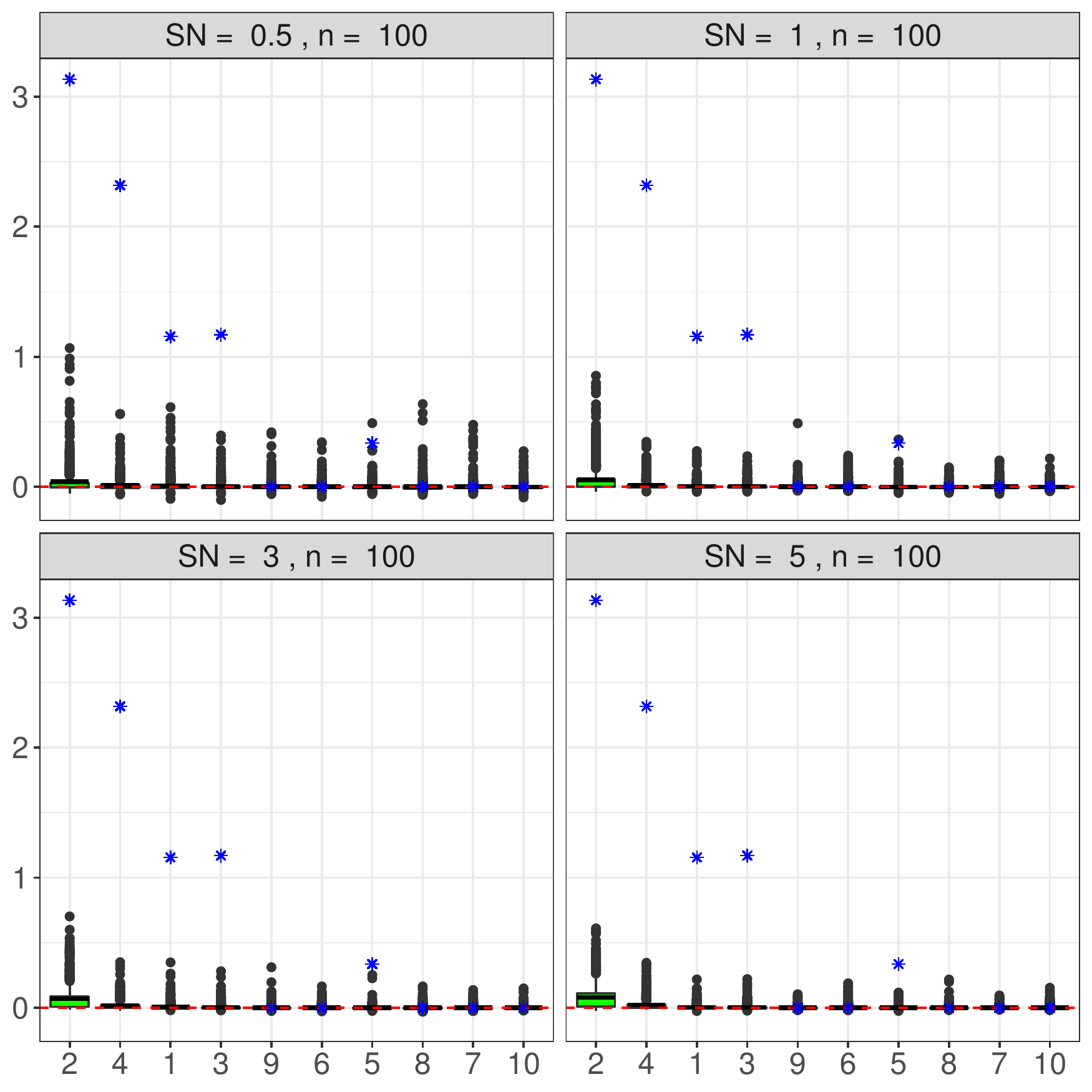}\\
		\includegraphics[width=.4\textwidth]{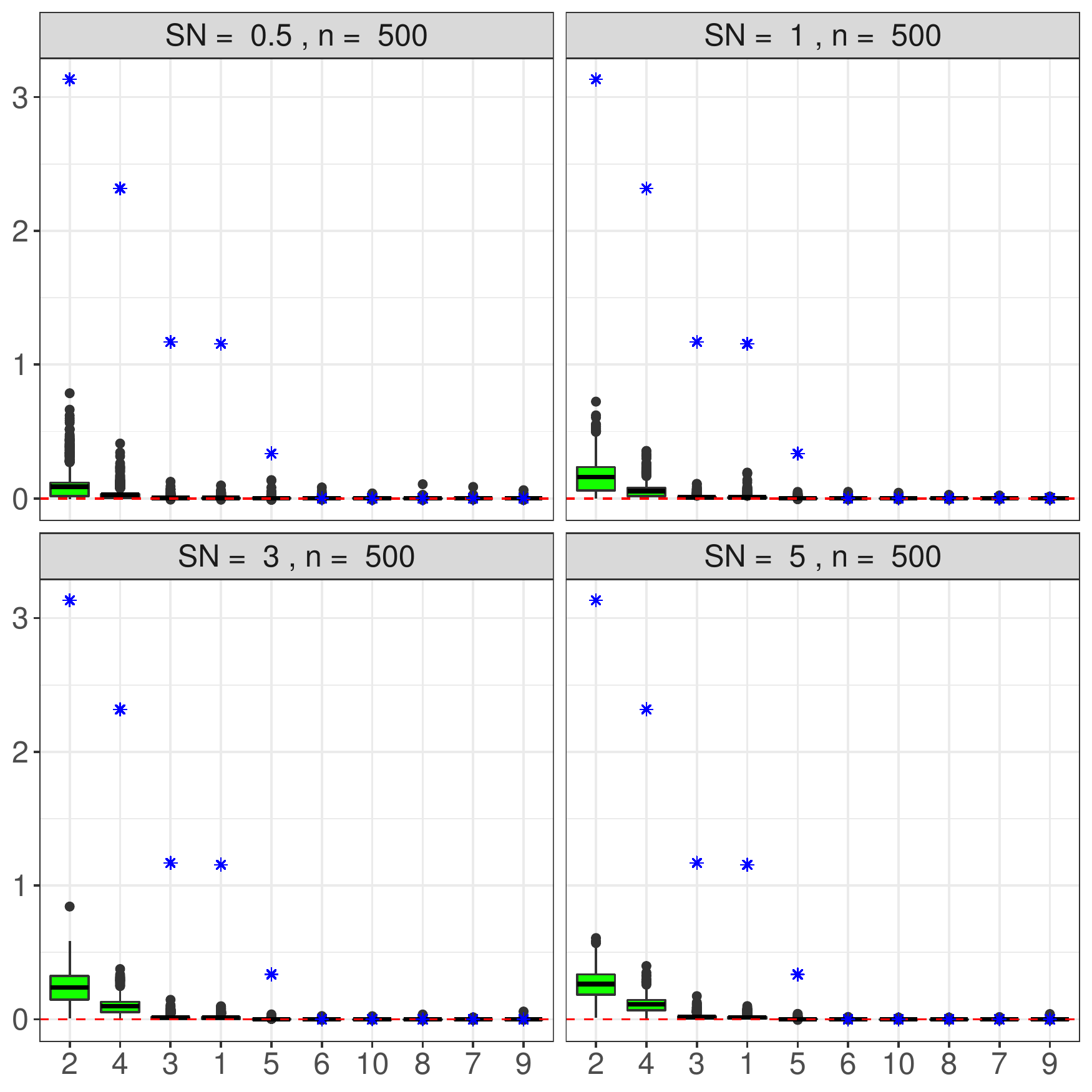}\quad
		\includegraphics[width=.4\textwidth]{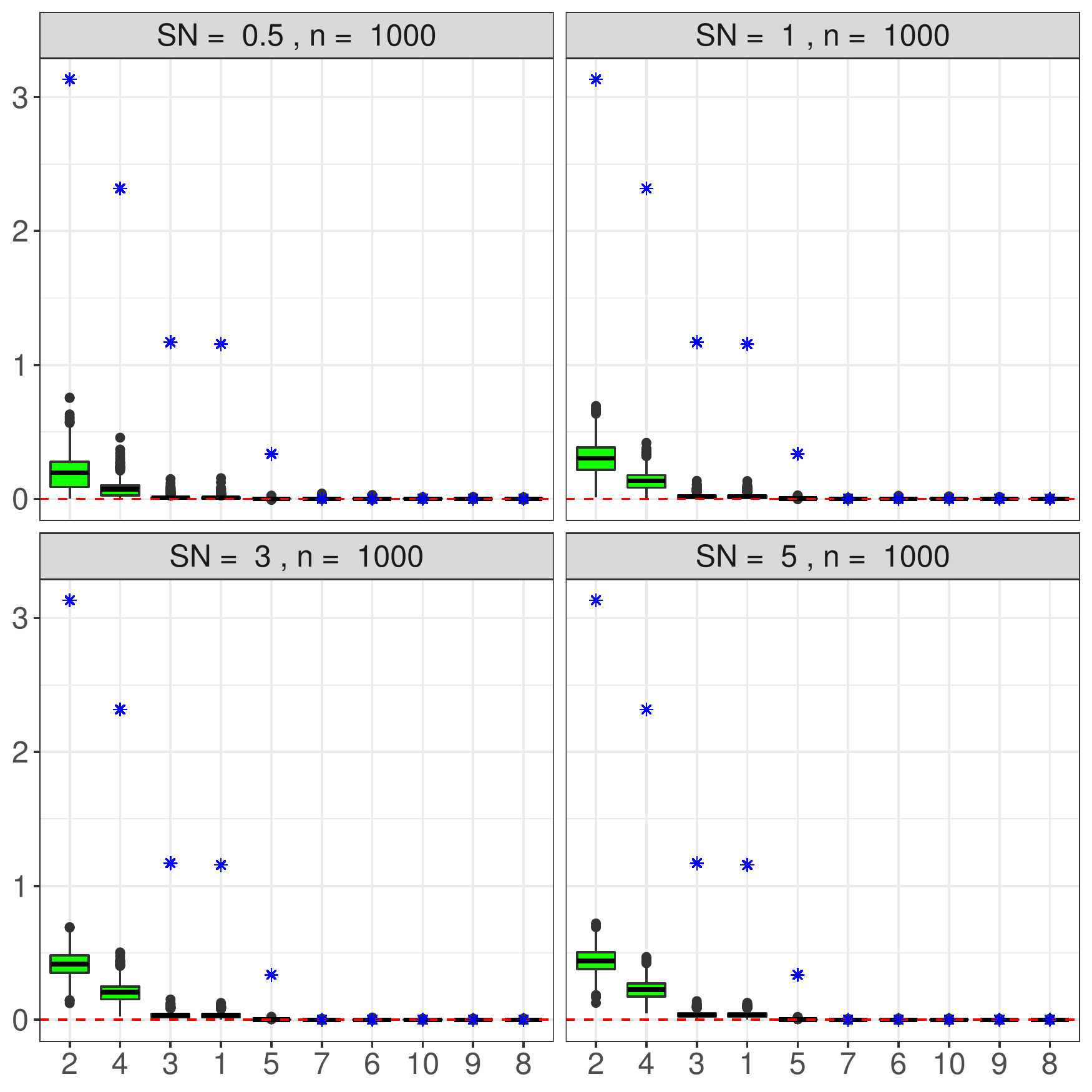}
		\caption{Simulation results for the permutation importance with various signal-to-noise ratios under a \textbf{trigonometric model} as described in $(1)$ of the main article using $MC = 1,000$ Monte-Carlo iterations under the \textbf{high-dimensional} setting. The solid lines refer to the empirical mean and \textcolor{blue}{$\star$} to a Monte-Carlo approximation of its expectation.  }
		\label{HighDimSinus}
	\end{minipage}\\[1em]
\end{figure}
\FloatBarrier

\begin{figure}[!ht]
	\begin{minipage}{\textwidth}
		\centering
		\includegraphics[width=.4\textwidth]{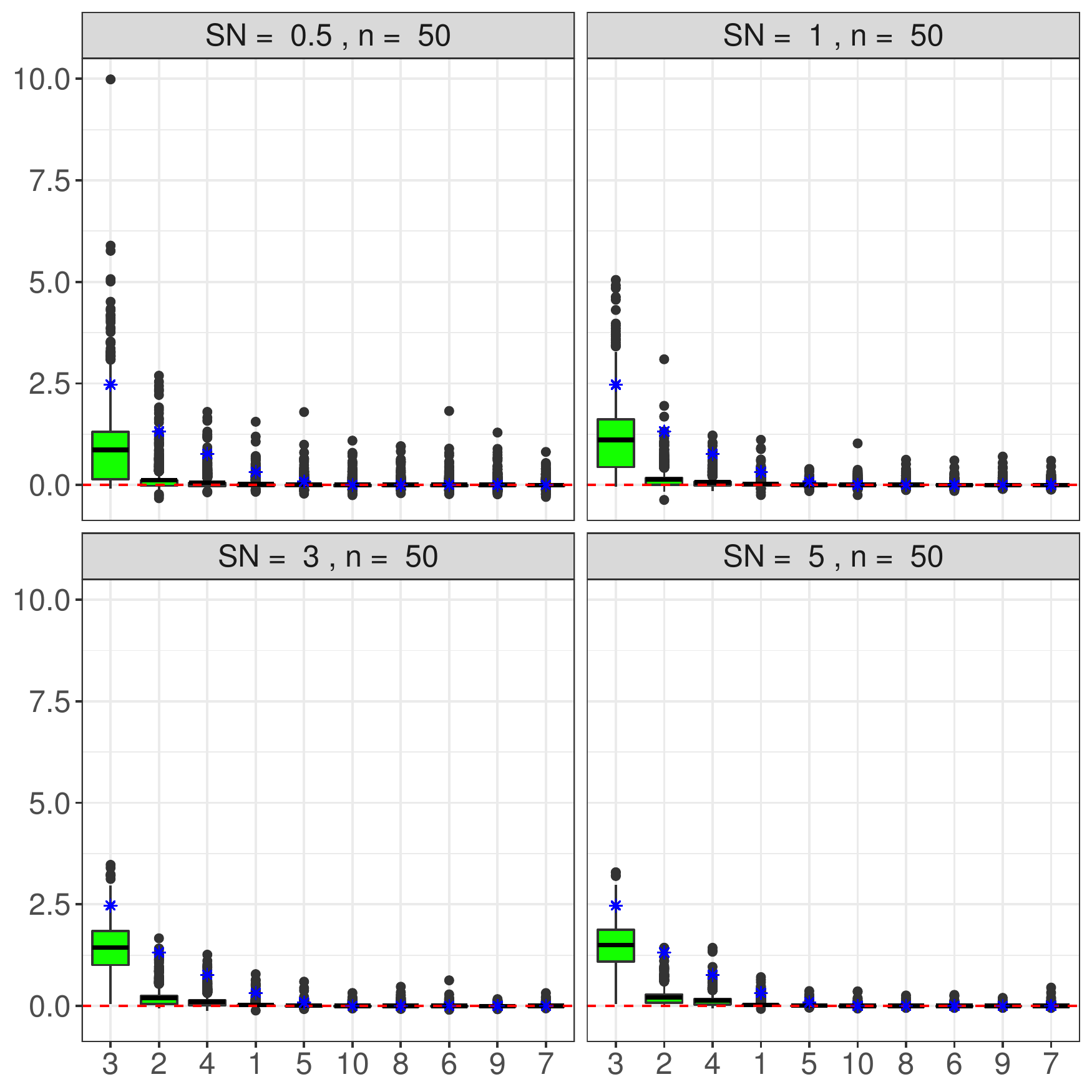}\quad
		\includegraphics[width=.4\textwidth]{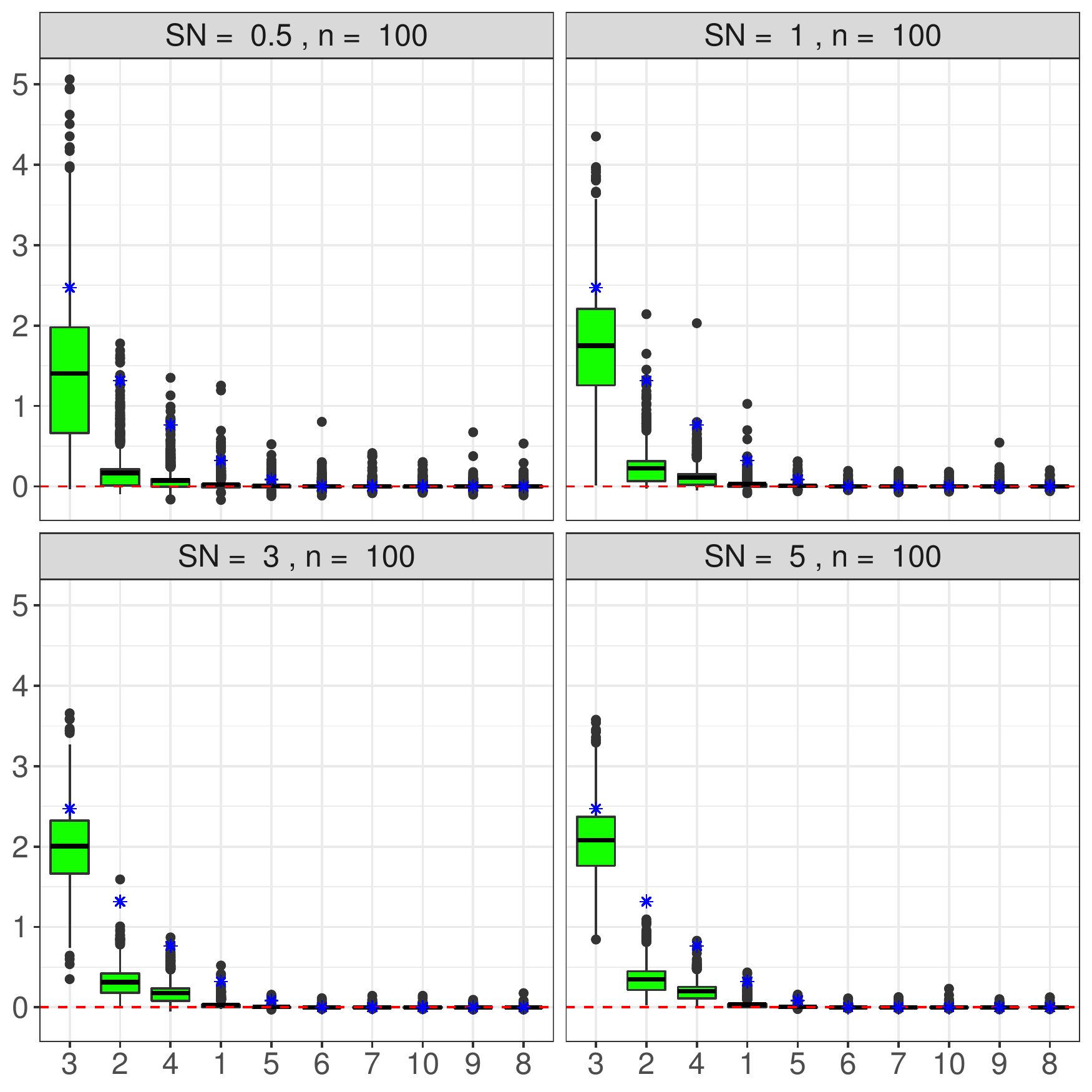}\\
		\includegraphics[width=.4\textwidth]{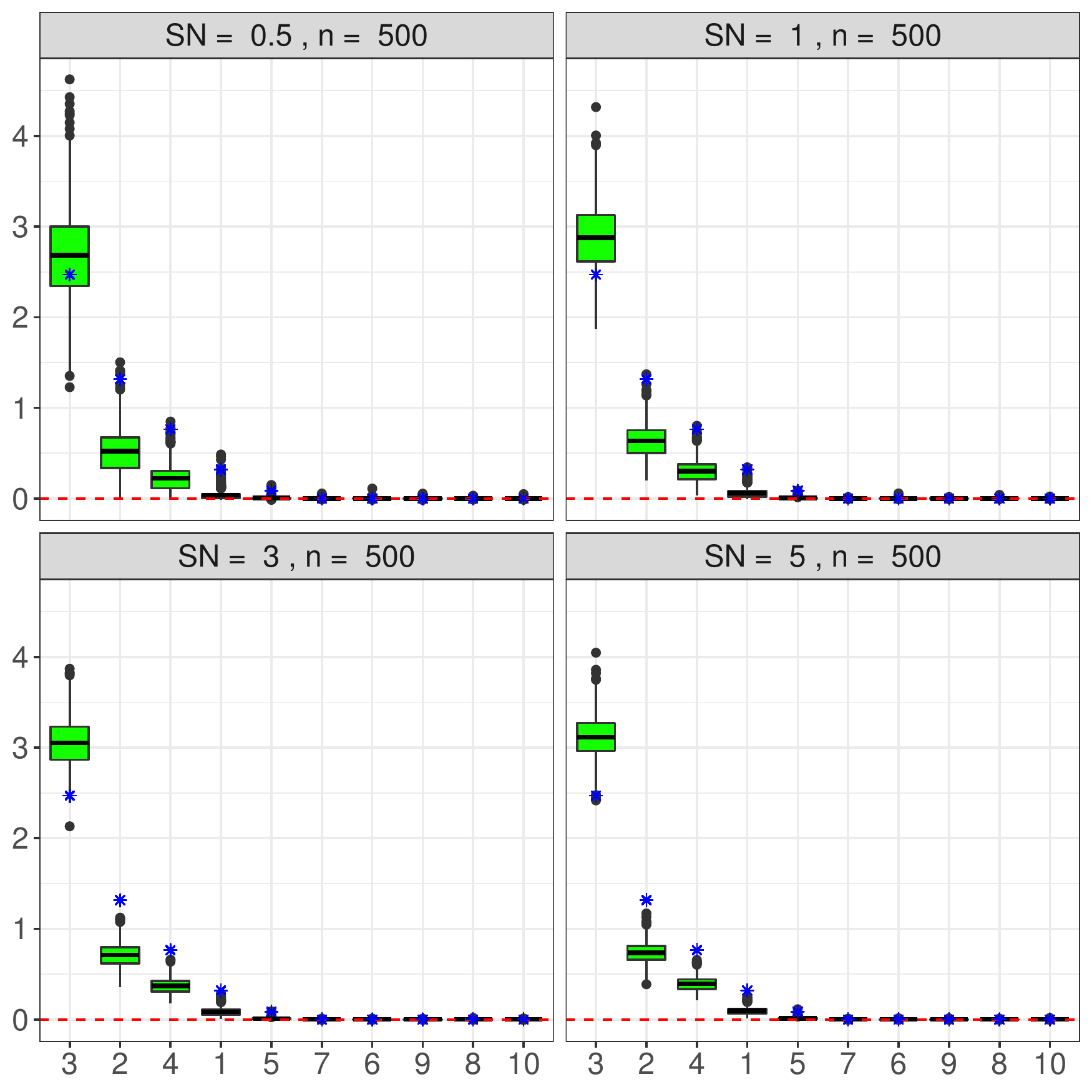}\quad
		\includegraphics[width=.4\textwidth]{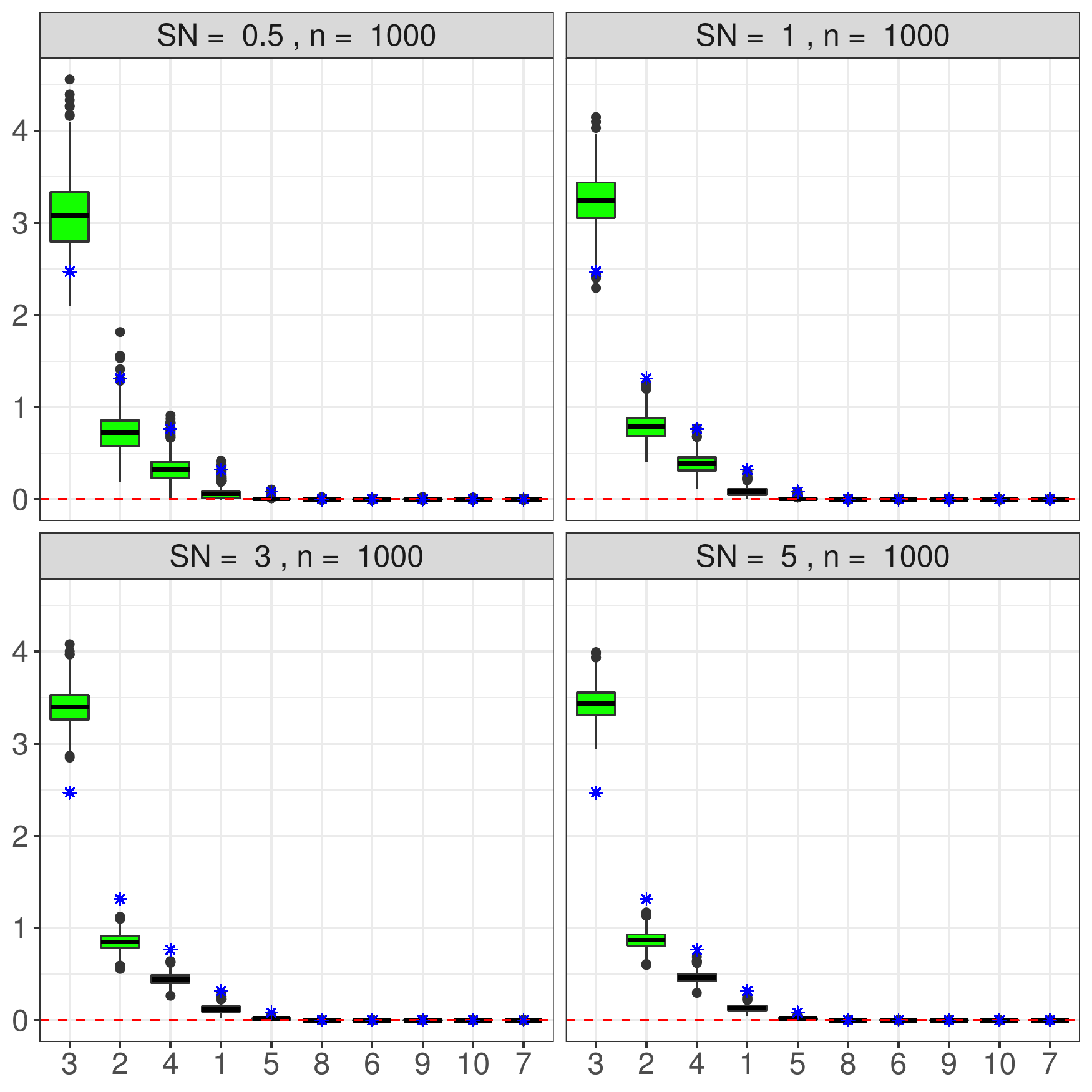}
		\caption{Simulation results for the permutation importance with various signal-to-noise ratios under a \textbf{non-continuous model} as described in $(1)$ of the main article using $MC = 1,000$ Monte-Carlo iterations under the \textbf{high-dimensional} setting. The solid lines refer to the empirical mean and \textcolor{blue}{$\star$} to a Monte-Carlo approximation of its expectation. }
		\label{HighDimNonCont}
	\end{minipage}\\[1em]
\end{figure}
\FloatBarrier

\end{document}